\documentclass[final]{colt2025} %

\usepackage{jmlrutils}
\usepackage{panos-colt}

\usepackage{amsmath}
\usepackage{amssymb}
\usepackage{mathtools}

\usepackage{microtype}
\usepackage{graphicx}
\usepackage{booktabs} 
\usepackage{textgreek} 
\usepackage{yhmath}

\newtheorem*{theorem*}{Theorem} 
\newtheorem*{proposition*}{Proposition}
\newtheorem*{corollary*}{Corollary}

\title[Optimal Scheduling of Dynamic Transport]{Optimal Scheduling of Dynamic Transport}
\usepackage{times}
\coltauthor{%
\Name{Panos Tsimpos} \Email{ptsimpos@mit.edu}\\
\addr Operations Research Center, MIT, Cambridge, MA 02139, USA
\AND
\Name{Zhi Ren} \Email{zren@mit.edu}\\
\addr Department of Mathematics, MIT, Cambridge, MA 02139, USA
\AND
\Name{Jakob Zech} \Email{jakob.zech@uni-heidelberg.de}\\
\addr Interdisciplinary Center for Scientific Computing, Heidelberg University, Heidelberg 69120, Germany
\AND
\Name{Youssef Marzouk} \Email{ymarz@mit.edu}\\
\addr Laboratory for Information and Decision Systems, MIT, Cambridge, MA 02139, USA
}

\begin{document}

\maketitle

\begin{abstract}
  Flow-based methods for sampling and generative modeling use continuous-time dynamical systems to represent a {transport map} that pushes forward a source measure to a target measure.
The introduction of a time axis provides considerable design freedom, and a central question is how to exploit this freedom. Though many popular methods seek straight line (i.e., zero acceleration) trajectories, we show here that a specific class of ``curved'' trajectories can significantly improve approximation and learning. In particular, we consider the unit-time interpolation of any given transport map $T$ and seek the schedule $\tau: [0,1] \to [0,1]$ that minimizes the spatial Lipschitz constant of the corresponding velocity field over all times $t \in [0,1]$. This quantity is crucial as it allows for control of the approximation error when the velocity field is learned from data. We show that, for a broad class of source/target measures and transport maps $T$, the \emph{optimal schedule} can be computed in closed form, and that the resulting optimal Lipschitz constant is \emph{exponentially smaller} than that induced by an identity schedule  (corresponding to, for instance, the Wasserstein geodesic). 
  Our proof technique relies on the calculus of variations and $\Gamma$-convergence, allowing us to approximate the aforementioned degenerate objective by a family of smooth, tractable problems.
\end{abstract}

\begin{keywords}
  Measure transport, optimal transport, calculus of variations, $\Gamma$-convergence, approximation theory, flow-based models
\end{keywords}

\medskip

\section{Introduction}
\label{Introduction}

Characterizing complex probability distributions is a central task in computational statistics and machine learning. A typical goal is to draw samples from a distribution of interest, given access to its unnormalized density or to a finite set of training samples. %
Alternatively, given training samples, one might wish to estimate the density function itself. \textit{Transportation of measure} provides a unifying approach to these goals \citep{EricDensityEstimation2, MoselhyM12/BayesianInferenceOT, EricDensityEstimation1, RezendeM15/VariationalNormalizingFlow,MMPS16, nips/KingmaD18/Glow, papamakarios2021normalizing}.
The essential idea of transport is to represent $\nu$, the target distribution of interest, as the \textit{pushforward} of a tractable source distribution $\mu$ by some map $T$, i.e., $T_\sharp\mu = \nu$. Sampling then follows by evaluating $T$ on samples from $\mu$, and simple expressions for the density of  $T_\sharp\mu$ are available when $T$ is invertible and differentiable \citep{wang2022MiniMaxDensityEstimation}.

To realize these approaches, we must find representations of transport maps $T$ that are sufficiently expressive and that allow easy computation of Jacobian determinants. To this end, a variety of parameterizations have been proposed, using classical function approximation schemes (such as polynomials or wavelets) \citep{focm/BaptistaMZ24}, neural networks \citep{icml/HuangKLC18/NeuralAutoFlow,wehenkel2019unconstrained,/nips/DurkanB0P19/NeuralSplineFlow}, and combinations thereof  \citep{icml/JainiSY19/SOSFlow}. Generally speaking, we can divide these representations into two categories:  \textit{static approaches} that represent the displacement from $x$ to $T(x)$ directly \citep{RezendeM15/VariationalNormalizingFlow,icml/HuangKLC18/NeuralAutoFlow,icml/JainiSY19/SOSFlow,wang2022MiniMaxDensityEstimation}, and \textit{dynamic approaches} that employ evolution over some fictitious time \citep{chen2018neuralODE, iclr/Grathwohl19/FFJORD,onken2021OTFlow,/iclr/LipmanCBNL23/FlowMatching,Liu2022FlowStraight, /corr/abs-2209-15571/EricStochasticInterpolant}. The latter is our focus here.

In the dynamical approach, a recurring structure is to represent the transport as the flow map of an ODE system:
\begin{equation}\label{eq:ODE-geneal-form}
    \begin{cases}
      \frac{d}{dt}X(x,t) &= v(X(x,t), t),\qquad t\in [0,1],\\
      X(x,0) &= x.
    \end{cases} 
\end{equation} %
Under mild assumptions,
\eqref{eq:ODE-geneal-form}
is solvable
  and induces trajectories %
  $t \mapsto X(x,t)$ for each $x$ that satisfy
\begin{equation}\label{eq:flowmap}
  X(x,t) = x + \int_0^t v(X(x,s),s)ds,\qquad  t\in [0,1].
\end{equation}
We will refer to the mapping $x \mapsto X(x, t)$ as the time-$t$ \textit{flow map} of the ODE.
The transport map obtained by
evaluating this flow map at the terminal time $t=1$ is denoted by
$T^v \coloneqq X(\cdot,1)$. Such parameterizations of the transport map guarantee invertibility and allow efficient computation of the Jacobian log-determinant $ \log \det \nabla_x T^v$ through the \textit{instantaneous-change-of-variables-formula}  \citep{chen2018neuralODE}.

\subsection{Challenges and motivations}

Though the dynamic approach to transport has demonstrated considerable success in applications, it is also understood that the extra freedom afforded by the time coordinate raises important challenges. \citet{finlay2020HowToTrain,onken2021OTFlow} observe that unfettered freedom in the dynamics can be practically detrimental: for ODE models trained through likelihood maximization, or any other criterion that depends only on a divergence between $T^v_\sharp\mu$ and $\nu$, the trajectory of the particles for $t < 1$ does \textit{not} directly impact the training objective. As a result, one often obtains highly irregular paths for $t \in [0,1]$ that produce large approximation and numerical integration errors.

For perspective, we note that for any absolutely continuous pair of measures $(\mu, \nu)$, there are in general infinitely many transport maps $T$ that achieve $T_\sharp \mu = \nu$; moreover, there are infinitely many velocity fields $v$ whose time-one flow maps \eqref{eq:flowmap} realize any given $T$. It is thus natural to ask what is the ``best'' velocity field in some relevant sense, e.g., from the perspective of approximation and statistical learning. While this general question remains very challenging, here we focus on its ``inner layer'' and ask:
\textit{Given a transport map $T$, what is the \emph{optimal} velocity field realizing $T$ as the time-one flow map of the ODE system \eqref{eq:ODE-geneal-form}?}
We shall define a specific notion of optimality, and a specific class of velocity fields over which we optimize, below.

A widely used strategy to address irregular paths is to seek \emph{straight-line} trajectories interpolating $x$ and $T(x)$ for some transport map $T$; this can be achieved through regularization of a log-likelihood training objective \citep{finlay2020HowToTrain,onken2021OTFlow, approximation_paper} or via methods that learn the velocity field directly via least squares \citep{Liu2022FlowStraight}. Straight-line trajectories are desirable in the sense that they minimize numerical integration errors; an explicit Euler scheme with a step size of one is exact if the trajectories are exactly straight. Yet such constructions essentially do \textit{not} exploit the time axis, and invite the question of why be dynamic in the first place. We also note that straight-line trajectories may not create space-time velocity fields $v(x,t)$ that are easy to approximate: in general, near regions of strong concentration or dilation, the velocity field realizing straight-line trajectories can have very large spatial derivatives. Consider, as a simple example, transport of a diffuse distribution at $t=0$ to a highly concentrated distribution at $t=1$: the Lipschitz constant of  $v(\cdot, t)$ will be very large near $t=1$. Figure~\ref{fig:Gaussianflow} (left) illustrates this phenomenon for two univariate Gaussians.
Simplicity in the Lagrangian picture (e.g., zero acceleration) does not translate to simplicity in the Eulerian sense. 

\begin{figure}[h]
\centering
      \includegraphics[trim=0 190 0 210,clip,width=2.8in]{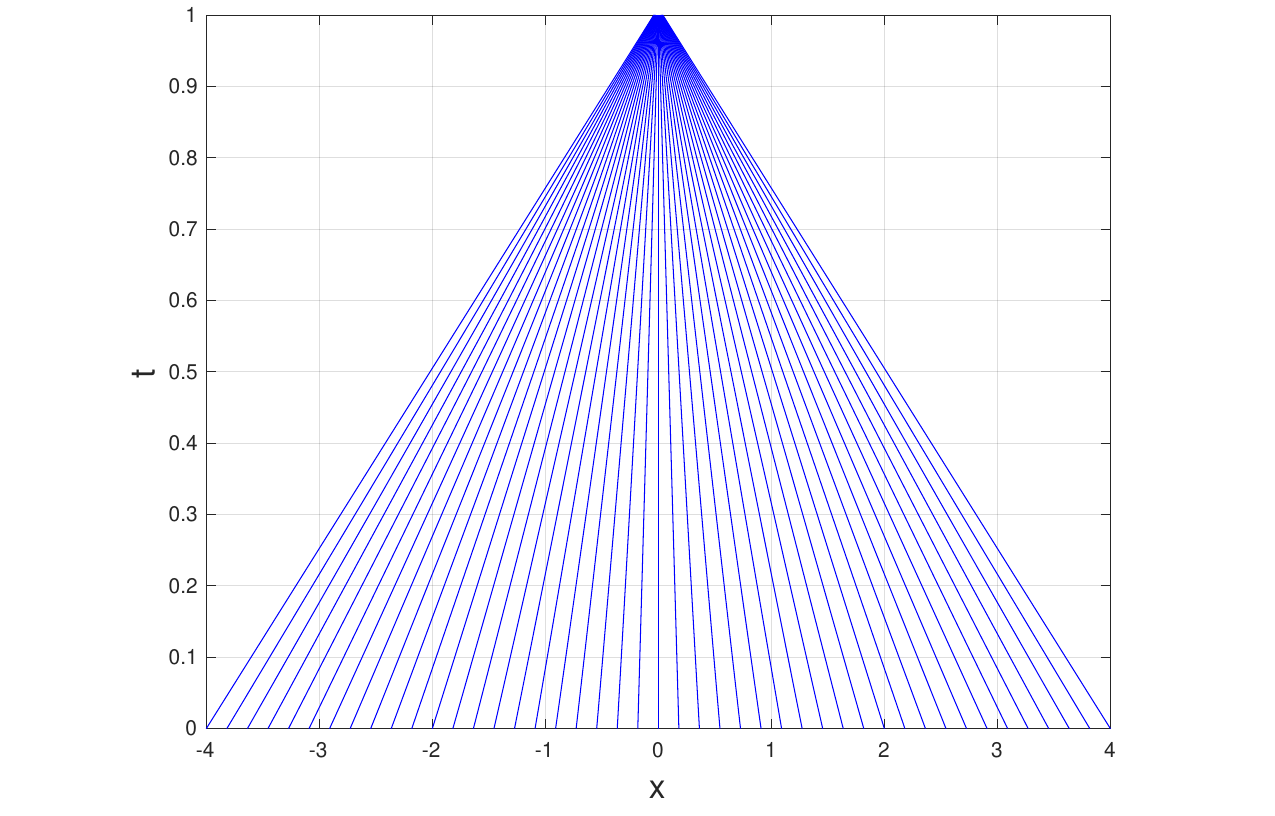}
      \includegraphics[trim=0 190 0 210,clip,width=2.8in]{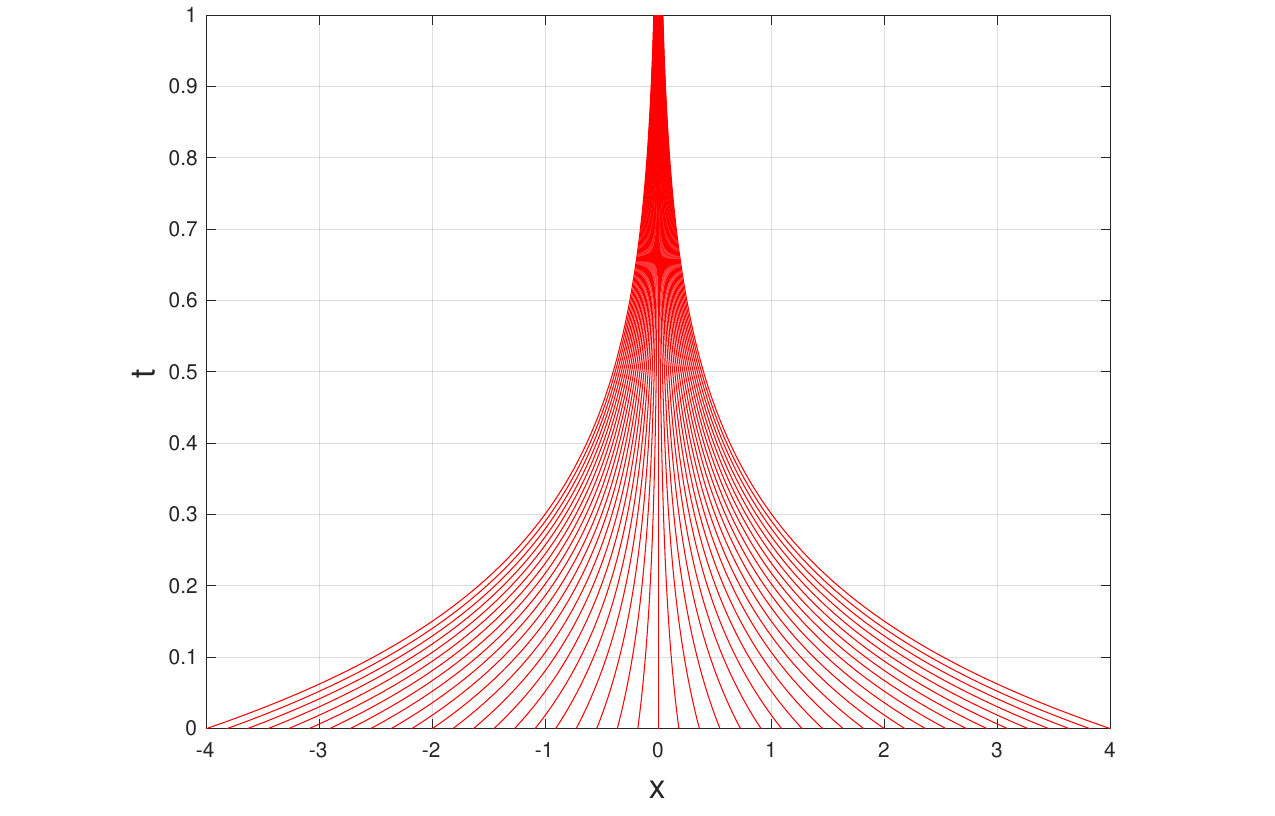}      
      \caption{Trajectories  $t \mapsto X(x, t)$ transporting $\mu = \mathcal{N}(0, \theta_1^2)$ at $t=0$ to $\nu = \mathcal{N}(0, \theta_2^2)$ at $t=1$, with $\theta_2 \ll \theta_1$.
       \textit{(left)}  Straight-line trajectories, with strong spatial derivatives $\partial_x v(x,t)$ apparent near $t=1$;  \textit{(right)} Trajectories produced by the optimal schedule derived in this work.}
  \label{fig:Gaussianflow}
   \end{figure}

The (spatial) Lipschitz constant of the velocity field $v$ plays a key role in stability analysis and approximation theory for ODE-based models: \citet{BentonErrorBoundsFlowMatching,StatisticalNeuralODE} show that distribution error in such models is controlled by the spatial Lipschitz constant of the underlying velocity field.  \citet{StatisticalNeuralODE, approximation_paper, /GuhringR21/NNapproximation1, /nn/Yarotsky17/NNapproximation2} additionally show that the error in a neural network approximation of any target velocity field $v$ is controlled by the spatial Lipschitz constant of the latter.

To make the question highlighted above precise, we will therefore focus on spatial Lipschitz constant of $v$, seeking to minimize it for any given transport $T$. We will do so without significantly increasing the small numerical integration error associated with straight-line trajectories (see Remark \ref{remark:approximation_error} below).
Our idea is intuitive. Given the one-parameter family of flow maps obtained by interpolating $T$ with the identity map, $X(\cdot, t): x \mapsto tT(x) + (1-t)x$ for $t \in [0,1]$, we introduce a \textit{schedule} $\tau: [0,1] \to [0,1]$ that is continuously differentiable and monotone increasing, while satisfying $\tau(0) = 0, \tau(1) = 1$.  The schedule captures {the rate at which we traverse time} and naturally defines a new family of flow maps $X_\tau(\cdot, t): x \mapsto \tau(t)\, T(x) + (1-\tau(t)) \, x$. This family satisfies $X_\tau(\cdot, 0) = \idmap$ and $X_\tau (\cdot, 1) = T$ and is indeed a dynamic transport scheme.
Now we consider velocity fields $v$ that realize $X_\tau(\cdot, t)$ via \eqref{eq:ODE-geneal-form}, and search over all admissible $\tau$ to find the schedule yielding the smallest spatial Lipschitz constant of $v$, uniformly over $t \in [0,1]$. Due to the chain rule, the velocity field that drives the dynamics is now scaled by $\dot{\tau}(t) \coloneqq \frac{d}{dt} \tau(t)$ at time $t$.
As we will show, the largest spatial Lipschitz constant of the original flow occurs at time $t=0$ and/or $t=1$, and a good $\tau$ slows down near these endpoints, effectively ``averaging out'' the spatial derivative over the entire time domain $[0,1]$. The trajectories traveled are still ``straight lines'' but with respect to $\tau$. %

To the best of our knowledge, the problem of finding schedules that maximize the {spatial regularity of a velocity field} has not yet been addressed in the literature. 
We will describe other related work below.

\subsection{Our contributions}
The main contributions of this work are as follows: 
\begin{itemize}
\item We formulate the \textit{optimal scheduling} problem: for linear interpolation of given a transport map $T$ and its realization via the flow map of an ODE system, we seek a schedule that minimizes a uniform (in time) bound on the spatial Lipschitz constant of the ODE velocity field $v$.
    
    \item We develop a variational approach to the problem of optimal schedules. Specifically, we show how one can tackle the problem of minimizing the desired Lipschitz bound by using the theory of $\Gamma$-convergence \citep{de1975tipo,dalmaso_gammaconvergence} to cast this non-differentiable problem as a limit of smooth objectives. Then we solve each problem by casting it in Lagrangian form and applying the direct method of the calculus of variations.  
    
    \item We show that the solution of the optimal scheduling problem satisfies an ODE that can be solved explicitly and in closed form. We then explore the properties of this optimal schedule $\tau_\infty$ in a series of examples and observe a universal behavior: $\tau_\infty$ has a \textit{sigmoid-like shape} with a unique inflection point $t_0 \in [0,1]$ and positive curvature in $[0,t_0]$.

    \item We show that, under the optimal schedule, the uniform Lipschitz bound decreases \emph{exponentially}, and discuss how this can dramatically improve the approximation power of flow models.
    
    \end{itemize}

While the focus of this paper is on the variational formulation and its analysis, we believe that the optimal schedules found here have many algorithmic implications. In practice, the transport map $T$ is unavailable; otherwise one could directly use $T$ for sampling and there is no need to learn its dynamical representation. As we will show, however, optimal schedules $\tau_\infty$ have a certain \textit{universal structure} that is independent of the underlying transport map $T$. As a result, they can be easily parameterized and made part of the learning problem. Moreover, the optimal schedules found in this work are \textit{independent of any training scheme}: the idea of introducing a time re-parameterized schedule can thus find use in  neural ODEs \citep{chen2018neuralODE}, flow matching \citep{/iclr/LipmanCBNL23/FlowMatching}, and stochastic interpolants \citep{/corr/abs-2209-15571/EricStochasticInterpolant}.

\subsection{Related work}
The problem of finding optimal schedules has drawn much attention lately, mostly in the context of \textit{diffusion models} \citep{/iclr/0011SKKEP21/DiffusionModels}. 
  Here we provide a brief review of relevant concurrent work. 

\citet{aranguri2025/OptScheduleInStochasticInterpolant} propose a time-dilated stochastic interpolant that effectively alters the noising schedule of a diffusion-based generative model; they show that the associated probability flow ODE can then be discretized in time with a non-uniform grid whose size is asymptotically independent of the dimension of the data-generating distribution. The notion of a noising schedule, while natural for diffusion paths, does not appear explicitly in our work. Moreover, the analysis in \citet{aranguri2025/OptScheduleInStochasticInterpolant} focuses on two-mode Gaussian mixture targets, in contrast with the much more general target distributions we consider here.

Again in the context of diffusion models, \citet{ScoreOptimalDiffusionSchedules} propose an algorithm for adapting the discretization of the backwards diffusion process, based on a notion of cost associated with each step. This cost is derived from a hypothetical predictor-corrector scheme for simulating the backwards diffusion equation; a discretization schedule that minimizes the sum of these costs, over all steps, is found iteratively. While the discretization found by this scheme is analogous to our $\tau$ applied to a uniform discretization of the unit time interval, the notion of cost is quite different from our Lipschitz objective.
\citet{icml/SabourFK24/OptScheduleinDiffusionModels} also consider the discretization of backwards diffusion paths, using numerical optimization over discretization points to minimize a certain upper bound on the Kullback--Leibler divergences between marginals of the continuous and discretized paths.

With the goal of solving stochastic optimal control problems, \citet{StochasticOptimalControlMatching1} propose a least squares-type objective that is analogous to a diffusion model loss function. To make the loss tractable, the authors introduce a family of matrix-valued time reparameterizations $M_t$ that are learned along with the control. The functions $t \mapsto M_t$ have the flavor of a generalized schedule, but no explicit form of the optimal choice is given, except to establish that it satisfies a complex integral equation.
\citet{StochasticOptimalControlMatching2} consider the reward fine-tuning of flow-based generative models, where, given a generative diffusion model that produces samples from some base distribution $\pi$ on $\R^d$, one wishes to modify it so that it samples from the tilted distribution $\pi \exp(r)$, for some reward function $r: \mathbb{R}^d \to \mathbb{R}$. The fine-tuning problem is recast as a stochastic optimal control problem, and the authors argue that a so-called ``memoryless'' noise schedule is necessary to achieve fine-tuning in this setting. Memorylessness turns out to be equivalent to a specific form of noise schedule for the forward process. While this schedule is given explicitly, the setting of diffusion models and the objective of fine-tuning are very different from those considered in this work.

Broadly, we comment that all these works have different notions of optimality, and none of them studies optimal schedules from the viewpoint of velocity approximation, as we do in this work.    

\section{Preliminaries}
\label{sec:preliminaries}

\subsection{Notation}
\label{subsec:notation}
Let $\N = \{ 1 , 2 , \dots \}$ be the set of positive integers and let $\N_\infty = \N \cup \{ \infty \}$.
For a closed subset $\Omega \subseteq \Rd$, we denote by $\mathcal{P}(\Omega)_{\textup{ac}}$ the set of Borel probability measures on $\Omega$ that are absolutely continuous with respect to the Lebesgue measure.
For $p \in \N$, we write $L^p(\Omega)$ for the space of measurable functions $f: \Omega \to \R$ such that $\int_{\Omega} |f|^p \, dx < \infty$, with $dx$ the Lebesgue measure. $L^\infty(\Omega)$ is then the space of measurable functions on $\Omega$ with finite essential supremum.
For $k \in \N$  and $p \in \N_\infty$, we write $W^{k,p}(\Omega)$ for the space of measurable functions $f: \Omega \to \R$ whose $k$-th order weak derivatives are in $L^p(\Omega)$; we also use the shorthand $H^k(\Omega) \coloneq W^{k,2}(\Omega)$.
Furthermore, for $k \in \N_\infty$, we write $C^k(\Omega)$ for the space of $k$-times continuously differentiable functions $f: \Omega \to \R$ whose $i$-th order derivatives, for $i \leq k$, are continuously extendable to the boundary $\partial \Omega$.  For $F: \Omega \to \Rd$
we write $F \in C^k(\Omega \, ; \Rd)$ to indicate that each component function $F_i$ of $F$ satisfies $F_i \in C^k(\Omega)$, for $1 \leq i \leq d$.
Finally, we introduce a few function spaces particular to our work. For $p \in \N$, let:
\begin{align}
    \label{eq:intro-T-2p}
    \mathcal{T}_{2p} &\coloneq \left\{ \tau \in W^{1,2p}([0,1]) \, : \, \tau(0) = 0 \, , \, \tau(1) = 1 \right\} \\
    \label{eq:intro-T-2p-b}
    \mathcal{T}_{2p}^{b} &\coloneq  \left\{ \tau \in \mathcal{T}_{2p} \, : \, 0 \leq \tau \leq 1 \right\} \\
    \label{eq:intro-T-infty}
    \mathcal{T}_{\infty} &\coloneq \left\{ \tau \in C^1([0,1]) \, : \, \tau(0) = 0 \, , \, \tau(1) = 1 \, , \, \dot{\tau} \geq 0 \right\}
\end{align}
where we note that \eqref{eq:intro-T-2p-b} is well defined due to the Sobolev embedding theorem; see, e.g., \citet[Theorem 5, Section 5.6]{evans2022partial}. These spaces satisfy $\mathcal{T}_\infty \subset \mathcal{T}_{2p}^{b} \subset \mathcal{T}_{2p}$ with strict inclusions.

\subsection{Sampling with measure transport}
\label{subsec:sampling-with-transport}
Now we present the main problem of our paper. Motivated by the discussion in Section~\ref{Introduction}, we consider two absolutely continuous measures $\mu, \nu \in \mathcal{P}_{\textup{ac}}(\Omega)$ supported on a  compact, convex domain $\Omega \subset \Rd$ and assume that there exists a measurable map $T: \Omega \to \Omega$ such that the pushforward of $\mu$ under $T$ is $\nu$, i.e., for all Borel sets $A \subset \Omega$ we have $\nu(A) = \mu(T^{-1}(A))$. The map additionally satisfies the following conditions:
\begin{assumption}
    \label{assumption:regularity-of-transport}
    For a compact and convex
    set $\Omega \subset \Rd$, the transport map $T: \Omega \to \Omega$ is of class $C^1(\Omega \, ; \Rd)$.
    Moreover, for all $s \in \Omega$, the Jacobian $\nabla T(s)$ is symmetric positive definite.
  \end{assumption}
For example, $T$ can be the optimal transport (Brenier) map induced by a strictly convex potential.

Now consider a path of measures with endpoints $\mu$ and $\nu$, given by $\mu_t = (T_t)_{\sharp} \mu$ for $t \in [0,1]$, with $T_t(x) \coloneq (1-t) x + t T(x)$. The one-parameter family of transport maps $x \mapsto T_t(x)$ is equivalently the time-parameterized flow map $x \mapsto X(x, t)$ of the following ODE:
\begin{equation}\label{eq:ODE}
    \begin{cases}
      \frac{d}{dt}X(x,t) &= T(x) - x,\qquad t\in [0,1],\\
      X(x,0) &= x.
    \end{cases}
  \end{equation}
Under the conditions of Assumption~\ref{assumption:regularity-of-transport}, \citet[Theorem 3.4]{approximation_paper} establishes the invertibility of $T_t$ for any $t \in [0,1]$. We can then use the chain rule and inverse function theorem to write the above ODE in the more standard form of \eqref{eq:ODE-geneal-form}, with a velocity field given by
\begin{equation*}
  v(\cdot, t) = (T - \idmap) \circ T_t^{-1}(\cdot) .
\end{equation*}

As discussed in Section~\ref{Introduction},  a key quantity from the perspective of approximation theory is the Lipschitz constant of $x \mapsto v(x,t)$, for each time $t$. We would like to control this quantity, $\Lip(v( \cdot , t))$, uniformly over $t \in [0,1]$.
By Proposition \ref{prop:lip-equiv-to-jacobian} in Appendix \ref{sec:appendix_T2}, we have:
\begin{equation}\label{eq:lip-jacobian}
    \Lip(v(\cdot, t)) = \sup_{x \in \Omega_t} \OperatorNorm{\nabla_x v(x, t)},
\end{equation}
where $\Omega_t = \{X(s,t) \, : \, s \in \Omega\}$ is the domain of $v(\cdot, t)$.
Hence, we want to control:
\begin{equation*}
    \Lambda = \sup_{t \in [0,1]} \sup_{x \in \Omega_t} \OperatorNorm{\nabla_x v( x, t)}  . 
  \end{equation*}
In the sequel, we will also make the following assumption:
\begin{assumption}
    \label{assumption:non-isometry}
    Let $(\sigma_j(s))_{j=1}^d$ be the eigenvalues of $\nabla T(s)$ at $s \in \Omega$. There exist $s_0 \in \Omega$ and $i$ such that $\sigma_i(s_0) \neq 1$.
\end{assumption}
This condition, albeit technical, makes the problem non-trivial by ruling out ``trivial'' transports, for example a translation of a distribution. 
Indeed, if $\sigma_i(s) = 1$ for all $i$ and $s \in \Omega$, then $\Lambda$ is already zero and there is no benefit in seeking an alternative schedule.
\subsection{Schedules}
\label{subsec:schedules}
Our goal is to reduce $\Lambda$ while staying on the trajectories prescribed by $T_t$.
A natural degree of freedom to exploit is the schedule $\tau: [0,1] \to [0,1]$ discussed in Section~\ref{Introduction}, which controls the ``rate'' at which we traverse time.
\begin{definition}
  A $C^1$ function $\tau: [0,1] \to [0,1]$ is called a \emph{schedule} if it is non-decreasing, with $\tau(0) = 0$ and $\tau(1) = 1$.
\end{definition}
Recall that the path of measures $\left( \mu_t \right)_{t \in [0,1]}$ is induced by flow maps $X(\cdot, t) \equiv T_t(\cdot)$.
A schedule naturally modifies the latter by setting $X_\tau(\cdot, t) \coloneqq X(\cdot, \tau(t))$.
The boundary conditions on $\tau$ ensure that $X_\tau(\cdot, 0) = \idmap$ and $X_\tau(\cdot, 1) = T$; hence, transport from $\mu = \mu_0$ to $\nu = \mu_1$ is still achieved. The ODE satisfied by this modified flow is
\begin{equation}\label{eq:ODE-schedule}
    \begin{cases}
      \frac{d}{dt}X_\tau(x,t) &= \dot{\tau}(t) \left( T(x) - x \right),\qquad t\in [0,1],\\
      X_\tau(x,0) &= x,
    \end{cases}
\end{equation}
and its associated velocity field is
\begin{align*}
    v_\tau(\cdot, t) &= \dot{\tau}(t) (T - \idmap) \circ (X_\tau)^{-1}(\cdot, t) \\
                     &= \dot{\tau}(t) (T - \idmap) \circ \left[ \, (1-\tau(t))\idmap + \tau(t)T \, \right]^{-1}, 
\end{align*}
defined on $\Omega_{\tau(t)}$.
We can, therefore, proceed as above to derive an expression for a modified uniform Lipschitz bound $\Lambda[\tau]$:
\begin{align}\label{eq:lip-jacobian-schedule}
    \Lambda[\tau] & \coloneq \sup_{t \in [0,1]} \sup_{x \in \Omega_{\tau(t)}} \, \dot{\tau}(t)  \OperatorNorm{\nabla_x v(x, \tau(t))} .
\end{align}
The key idea is that the velocity term $\dot{\tau}$ will allow us to suppress the spatial supremum locally in time.
The optimization problem that we will investigate is then
\begin{equation}\label{eq:opt-problem}
    \inf_{\tau \in \mathcal{T}_{\infty}} \Lambda[\tau].
\end{equation}
In the next section, we will show how to solve problem \eqref{eq:opt-problem} to optimality.

\section{Main Results}
\label{sec:main-results}

Given a transport map $T$ that pushes forward $\mu $ to $\nu$, with $\mu, \nu \in \mathcal{P}_{\text{ac}}(\Omega)$ and both $T$ and $\Omega$ satisfying Assumption~\ref{assumption:regularity-of-transport}, we seek an optimizer $\tau_\infty$ of problem \eqref{eq:opt-problem}, called the \textit{optimal schedule}. Our strategy is to put \eqref{eq:opt-problem} in the form $\inf_{\tau \in \mathcal{T}_{2}} \int_{0}^{1} \lambda(\tau(t), \dot{\tau}(t)) \, dt$ for a search space $\mathcal{T}_{2}$ (cf.\ \eqref{eq:intro-T-2p}) and a \textit{Lagrangian function} $\lambda: \R \times \R \to \R$. The space $\mathcal{T}_{2}$ strictly contains the space $\mathcal{T}_{\infty}$ defined in \eqref{eq:intro-T-infty}. A subsequent argument will establish that the minimizer $\tau_\infty \in \mathcal{T}_{2}$ is, in fact, in $\mathcal{T}_{\infty}$.

The first step is to reformulate problem \eqref{eq:opt-problem} in a more explicit manner, i.e., without appealing to the Lipschitz constant directly.
\begin{theorem}
    \label{thm:reformulate-opt-problem}
    Let $T$ satisfy Assumptions~\ref{assumption:regularity-of-transport} and \ref{assumption:non-isometry}, and let $(\sigma_i(s))_i$ denote the eigenvalues of $\nabla T(s)$, for $s \in \Omega$.
    Define $f(s) \coloneq \max_i \sigma_i(s) - 1$, $g(s) \coloneq \min_i \sigma_i(s) - 1$, and $\Omega_{[0,1]} \coloneq \Omega \times [0,1]$.
    Then $\Lambda[\tau]$ can be written as
    \begin{align}
        \label{eq:reformulated-Lambda}
        \Lambda[\tau] = \sup_{ (x,t) \in \Omega_{[0,1]} } |\dot{\tau}(t)| \, \max \Bigg\{ 
            \left| \frac{f(x)}{1 + \tau(t) f(x)} \right| \, , \,
            \left| \frac{g(x)}{1 + \tau(t) g(x)} \right| 
        \Bigg\}.
    \end{align}
\end{theorem}
The proof is in Appendix \ref{sec:appendix_T2}.

To proceed with the variational approach, we wish to cast \eqref{eq:reformulated-Lambda} as an integral of a Lagrangian. Appealing to Lemma \ref{lemma:exchange-max-sup}, we can rewrite \eqref{eq:reformulated-Lambda} as
\begin{equation*}
    \label{eq:reformulated-Lambda-swapped-max-sup}
    \Lambda[\tau] = \max \left\{
        \left\| \frac{\dot{\tau}(t) f(x)}{1 + \tau(t) f(x)} \right\|_{L^\infty(\Omega_{[0,1]})} \, , \, \left\| \frac{\dot{\tau}(t) g(x)}{1 + \tau(t) g(x)} \right\|_{L^\infty(\Omega_{[0,1]})}
    \right\} \, .
\end{equation*}
The key observation here is that an $L^\infty$ norm is the limit of $L^p$ norms.
In fact, this approximation can be iterated in the sense that 
\begin{equation*}
    \max \left\{ \norm{F}_{L^\infty}, \norm{G}_{L^\infty} \right\} = \lim_{p \to \infty} \left(  \norm{F}_{L^p}^p + \norm{G}_{L^p}^p \right)^{\frac{1}{p}}
\end{equation*}
for functions $F, G \in C^0(\Omega_{[0,1]})$; see Lemma \ref{lemma:Linfty-to-Lp-appendix} for a proof.
It is therefore natural to study the problems
\begin{equation}
    \label{eq:opt-problem-lagrangian-with-p}
    \inf_{\tau \in \mathcal{T}_{\infty}} \int_{0}^{1} \lambda_p(\tau(t), \dot{\tau}(t)) \, dt \, ,
\end{equation}
where
\begin{align}
  \label{eq:lambda-p-def}
    \lambda_p(\tau(t), \dot{\tau}(t)) \coloneq \int_{\Omega} 
    \left( \frac{\dot{\tau}(t) f(x)}{1 + \tau(t) f(x)} \right)^{2p} \, +
    \left( \frac{\dot{\tau}(t) g(x)}{1 + \tau(t) g(x)} \right)^{2p} \, dx \, .
\end{align}
Notice that for each $\tau \in \mathcal{T}_{\infty}$ we have the pointwise limit $\Lambda[\tau] = \lim_{p \to \infty} \left( \int_{0}^{1} \lambda_p(\tau(t), \dot{\tau}(t)) \, dt \right)^{\frac{1}{2p}}$
and the monotonicity of $x \mapsto x^{\frac{1}{2p}}$ allows us to optimize under the exponent. 

While this relaxation is intuitive, pointwise convergence of objectives does \textit{not} guarantee convergence of minimizers.
Understanding the latter is exactly the aim of $\Gamma$-convergence, a powerful tool in the calculus of variations. 
For a general topological space $X$, a functional $F: X \to \R$ and a family of functionals $(F_p)_{p \in \N}$ with $F_p: X \to \R$, one says that $F_p$ $\Gamma$-converges to $F$ if for any $x \in X$ we have \citep{braides2006handbook}
\begin{align*}
    & F(x) \leq \liminf_{p \to \infty} F_p(x_p) \text{ for any sequence } x_p \to x, \ \text{and} \\
    & F(x) \geq \limsup_{p \to \infty} F_p(x_p) \text{ for some sequence } x_p \to x \, .
\end{align*}
If $\Gamma$-convergence is established, together with a so-called \textit{equi-coercivity}\footnote{
    A family of functionals $F_p$ is said to be equi-coercive \citep{braides2006handbook} if for any $t > 0$ there is a compact set $K_t \subset X$ such that for all $p$ we have $\left\{ x \in X : F_p(x) \leq t \right\} \subset K_t$.
} condition, one can deduce that subsequential limits of minimizers of each $F_p$ are minimizers of $F$. This is the celebrated \textit{fundamental theorem of $\Gamma$-convergence} \citep{braides2006handbook}.
Using this machinery, we prove the following result.
\begin{theorem}
    \label{thm:Linfty-to-Lp-main-text}
    Let $\mathcal{T}_2$ be as in \eqref{eq:intro-T-2p} with $p=1$.
    For arbitrary $p \in \N$, consider the optimization problem
    \begin{equation}
        \label{eq:opt-problem-Lp-approx-in-theorem}
        \inf_{\tau \in \mathcal{T}_{2}} \int_{0}^{1} \, \lambda_p \left (\tau(t), \dot{\tau}(t) \right ) \, dt  \, ,
    \end{equation}
    with $\lambda_p$ defined in \eqref{eq:lambda-p-def}.
    If $\tau_p$ is a minimizer of \eqref{eq:opt-problem-Lp-approx-in-theorem}, then any subsequential $L^2$ limit $\tau_{p_j} \to \tau_\infty$ as $j\to\infty$ is a minimizer of \eqref{eq:opt-problem}.
  \end{theorem}
The proof is in Appendix \ref{subsec:proof-of-Linfty-to-Lp}. As discussed above, the key idea is to establish that the functionals
\begin{align*}
    \Lambda_p :  \mathcal{T}_{\infty} & \to \R \\
    \tau & \mapsto \left( \int_{0}^{1} \, \lambda_p(\tau(t), \dot{\tau}(t)) \, dt \right)^{\frac{1}{2p}}
\end{align*}
$\Gamma$-converge to $\Lambda$, and that the family $(\Lambda_p)_{p \in \N}$ is equi-coercive. Notice that in our previous discussion of $\Gamma$-convergence, the topology of the underlying space $X$ was intentionally left vague: too strong of a topology
will break the coercivity condition, whereas too weak of a topology will make verifying $\Gamma$-convergence impossible. Here, the balance is struck by expanding the space of feasible schedules to a closed subset of the Sobolev space $H^{2}([0,1])$ and endowing
this subset
with the topology of weak convergence (or, equivalently, the weak topology). Inside sets of bounded $H^1$ norm,
this topology is equivalent to the strong $L^2$ topology; hence the conclusion of the theorem.

Having established the validity of our approximation, we now solve the approximating problem \eqref{eq:opt-problem-Lp-approx-in-theorem} by employing calculus of variations.
\begin{theorem}
    \label{thm:characterizarion-of-Lp-solutions-main-text}
    There exists a minimizer $\tau_p$ of \eqref{eq:opt-problem-Lp-approx-in-theorem} that satisfies the following ODE:
    \begin{align}
        \label{eq:ODE-solution-of-Lp-problem-in-theorem}
        \frac{d}{dt} \tau_p(t) = \frac{1}{Z_p} \Bigg( 
        \int_{\Omega} 
        \frac{f(s)^{2p}}{\left (1 + \tau_p(t) f(s) \right)^{2p}} \, + 
        \frac{g(s)^{2p}}{\left (1 + \tau_p(t) g(s) \right)^{2p}} \, ds 
        \Bigg)^{-1/(2p)} \, ,
    \end{align}    
    with boundary conditions $\tau_p(0) = 0$ and $\tau_p(1) = 1$ and
    \begin{equation*}
        Z_p = \int_{0}^{1} \left( \int_{\Omega} \left( \frac{f(s)}{1 + \tau_p(t) f(s)} \right)^{2p} \, ds + \int_{\Omega} \left( \frac{g(s)}{1 + \tau_p(t) g(s)} \right)^{2p} \, ds \right)^{-\frac{1}{2p}} \, dt.
    \end{equation*}
\end{theorem}
Proving the above theorem (see Appendix \ref{subsec:solving-the-Lp-ODE-variational-caculus}) is a mostly technical endeavor, i.e., verifying that the Lagrangian $\lambda_p$ is tame enough so that the minimizers of \eqref{eq:opt-problem-Lp-approx-in-theorem} satisfy the (strong) Euler--Lagrange equations.
The difficulty lies in proving that minimizers are, \textit{a priori}, in $C^2([0,1])$; standard regularity theory does not apply.
To overcome this, we further approximate $\lambda_p$ by the family $\lambda_p^\epsilon(x, v) = \lambda_p(x,v) + \epsilon \, v^2$, which possess smooth minimizers, and take another $\Gamma$-limit along $\epsilon \to 0$.

As a result, if we find \textit{some} $L^2$ subsequential limit of the solutions $\tau_p$ of the ODEs \eqref{eq:opt-problem-Lp-approx-in-theorem}, we will find a solution to the original problem \eqref{eq:opt-problem}.
One is tempted to take the (pointwise) limit $p\to\infty$ of the right hand side of \eqref{eq:ODE-solution-of-Lp-problem-in-theorem}  to obtain a new ODE. This ODE in fact characterizes a solution of \eqref{eq:opt-problem}.
\begin{theorem}
    \label{thm:solution-of-optimal-problem-is-solution-of-L-inf-ODE-main-text}
    Let $\tau_\infty$ be the solution to the following ODE,
    \begin{align}
        \label{eq:ODE-solution-of-optimal-problem-in-theorem}
        \frac{d}{dt} \tau_\infty(t) = \frac{1}{Z} \max \Bigg\{ 
        \left\| \frac{f(s)}{1 + \tau_\infty(t) f(s)} \right\|_{L^\infty},
        \left\| \frac{g(s)}{1 + \tau_\infty(t) g(s)} \right\|_{L^\infty} 
        \Bigg\}^{-1} \, ,
    \end{align}    
    with boundary conditions $\tau_\infty(0) = 0$ and $\tau_\infty(1) = 1$ and
    \begin{equation*}
        Z = \int_{0}^{1} \max \Bigg\{
        \left\| \frac{f(s)}{1 + \tau_\infty(t) f(s)} \right\|_{L^\infty},
        \left\| \frac{g(s)}{1 + \tau_\infty(t) g(s)} \right\|_{L^\infty}
        \Bigg\}^{-1} \, dt \, .
    \end{equation*}
    If $\tau_p$ are solutions to the ODEs \eqref{eq:opt-problem-Lp-approx-in-theorem}, then there is a subsequence $\tau_{p_j} \to \tau_\infty$ in $L^2$ and therefore, by Theorem \ref{thm:Linfty-to-Lp-main-text}, $\tau_\infty$ is optimal for \eqref{eq:opt-problem}.
\end{theorem}
To prove the above theorem (see Appendix \ref{subsec:appendix-subseq-convergence-in-L2}), one applies a Gr\"onwall inequality \citep{howard1998gronwall} to control the difference between solutions to the ODEs \eqref{eq:ODE-solution-of-Lp-problem-in-theorem} and \eqref{eq:ODE-solution-of-optimal-problem-in-theorem}.
The key idea, and admittedly somewhat overpowered in the context of this proof, is to combine this with the Arzelà--Ascoli theorem, allowing one to obtain uniform convergence of the right hand side of \eqref{eq:ODE-solution-of-Lp-problem-in-theorem} to that of \eqref{eq:ODE-solution-of-optimal-problem-in-theorem}.
This result, combined with the Gr\"onwall estimate, gives the required convergence.

Somewhat remarkably, the ODE \eqref{eq:ODE-solution-of-optimal-problem-in-theorem} is solvable in closed form:
\begin{theorem}
    \label{thm:solution-of-L-inf-ODE-main-text}
    Let $f^* \coloneq \sup_{s \in \Omega} f(s)$, $g_* \coloneq \inf_{s \in \Omega} g(s)$ and consider
    \begin{align}
        \label{eq:transition-time-in-theorem}
        & t_0 = \frac{\ln \left[ \frac{1}{2} \left( 1 - \frac{f^*}{g_*} \right) \right]}{\ln \left[ \frac{1}{4} \left(2 - \frac{ f^* }{ g_* } - \frac{g_*}{f^*}  \right) \right] - \ln \left( g_* + 1 \right) }, \\
        \label{eq:transition-point-in-theorem}
        & \tau(t_0) = - \frac{1}{2}\left(\frac{1}{f^*} + \frac{1}{g_*}\right) .
\end{align}
    If $0 \leq t_0 \leq 1$, then the solution to the ODE \eqref{eq:ODE-solution-of-optimal-problem-in-theorem} is given by:
    \begin{equation}
        \label{eq:solution-of-L-inf-ODE-main-text}
        \tau_\infty(t) =
        \begin{cases}
            {\frac{1}{f^*} }\left\{ {\frac{1}{4 (g_* + 1)}} \left( 2 -  \frac{f^*}{g_*} - \frac{g_*}{f^*} \right) \right\}^t - \frac{1}{f^*}, &\hspace{-2.1mm} t \leq t_0 \\[1.2em]
            \frac{1}{2} \left( \frac{1}{g_*} {-} \frac{1}{f^*} \right) 
            \hspace{-1mm}\left\{ \frac{2 (g_* + 1) }{(1 - \frac{g_*}{f^*})} \right\}^t \hspace{-1mm}
            \left\{ \frac{1 - \frac{g_*}{f^*}}{2} \right\}^{1-t} \hspace{-2mm}{-} \frac{1}{g_*}, &\hspace{-2.1mm}t \geq t_0 \, .
        \end{cases}
    \end{equation}
    Otherwise, if $t_0 \notin [0,1]$, we have two cases: if $f^* \geq -g_*$ then the solution to \eqref{eq:ODE-solution-of-optimal-problem-in-theorem} is given by
    \begin{equation}
    \label{eq:solution-of-L-inf-ODE-simple-form-main-text}
    \tau_\infty(t) = \frac{ (f^* + 1)^t - 1 }{f^*} \, ,
    \end{equation}
    and if $f^* < -g_*$, then the solution to \eqref{eq:ODE-solution-of-optimal-problem-in-theorem} is given by
    \begin{equation}
    \label{eq:solution-of-L-inf-ODE-simple-form-2-main-text}
    \tau_\infty(t) = \frac{ (g_* + 1)^{t} - 1 }{g_*} \, .
    \end{equation}
\end{theorem}
The proof of this theorem (see Appendix \ref{appendix:solving-the-L-infty-ODE}) is a cumbersome computation. The key observation is that the infinity norms inside the $\max$ in \eqref{eq:ODE-solution-of-optimal-problem-in-theorem} can be computed explicitly in terms of the data of the problem. Once this is done, 
we notice that the $\max$ has one or zero transition points, in the sense that there is at most one $t_0 \in [0,1]$ where the two functions inside the $\max$ are equal.
We can then solve this problem explicitly by solving two ODEs, one in $[0,t_0]$ and one in $[t_0,1]$, while matching the solutions at $t_0$.

Due to Theorem \ref{thm:solution-of-L-inf-ODE-main-text}, the optimal schedule $\tau_\infty$ is available in closed form given tight upper\slash lower bounds on the largest\slash smallest eigenvalues of the Jacobian of the transport map $T$, respectively.
Intuitively, these quantities capture the worst-case dilation and concentration associated with the transport.
If $T$ is available numerically, then obtaining these values is relatively straightforward via differentiation and optimization, but that is a trivial case.
More interestingly,
if $T$ is a Brenier map, then one can estimate the largest and smallest eigenvalues of the Hessian of the convex potential $\phi$ satisfying $T = \nabla \phi$. 
It is common to assume there are constants $\alpha, \beta$ such that
\begin{equation}
\label{eq1}
\alpha I_d \preceq \nabla^2 \phi(s) \preceq \beta I_d, \quad \forall s \in \Omega \, ,
\end{equation}
i.e., that the potential is strongly convex and smooth.
This yields $f^* \leq \beta - 1$ and $g_* \geq \alpha - 1$, leading to an estimate %
of the optimal schedule; of course, if \eqref{eq1} is tight then the estimate is exact. 

Finally, we can substitute our optimal schedule $\tau_\infty$ in the objective functional $\Lambda[ \, \cdot \,]$ to compute its optimal value.
To make notation more transparent, set
\begin{align*}
    \sigma_{\min}^* \coloneq \inf_{s \in \Omega} \sigma_{\min}(s) \quad \textup{ and } \quad
    \sigma_{\max}^* \coloneq \sup_{s \in \Omega} \sigma_{\max}(s) \, ,
\end{align*}
noting that that $\sigma^*_{\max} = f^* + 1$ and $\sigma^*_{\min} = g_* + 1$.
Then we have:
\begin{theorem}
    \label{thm:non-asymptotic-result-main-text}
    For $\Lambda$ as in \eqref{eq:lip-jacobian-schedule} and the trivial schedule $\bar{\tau}: t \mapsto t$, 
    \begin{equation*}
        \Lambda[\bar{\tau}] = \max \left\{ \sigma_{\max}^* - 1 \, , \, \frac{1 - \sigma^*_{\min}}{\sigma^*_{\min}} \right\} \, .
    \end{equation*}
    For the optimal schedule $\tau_\infty$, i.e., the minimizer of problem \eqref{eq:opt-problem} given in Theorem~\ref{thm:solution-of-optimal-problem-is-solution-of-L-inf-ODE-main-text}, we have the following cases. If there exists $t_0 \in [0,1]$ such that $\tau_\infty(t_0) = - \frac{1}{2}\left( \frac{1}{\sigma_{\max}^* - 1} + \frac{1}{{\sigma_{\min}^* - 1}} \right)$, then
    \begin{align}
        \label{eq:generic-formula-non-asymptotic-result}
      \Lambda[\tau_\infty] = & \ln\left[\frac{\sigma^*_{\max}-1}{\sigma^*_{\min}}\right] + 
                        \ln\left[\frac{1}{4}\left(\frac{1}{1 - \sigma_{\min}^*} +\frac{1 - \sigma_{\min}^*}{(\sigma^*_{\max}-1)^2} + \frac{2}{\sigma_{\max}^* -1} \right)\right] \, .
    \end{align}
    Otherwise, we have
    \begin{equation*}
        \Lambda[\tau_\infty] = 
        \begin{cases}
            \ln \left( \sigma^*_{\max} \right) & \text{if } \sigma_{\max}^* + \sigma_{\min}^* \geq 2 \\
            -\ln \left( \sigma^*_{\min}\right) & \text{if } \sigma_{\max}^* + \sigma_{\min}^* < 2 \, .
        \end{cases}
    \end{equation*}
  \end{theorem}
  
This result (see proof in Appendix \ref{subsec:exponential-Lipschitz-improvement}) tells us that, for large $\sigma_{\max}^*$ and small $\sigma_{\min}^*$, the optimal uniform Lipschitz bound $\Lambda[\tau_\infty]$ is \emph{exponentially smaller} than the bound corresponding to the trivial schedule $\Lambda[\bar{\tau}]$, up to corrections of order one.
More precisely, for real-valued functions $\alpha, \beta: \R^2 \to \R$, let $\alpha \asymp \beta$ denote that there exist constants $c_1, c_2 > 0$ and some $y \in \R^2$ such that $c_1 \beta(x) \leq \alpha(x) \leq c_2 \beta(x)$ for all $x_i \geq y_i$ and $i \in \{1,2\}$. 
The Lipschitz bound then scales as follows:
\begin{corollary}
    \label{corr:Lipschitz-constant-improvement}
    Viewing $\Lambda[\tau_\infty]$ and $\Lambda[\bar{\tau}]$ as functions of $\sigma^*_{\max}$ and $1/\sigma^*_{\min}$, we have
    \begin{align*}\label{eq:LipTrivialSchedule}
        \Lambda[\bar{\tau}]  &\asymp \max \left\{ \sigma_{\max}^* \, , \, \left( \sigma_{\min}^* \right)^{-1} \right\} \\
        \Lambda[\tau_\infty] &\asymp \ln \sigma^*_{\max} - \ln \sigma^*_{\min} .
    \end{align*}
\end{corollary}

\smallskip

\begin{remark}
\label{remark:approximation_error}
This result has important consequences for control of the error in the target measure $\nu$ when the velocity is learned from data or otherwise approximated. In fact, we claim that the standard error bounds become exponentially tighter. 
Consider an approximate velocity  $(x,t) \mapsto \hat{v}(x,t)$ that is $\epsilon$-close to the true velocity field $(x,t) \mapsto v(x,t)$ in a $C^0$ sense. By \citet[Corollary 5.2]{approximation_paper}, the resulting distribution approximation error in the Wasserstein-2 metric is bounded by
$\epsilon \exp(\sup_{t \in [0,1]} L_t)$ where $L_t = \mathrm{Lip} (v(\cdot, t) )$ is the spatial Lipschitz constant of $v$. Thus, if we re-parameterize time with the optimal schedule $\tau_\infty$, the bound on the distribution error becomes $\mathcal O\left(\frac{\sigma^*_{\max}}{\sigma^*_{\min}}\epsilon\right)$ instead of $\mathcal O\left(\exp\left(\max \left\{ \sigma_{\max}^* \, , \, \left( \sigma_{\min}^* \right)^{-1} \right\}\right)\epsilon\right)$  for the trivial schedule.
On the other hand, while the straight-line trajectories obtained with the trivial schedule have essentially zero numerical integration error, we claim that the price we pay in numerical error due to non-straight trajectories is small compared to the improvement in the approximation error. Consider approximating the ODE flow map $x \mapsto X_{\tau_\infty}(x, 1)$
with a forward Euler discretization of step size $h$. The numerical integration error 
is then bounded by $\frac{hM}{2\Lambda[\tau_\infty]}\left(e^{\Lambda[\tau_\infty]} - 1\right)$ \citep{Book:NumericalAnalysis}, where $M = \sup_{x\in\Omega}\sup_{t\in[0, 1]} \left|\frac{d^2}{dt^2}X_{\tau_\infty}(x,t)\right| = \sup_{x\in\Omega}\sup_{t\in[0, 1]}\ddot{\tau}_\infty(t)|T(x) - x| = \mathcal O (1)$, if we assume the  eigenvalues are uniformly upper and lower bounded.
Using the triangle inequality, it is not hard to see that combined distribution approximation error in a Wasserstein-$2$ sense is bounded by \citep[Theorem $2$]{Sagiv2019TheWDWasserstainErrorBoundsPushForward} \[\mathcal O\left(\left(\frac{\sigma^*_{\max}}{\sigma^*_{\min}}\right)\epsilon + \left(\frac{\sigma^*_{\max}}{\sigma^*_{\min}\left( \ln \sigma^*_{\max} - \ln \sigma^*_{\min} \right)}\right)h\right),\]
which is still exponentially tighter than the combined error with the trivial schedule. We believe this heuristic argument justifies working with velocity fields that are constant speed with respect to $\tau$,
as they exponentially lower the Lipschitz constant while maintaining good control over the integration error.
\end{remark}

\begin{remark}
    \label{remark:contatn-sup-and-ansatz}
    In this paper, we have solved the optimization problem \eqref{eq:opt-problem} by variational calculus. One might instead try treating the problem directly, especially with foresight of the solution. For example, we notice that the optimizer of problem \eqref{eq:opt-problem}, as given in Theorem \ref{thm:solution-of-L-inf-ODE-main-text}, eliminates time dependence from the argument to the supremum in equation \eqref{eq:reformulated-Lambda}. Thus we might have ``guessed'' \emph{a priori} that the optimal solution satisfies
    \begin{equation*}
        |\dot{\tau}(t)| \, \sup_{ x \in \Omega }  \, \max \Bigg\{ 
            \left| \frac{f(x)}{1 + \tau(t) f(x)} \right| \, , \,
            \left| \frac{g(x)}{1 + \tau(t) g(x)} \right| 
        \Bigg\} = Z, 
    \end{equation*}    
for some constant $Z \in \R$, which anticipates the form of the ODE in Theorem \ref{thm:solution-of-optimal-problem-is-solution-of-L-inf-ODE-main-text}.  This link can be seen by using Lemma \ref{lemma:exchange-max-sup} to exchange the $\sup$ and $\max$ in \eqref{eq:reformulated-Lambda} and then Lemma \ref{lemma:triviality-of-suprema} to compute the $\sup$, for any $t \in [0,1]$. Writing down such an ansatz for the minimizer, however, does not establish its optimality.
\end{remark}

\section{Examples}
\label{sec:examples}
In this section, we present analytical examples of  optimal schedules $\tau_\infty$. In the process, qualitative features of $\tau_\infty$ are revealed.
\subsection{Univariate Gaussian distribution}
\label{subsec:univariate-gaussian}
Perhaps the  simplest  example to consider is transport between univariate Gaussian distributions\footnote{See Appendix \ref{sec:appendix-b} for a detailed discussion.} $\mu_1 = \mathcal{N}(\mu_1, \theta_1^2)$ and 
$\mu_2 = \mathcal{N}(\mu_2, \theta_2^2)$. The optimal transport map between $\mu_1$ and $\mu_2$ is linear \citep{chewi2024statistical}:
\begin{equation*}
    T(x) = \mu_2 + \frac{\theta_2}{\theta_1}(x - \mu_1) .
\end{equation*}
First, notice that working in one dimension means that for all $s \in \Omega$ we have $\sigma_{\max}(s) = \sigma_{\min}(s) = \theta_2 / \theta_1$.
Thus we must have $f^* = g_*$, meaning that the transition time is $t_0 = - \infty$. 
Substituting $f^* = g_* = \theta_2 / \theta_1  - 1 $ into \eqref{eq:solution-of-L-inf-ODE-simple-form-main-text}, we obtain
\begin{equation*}
    \tau_\infty(t) = \frac{ \left ( \frac{\theta_2}{\theta_1 }\right )^t - 1 }{\frac{\theta_2}{\theta_1} - 1} .
\end{equation*}
Notice that the boundary conditions are satisfied, i.e., $\tau_\infty(0) = 0$ and $\tau_\infty(1) = 1$, and, moreover, the concavity or convexity of the curve depends on whether $\theta_1 > \theta_2$ or $\theta_1 < \theta_2$, respectively.
This is intuitive: if $\theta_1 > \theta_2$ then we are mapping a wider Gaussian to a narrower one, causing particles traveling along the flow to concentrate at $t=1$, making the spatial Lipschitz constant largest at that time.
The optimal schedule \textit{slows} precisely near $t=1$ to suppress the increased Lipschitz constant.
Symmetrically, when $\theta_1 < \theta_2$ particles \say{explode} at small times $t \approx 0$ and the optimal schedule slows near $t=0$, tempering this behavior.
Finally, for $\theta_1 \nearrow \theta_2$,\footnote{The limit is needed since the case $\theta_1 = \theta_2$ is not covered by our framework, which assumes that our transformation $T$ is \textit{not} an isometry. } one can write $\theta_2 / \theta_1 = 1 + \epsilon$ with $\epsilon \searrow 0$ and Taylor expand to obtain $\tau_\infty(t) = \frac{ (1 + \epsilon \, t + o(\epsilon) ) - 1 }{ \epsilon } = t + \frac{o(\epsilon)}{\epsilon} \to t \quad \textup{as} \quad \epsilon \searrow 0$.
In other words, as $T$ tends to an isometry (a translation in this case) the optimal schedule tends to the trivial schedule.
Finally, one may compute:
\begin{equation*}
    \Lambda[t \mapsto t] = \max\left\{ \frac{\theta_2}{\theta_1} - 1 \, , \, \frac{\theta_1}{\theta_2} \right\} \quad \text{ and } \quad \Lambda[\tau_\infty] = \left| \ln \left( \frac{\theta_2}{\theta_1} \right) \right|\, ,
\end{equation*}
making the exponential improvement manifest.
Paths following the trivial schedule and the optimal schedule are illustrated in Figure~\ref{fig:Gaussianflow}, and an example of the optimal schedule is plotted in Figure~\ref{fig:schedules} (left).

\begin{figure}[h]
\centering
      \includegraphics[trim=0 195 0 210,clip,width=2.6in]{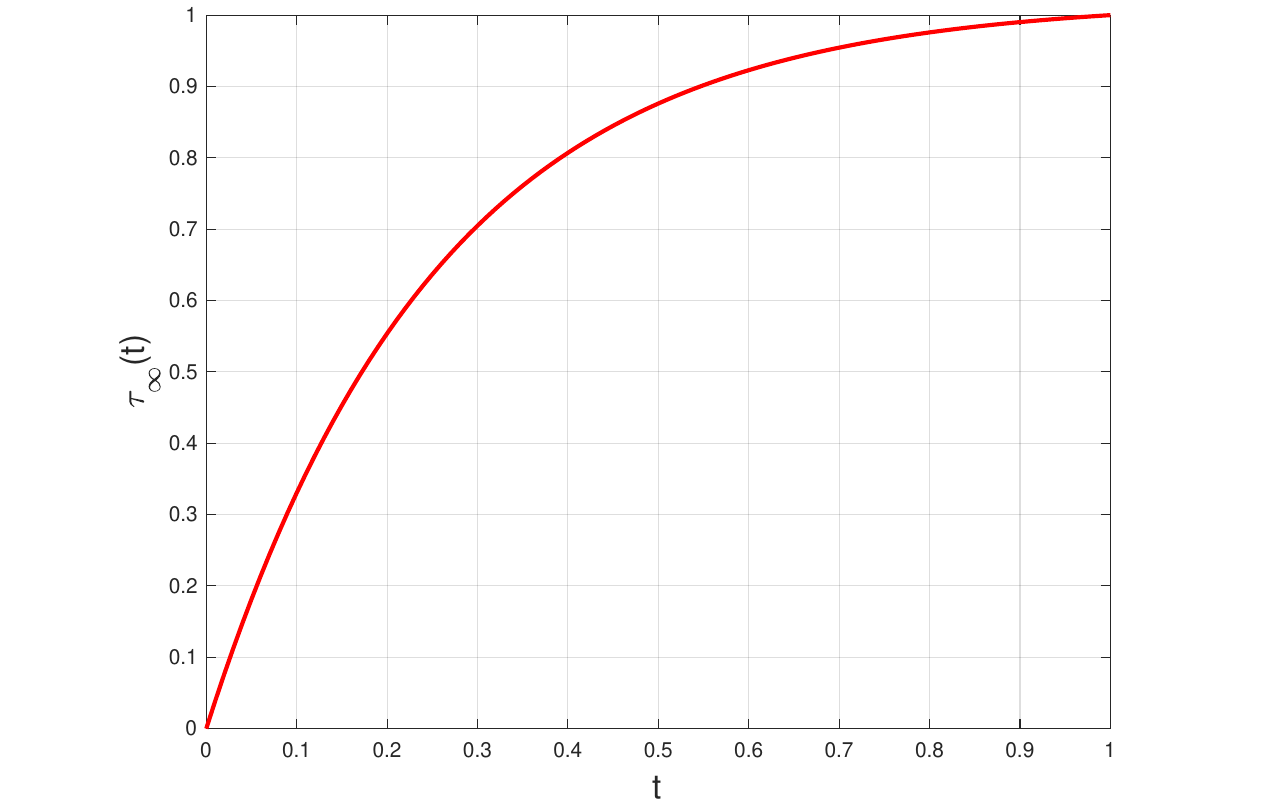}
      \includegraphics[trim=0 195 0 210,clip,width=2.6in]{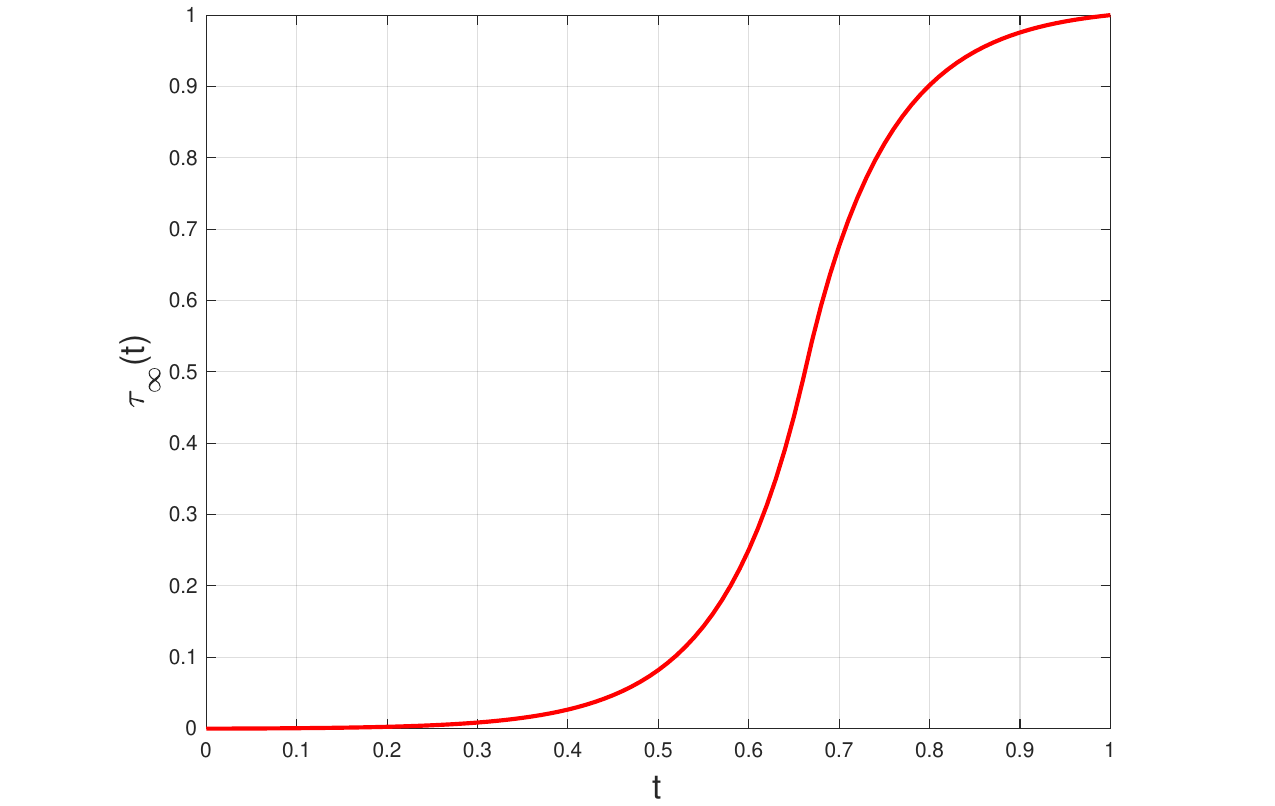}      
    \caption{Optimal schedules $\tau_\infty$: \textit{(left)} for the Gaussian example of Figure~\ref{fig:Gaussianflow} and Section~\ref{subsec:univariate-gaussian}, where $\theta_2 \ll \theta_1$; and \textit{(right)} for the Gaussian mixture example given at the end of Section~\ref{subsec:multi-modal-distributions}.}
  \label{fig:schedules}
\end{figure}

\subsection{Generic distributions $\mu$ and $\nu$ and ill-conditioned $T$}
\label{sec:illconddists}
For any $\mu, \nu \in \mathcal{P}_{\textup{ac}}(\Omega)$, let $T$ be such that $\sigma^*_{\max} \gg 1$ and $\sigma^*_{\min} \ll 1$.
This type of ill-conditioned transport is, in some sense, exactly what we aim to address in this paper.
To make notation more explicit, write:
\begin{align*}
    M & = \sigma_{\max}^* \coloneqq \sup_{s \in \Omega} \sigma_{\max}(s) \gg 1 \quad \text{ and } \quad \epsilon = \sigma_{\min}^* \coloneqq \inf_{s \in \Omega} \sigma_{\min}(s) \ll 1 \, .
\end{align*}
First, we will attempt to understand the qualitative behavior of the optimal schedule $\tau_\infty$.
Looking at \eqref{eq:solution-of-L-inf-ODE-main-text} and recalling that $f^* = \sigma_{\max}^* - 1$ and $g_* = \sigma_{\min}^* - 1$, we have:
\begin{equation*}
    \tau_\infty(t) = 
    \begin{cases}
        \frac{1}{M-1} \left\{ \frac{1}{4\epsilon} \left( \frac{M-1}{1-\epsilon} + \frac{1-\epsilon}{M-1} + 2 \right) \right\}^t - \frac{1}{M-1} \, , \, & t \leq t_0,\\
        -\frac{1}{2} \left( \frac{1}{M-1} + \frac{1}{1-\epsilon} \right) \left\{ \frac{\epsilon}{ \frac{1}{2}(1 + \frac{1 - \epsilon}{ M - 1} )} \right\}^t \, \left\{ \frac{1 + \frac{1 - \epsilon}{M-1}}{2} \right\}^{1-t} + \frac{1}{1-\epsilon} \, , \, & t > t_0 \, .
    \end{cases}
\end{equation*}
Now considering the double limit $\frac{1}{\epsilon}, M \to \infty$ and using Landau notation as needed, we have:
\begin{equation*}
    \tau_\infty(t) = 
    \begin{cases}
        \frac{1}{M} \left( e^{ t \ln \left( \frac{3}{4} \frac{M}{\epsilon} \right) + o(1) } \right)\left( 1 + o(1) \right) \, , \, & t \leq t_0,\\
        -\frac{1}{2} \left( 1 + o(1) \right) e^{t \, \ln \left( 4 \epsilon \right) + \ln\frac{1}{2} + o(1)} + 1 + o(1) \, , \, & t > t_0 \, .
    \end{cases}
\end{equation*}
What this first-order analysis reveals is that, for large $M$ and small $\epsilon$, the optimal schedule exhibits a \textit{sigmoid-like shape} with a transition time $t_0 \in (0,1)$.
Specifically, for $t \in [0, t_0]$ the curve will be of the form $A e^{\alpha t}$, for $A, \alpha \in \R^+$, whereas for $t \in [t_0, 1]$ the curve will behave as $-B e^{-\beta t}$, for $B, \beta \in \R^+$.
This behavior is intuitive: the optimal schedule is attempting to slow the dynamics at the beginning \textit{and} at the end of the transport, where the Lipschitz constant is largest.
Turning now to the transition time $t_0$, we use \eqref{eq:transition-time-in-theorem} to compute
\begin{align*}
    t_0 = \frac{ \ln \left[ \frac{1}{2} \left( 1 + \frac{M-1}{1-\epsilon} \right) \right] }{ \ln(\frac{1}{\epsilon}) + \ln\left[ \frac{1}{4} \left( \frac{M-1}{1-\epsilon} + \frac{1 - \epsilon}{M-1} + 2 \right) \right] } =\frac{\ln M + o(1)}{ \ln \epsilon^{-1} + \ln M + o(1) } \, .
\end{align*}
Notice that if we take $\epsilon^{-1} = \Theta(M^\gamma)$ for some $\gamma \in \R^+$, then we have:
\begin{equation*}
    t_0 = \frac{\ln M}{  \ln \left( M^\gamma \right) + \ln M } + o(1) = \frac{1}{\gamma + 1} + o(1) .
\end{equation*}
That is, if $\epsilon$ decays within a polynomial factor of $1/M$, then \textit{we will always have a transition time} $t_0 \in (0,1)$ for any positive exponent $\gamma \in \R^+$.
Moreover, notice that $t_0 > 1/2 \iff \gamma < 1$, $t_0 < 1/2 \iff \gamma > 1$ and $t_0 = 1/2 \iff \gamma = 1$. This is intuitive: if $\epsilon$ decays within the same order of magnitude as $1/M$, then the optimal schedule treats contraction and expansion equally.
It therefore allocates the same amount of time to suppressing either type of singularity. 
If $\epsilon$ decays more quickly than $1/M$ (i.e., $\gamma > 1)$, then the optimal schedule will allocate more time to suppressing the contractive behavior of $T$;
similarly, if  $\epsilon $ decays  more slowly ($\gamma < 1$), then it will spend more time suppressing the expansive behavior of $T$.
Finally, notice that if $\gamma = 0$---that is, if  $\epsilon$ remains bounded away from zero---we have $t_0 = 1$, meaning that the optimal schedule obeys the simpler form \eqref{eq:solution-of-L-inf-ODE-simple-form-main-text}, effectively tempering only the dilation of the transport.
To finish our asymptotic analysis, notice that if $\epsilon^{-1} = \mathcal{O}(\ln M)$ then we have $t_0 \to 1$, and if $\epsilon^{-1} = \Omega(e^M)$ then $t_0 \to 0$. 
In other words, if $\epsilon^{-1}$ grows sub-logarithmically then $\tau_\infty$ will only slow at the beginning, and if $\epsilon$  grows super-exponentially then $\tau_\infty$  will only slow at the end.

\subsection{Multi-modal distributions}
\label{subsec:multi-modal-distributions}
In this section, we discuss the implications of the optimal schedule on multi-modal distributions, which naturally lead to pathological flows.
Fix a compact and convex $\Omega \subset \Rd$ and let $\mu$ be localized around some $x_0 \in \Omega$ while also satisfying $\supp \mu = \Omega$.
By ``localized,'' we mean that there is some $r \ll \textup{diam} \, \Omega$ such that $\mu(B_r(x_0)) > 1 - \epsilon$ for some small $\epsilon > 0$.
Now consider weights $q_1, \dots, q_n > 0$ such that $\sum_i q_i = 1$ and let 
   $ \nu = \sum_{i=1}^n q_i \, \nu_i$,
where each $\nu_i$ is localized around some $x_i \in \Omega$ with all $x_i$ distinct, again meaning that for $r_i \ll \textup{diam} \, \Omega$ we have $\nu_i(B_{r_i}(x_i)) > 1 - \epsilon$. Also, $\supp \nu_i = \Omega$ for all $i$.
Moreover, let $D = \max_{i,j} \| x_i - x_j \|$ be the maximal modal separation and, without loss of generality, assume that $x_1$ and $x_2$ satisfy $\| x_1 - x_2 \| = D$.
Now suppose $T$ is a transport map coupling $\mu$ to $\nu$ such that for some $c>0$, depending only on the domain $\Omega$ and the weights $(q_i)_{i=1}^n$, we have that for all $x, y \in \Omega$:
\begin{equation}
    \label{eq:well-behaved-T}
    \left \|x - y \right \|  \frac{1}{c} \leq \left \| T(x) - T(y) \right  \| \leq \|x-y\| \, c \, D \, .
\end{equation}
Such inequalities are expected to hold for regular enough $\mu$ and $\nu$, e.g., if their densities are uniformly bounded away from zero.\footnote{Specifically, one does not expect to dilate $\mu$ more than the maximal separation $D$ of the modes.}
Furthermore, assume that $q_1, q_2 \geq 2 \epsilon$, i.e., the most well-separated modes contain a non-trivial amount of mass; and that $r_i < D /(2n)$ and $\epsilon < 1/2$, i.e., all modes are localized on a scale smaller than $D$.
Since $T_{\sharp} \mu = \nu$, there must exist $y_1, y_2 \in B_{r_0}(x_0)$ such that $T(y_1) \in B_{r_1}(x_1)$ and $T(y_2) \in B_{r_2}(x_2)$.
In particular, this implies that
  $  \| T(y_1) - T(y_2) \| \geq D - (r_1 + r_2) \geq D / 2 $.
Thus, the Lipschitz constant of $T$ is at least $D/2$ and the objective under the trivial schedule satisfies
\begin{equation*}
    \Lambda[t \mapsto t] \geq D/2.
\end{equation*}
Due to \eqref{eq:well-behaved-T}, however, we have that the minimal and maximal eigenvalues of the Jacobian of $T$ are bounded by $1/c$ and $c D$, respectively.
Hence via Theorem \ref{thm:non-asymptotic-result-main-text} we have
\begin{equation*}
    \Lambda[\tau_\infty] \leq \ln \left ( \frac{c^2}{4} \, D \right ) + o(1) \quad \textup{as} \quad D \to \infty \, .
\end{equation*}
This illustration is meant to convey that when the source is unimodal and the target is multi-modal, we expect that under the trivial schedule, the Lipschitz constant of our flow will be at least as bad as the maximal separation of the modes $D$. Applying the optimal schedule will, most often, bring this down to a logarithmic factor in $D$, within constants.

In Figures~\ref{fig:GMMflow} and \ref{fig:GMMderivatives}, we illustrate this scaling for a standard Gaussian source measure and a bimodal Gaussian mixture target, $\nu = 0.8\mathcal{N}(-2, 0.02^2) + 0.2 \mathcal{N}(2, 0.01^2)$. Figure~\ref{fig:GMMflow} shows that trajectories under the optimal schedule $\tau_\infty$ slow down near both $t=0$ and $t=1$, consistent with the discussion in Section~\ref{sec:illconddists}. The optimal schedule for this source/target pair is illustrated on the right of Figure~\ref{fig:schedules}. Figure~\ref{fig:GMMderivatives} shows the spatial Lipschitz constant of the velocity field $v$ as a function of time, for both the trivial and optimal schedules. The exponential improvement in the worst-case Lipschitz constant is manifest, and we see that the optimal schedule indeed makes the Lipschitz constant itself become constant over time, consistent with the observation made in Remark~\ref{remark:contatn-sup-and-ansatz}.

\begin{figure}
\centering
      \includegraphics[trim=0 195 0 210,clip,width=2.9in]{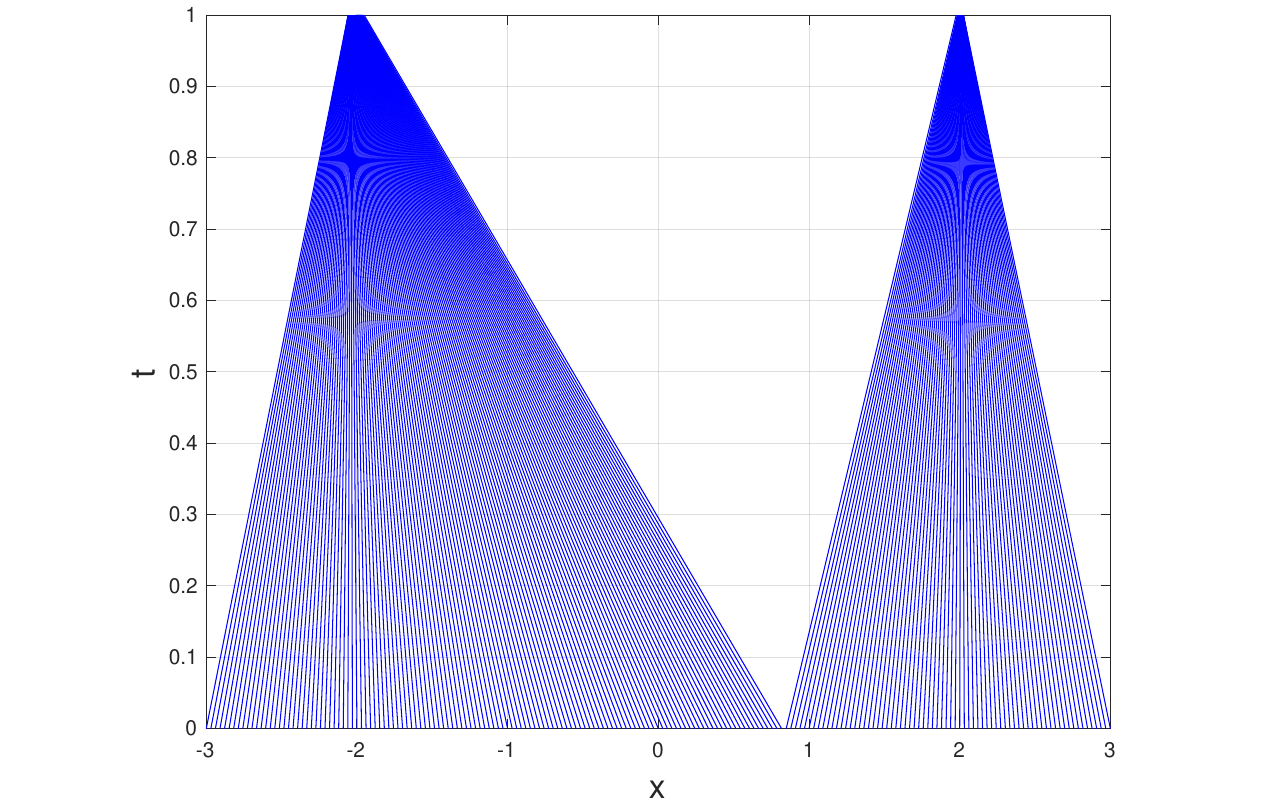}
      \includegraphics[trim=0 195 0 210,clip,width=2.9in]{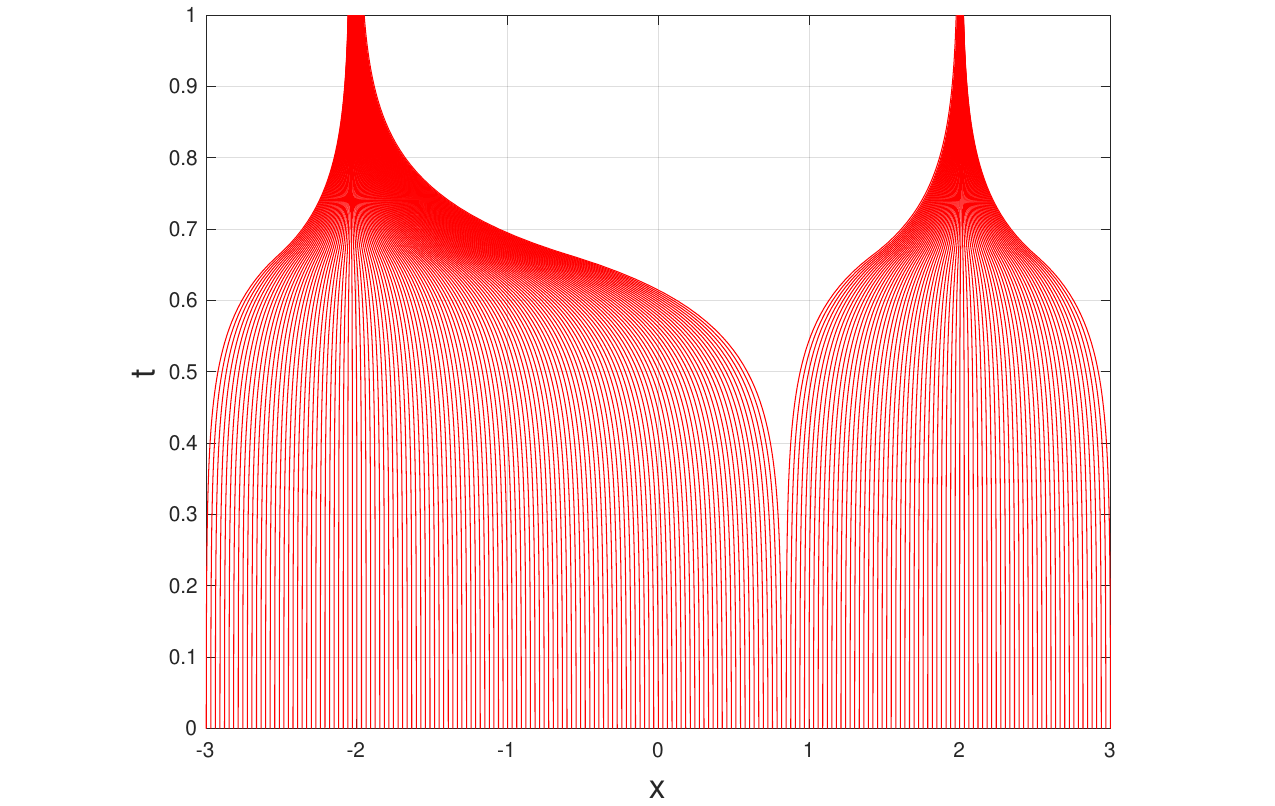}      
    \caption{Trajectories  $t \mapsto X(x, t)$ transporting a standard Gaussian $\mu$ at $t=0$ to a bimodal Gaussian mixture $\nu = 0.8\mathcal{N}(-2, 0.02^2) + 0.2 \mathcal{N}(2, 0.01^2)$ at $t=1$:
       \textit{(left)}  straight-line trajectories;  \textit{(right)} trajectories produced by the optimal schedule $\tau_\infty$.}
  \label{fig:GMMflow}
   \end{figure}

\begin{figure}
\centering
      \includegraphics[trim=0 195 0 210,clip,width=3.5in]{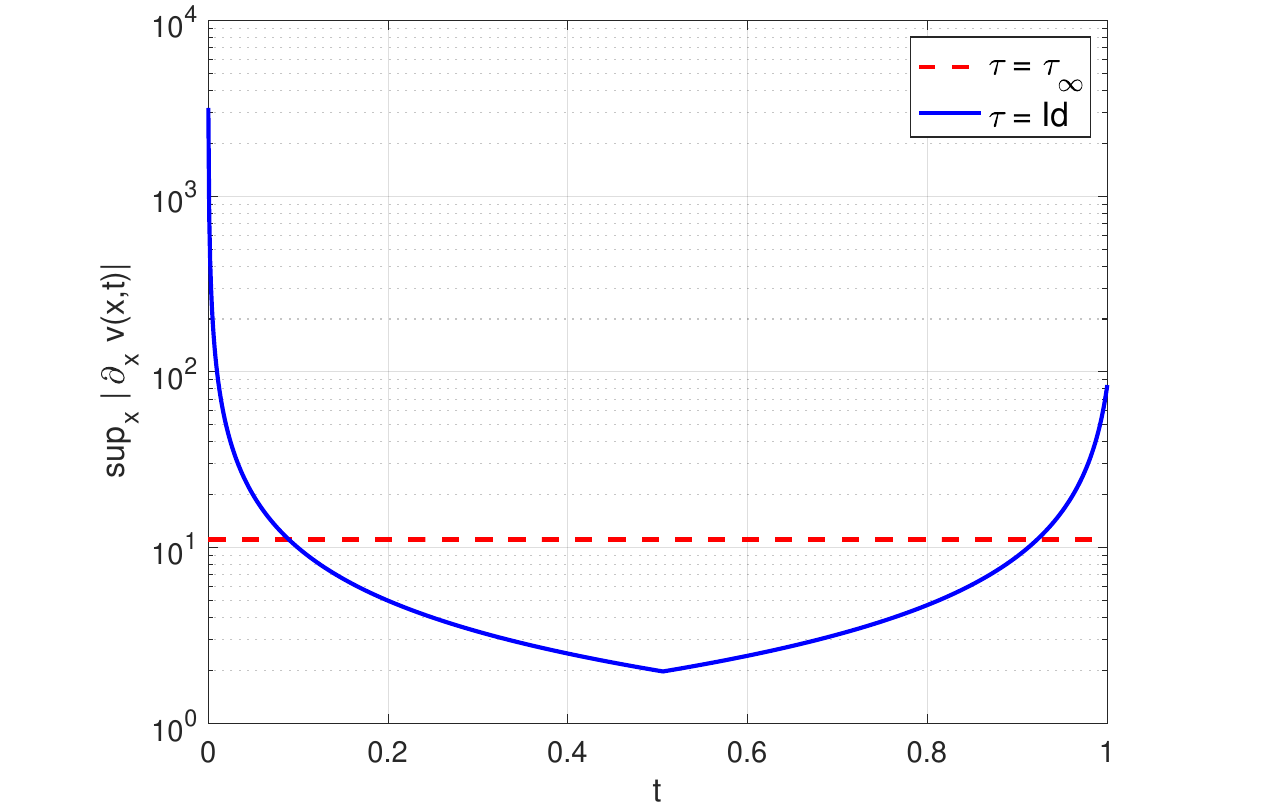}
     \caption{Spatial Lipschitz constant of the velocity field versus time $t$, for the Gaussian mixture example described in Figure~\ref{fig:GMMflow}: under the trivial schedule, $\Lip  (v(\cdot,t))$ (blue/solid line); and under the optimal schedule, $\Lip (v_{\tau_\infty}(\cdot,t)  )$ (red/dashed line).}
  \label{fig:GMMderivatives}
  \end{figure}

\section{Discussion}
\label{Discussion}
The present work suggests many avenues for future exploration. First, we believe that the proof techniques developed here can be extended to more general schedules and flows.
For instance, we can extend the notion of a schedule, looking for \textit{matrix-valued} maps $\chi: [0,1] \to \operatorname{GL}(d)$ that modify the linear dynamics via $\frac{d}{dt}X_\chi(x,t) = \dot{\chi}(t) \left( T(x) - x \right)$,
where $\chi(t) = \tau(t) I_d$ recovers the present problem. It is also of interest to solve the optimal scheduling problem for paths that depart from the displacement interpolation of a given transport; this of course touches on the broader problem of identifying optimal \textit{paths}, again from the perspective of velocity field regularity. 
We also note that the symmetry assumption on $\nabla T$ (cf.\ Assumption~\ref{assumption:regularity-of-transport}) can likely be relaxed if one is willing to work with \emph{singular values} instead of eigenvalues; in this case, however, one would loose the straightforward dilation / compression intuition used to understand the optimal schedule.

Second, we suggest that our method for finding optimal schedules, and the form of these schedules, can spark more efficient algorithms for learning dynamical representations of transport.  While the optimal $\tau_\infty$ in Theorem~\ref{thm:solution-of-L-inf-ODE-main-text} depends on the eigenvalues of $\nabla T$, its general structure (a sigmoid-like curve with transition time $t_0 \in [0,1]$) is independent of $T$. It could thus be well represented with a few-parameter function and simply made part of the learning problem. Alternatively, we note that the optimal schedule does not require full information about the underlying map $T$: only the maximal and minimal eigenvalues of its Jacobian are needed. Thus one could devise an iterative scheme starting, e.g., with the trivial schedule $\tau = \text{Id}$, estimating a velocity field $v$ and the maximal and minimal eigenvalues of the resulting $\nabla T^v$, computing the implied optimal schedule $\tau_\infty$, updating the velocity field $v_\tau$ accordingly, and so on. We leave the theoretical and empirical exploration of these ideas for future work.

\acks{PT, ZR, and YM acknowledge support from the US Air Force Office of Scientific Research (AFOSR) via award number FA9550-20-1-0397. ZR and YM also acknowledge support from the US Department of Energy, Office of Advanced Scientific Computing Research, under grants DE-SC0021226 and DE-SC0023187. PT acknowledges support from a George and Marie Vergottis Fellowship at MIT. ZR acknowledges support from a US National Science Foundation Graduate Research Fellowship. PT also thanks Youssef Chaabouni, Mehdi Makni,  and Baptiste Rabecq for helpful discussions.}

\bibliography{bibliography}
\newpage
\appendix
\section{Proofs of main results}
\label{sec:Appendix}
\subsection{Preliminary results}
\label{subsec:appendix-preliminaries}
We start by proving auxiliary results.

\begin{lemma}
    \label{lemma:silly-but-useful}
    For $\theta \in [0,1]$ define the one-parameter family:
    \begin{align*}
        \phi_\theta: [-1, \infty] &\to \overline{\R} \\
         x & \mapsto \left| \frac{x}{1 + \theta x} \right| ,
     \end{align*}
    where $\overline{\R}$ denotes the extended reals.
    Then, for any non-empty $S \subset (-1, \infty)$ such that $\sup S < \infty$ and $\inf S > -1$ we have:
    \begin{equation*}
        \sup_{x \in S} \phi_\theta(x) = \max \left\{ \left| \frac{ \sup S }{ 1 + \theta \sup S  } \right| , \left| \frac{ \inf S }{ 1 + \theta \inf S } \right| \right\} = \max \left\{ \frac{ \sup S }{ 1 + \theta \sup S  } , \frac{ - \inf S }{ 1 + \theta \inf S } \right\} .
    \end{equation*}
\end{lemma}
\begin{proof}
    First, we can simplify
    \begin{equation*}
        \phi_\theta(x) = \begin{cases}
            \frac{x}{1 + \theta x} & x \geq 0 \\
            \frac{-x}{1 + \theta x} & x < 0
        \end{cases},
    \end{equation*}
    and then, taking a derivative, we obtain
    \begin{equation*}
        \frac{d}{dx} \phi_\theta(x) = \begin{cases}
            \frac{1}{(1 + \theta x)^2} & x > 0 \\
            \frac{-1}{(1 + \theta x)^2} & x < 0 \, ,
        \end{cases}
    \end{equation*}
    where the derivative is ill-defined at $0$; however $\phi_\theta(0) = 0$. We notice that for any $\theta$ the function decreases monotonically in $(-1, 0)$ and grows monotonically in $(0, \infty)$;
    hence for $S \subset (-1, \infty)$ we have
    \begin{equation*}
        \sup_{x \in S} \phi_\theta(x) = \max \left\{ \frac{ \sup S }{ 1 + \theta \sup S } , \frac{ - \inf S }{ 1 + \theta \inf S } \right\} \, .
    \end{equation*}
    For the second part of the claim we want to prove that
    \begin{equation}
        \label{eq:silly-lemma-second-claim}
        \max \left\{ \frac{ \sup S }{ 1 + \theta \sup S } , \frac{ - \inf S }{ 1 + \theta \inf S } \right\} = \left\{ \left| \frac{ \sup S }{ 1 + \theta \sup S } \right| \, , \, \left| \frac{ \inf S }{ 1 + \theta \inf S } \right| \right\} \, .
    \end{equation}
    First, notice that for all $\theta \in [0,1]$ and any $x \in (-1, \infty)$ we have $1 + \theta \, x  \geq 0$; thus:
    \begin{align*}
        1 + \theta \, \sup S \geq 0 \quad \text{and} \quad 1 + \theta \, \inf S \geq 0 
    \end{align*}
    for any $\theta \in [0,1]$ using that $-1 < \inf S \leq \sup S < \infty$.
    Now if $S \subset (-1, 0)$ then by the monotonicity of $\phi_\theta$ we have
    \begin{equation*}
        \frac{ -\sup S }{ 1 + \theta \sup S } \leq \frac{ -\inf S }{ 1 + \theta \inf S } ,
    \end{equation*}
    and both terms are positive so equation \eqref{eq:silly-lemma-second-claim} is verified. Simiarly, equation \eqref{eq:silly-lemma-second-claim} can be verified when $S \subset (0, \infty)$. 
    Finally, if $S \cap (0, \infty) \cap (-1, 0) \neq \emptyset$ then $\sup S > 0$ and $\inf S < 0$ and so equation \eqref{eq:silly-lemma-second-claim} is trivially true.
    This finishes the proof.
\end{proof}

\begin{lemma}
    \label{lemma:eigenvalues-continuity-and-positivity}
    Let $(\sigma_i(s))_i$ denote the eigenvalues of $\nabla T(s)$, for $s \in \Omega$, recalling that by Assumption \ref{assumption:regularity-of-transport} they are real.
    Moreover, let $\sigma_{\max}(s) \coloneq \max_i \sigma_i(s)$ and $\sigma_{\min}(s) \coloneq \min_i \sigma_i(s)$.
    Then, the maps $\Omega \to \R$ given by
        \begin{equation*}
            s \mapsto \sigma_{\max}(s) \quad \text{ and } \quad s \mapsto \sigma_{\min}(s)
        \end{equation*}
        are continuous. Moreover, there exists $c \in \R^+$ such that for all $s \in \Omega$:
        \begin{equation}
            \label{eq:lambda-bounds}
            \sigma_{\min}(s) \geq c > 0 \, .
        \end{equation}
\end{lemma}
\begin{proof}
    For the first claim, notice that the map $s \mapsto \sigma_{\max}$ coincides with the map $s \mapsto \OperatorNorm{\nabla T(s)}$.
    Now, by Assumption \ref{assumption:regularity-of-transport} we have that $T: \Omega \subset \Rd \to \Omega$ is in $C^1(\Omega;\Omega)$, meaning that for each $1 \leq i,j \leq d$ the maps 
    $s \mapsto \partial_i \phi_j(s)$ are continuous. As these form the entries of the Jacobian, the map $s \mapsto \nabla T(s)$ is continuous in the operator norm; hence $s \mapsto \sigma_{\max}(s)$ is also continuous as a map $\Omega \to \R$.
    Similarly, notice that the map $s \mapsto \sigma_{\min}(s)$ coincides with the map $s \mapsto \OperatorNorm{(\nabla T(s))^{-1}}^{-1}$.
    This can be viewed as a composition of the maps:
    \begin{equation*}
        s \mapsto \nabla T(s) \mapsto (\nabla T(s))^{-1} \mapsto \OperatorNorm{(\nabla T(s))^{-1}} \mapsto \OperatorNorm{(\nabla T(s))^{-1}}^{-1} .
    \end{equation*} 
    Since the inversion operator is continuous in $\textup{GL}(d)$ and the spectrum of $\nabla T(s)$ contains positive numbers at each $s \in \Omega$, by Assumption \ref{assumption:regularity-of-transport}, it follows that the above composition is continuous finishing the proof of the first claim.

    The second part is an immediate consequence of Assumption \ref{assumption:regularity-of-transport}. Namely, since $\sigma_i(s) > 0$ for all $i$ we can take $c \coloneq \inf_{s \in \Omega} \sigma_{\min}(s)$ using the continuity of $\sigma_{\min}$ and the compactness of $\Omega$.

\end{proof}

\begin{lemma}
    \label{lemma:exchange-max-sup}
    Let $S$ be any set and $H, K : S \to \overline{\R}$ be arbitrary functions taking values in the extended real line. Then:
    \begin{equation}
        \sup_{s \in S} \max\{ H(s) \, , \, K(s) \} = \max\{ \sup_{s \in S} H(s) \, , \, \sup_{s \in S} K(s) \}
    \end{equation} 
\end{lemma}
\begin{proof}
    The proof is elementary. Start by considering the case where $\sup_{s \in S} H(s) = +\infty$ or $\sup_{s \in S} K(s) = +\infty$.
    Then, the right hand side is clearly infinite and the left hand side is bounded below by both $\sup_{s \in S} H(s)$ and $\sup_{s \in S} K(s)$
    and so it is also infinite. Now if $\sup_{s \in S} H(s) = -\infty$ then for all $s \in S$ we have $H(s) = -\infty$ and so both the left and 
    right hand side reduce to $\sup_{s \in S} K(s)$. Finally, if $\sup_{s \in S} H(s) = \sup_{s \in S} K(s) = -\infty$ then both sides are equal to $-\infty$.
    Now assume that $H, K$ are bounded. Start by noticing that for any $s \in S$ we have:
    \begin{align*}
        H(s) &\leq \sup_{s \in S} H(s) \\
        K(s) &\leq \sup_{s \in S} K(s)
    \end{align*}
    and so, for any $s \in S$:
    \begin{align*}
        \max\{ H(s), K(s) \} &\leq \max\{ \sup_{s \in S} H(s), \sup_{s \in S} K(s) \} \\
    \end{align*}
    taking the supremum we conclude:
    \begin{align*}
        \sup_{s \in S} \max\{ H(s), K(s) \} &\leq \max\{ \sup_{s \in S} H(s), \sup_{s \in S} K(s) \} \\
    \end{align*}
    which is the first needed inequality. For the other inequality, start noticing that for any $s \in S$ we have:
    \begin{align*}
        H(s) &\leq \max\{ H(s), K(s) \} \\
        K(s) &\leq \max\{ H(s), K(s) \}
    \end{align*}
    Thus, taking suprema on both sides:
    \begin{align*}
        \sup_{s \in S} H(s) &\leq \sup_{s \in S} \max\{ H(s), K(s) \} \\
        \sup_{s \in S} K(s) &\leq \sup_{s \in S} \max\{ H(s), K(s) \}
    \end{align*}
    and so:
    \begin{align*}
        \max\{ \sup_{s \in S} H(s), \sup_{s \in S} K(s) \} &\leq \sup_{s \in S} \max\{ H(s), K(s) \} \\
    \end{align*}
    which is the second needed inequality. This concludes the proof.
\end{proof}

\subsection{Proof of Theorem \ref{thm:reformulate-opt-problem}}
\label{sec:appendix_T2}
\begin{proposition}
    \label{prop:lip-equiv-to-jacobian}
    Fix $\Omega \subset \R^d$ compact, convex and let $v: \Omega\to \R^d$ be a $C^1$ vector field.
    Then the Lipschitz constant of $v$, denoted by $\Lip(v)$, is the smallest number such that for all $x, y \in \Omega$ we have
    \begin{equation}
        \| v(x) - v(y) \| \leq \Lip(v) \| x - y \| \, .
    \end{equation}
    Recalling the spectral norm of $v$ is defined as
    \begin{equation}
        \| v \|_{\textup{op}} = \sup_{\| x \| = 1} \| v(x) \| \, 
    \end{equation}
    and denoting the Jacobian of $v$ at $x$ by $\nabla v(x)$, we have the equality
    \begin{equation}
        \Lip(v) = \sup_{x\in \Omega} \| \nabla v(x) \|_{\textup{op}} \, .
    \end{equation}
\end{proposition}

\noindent The proof is an immediate consequence of the following two lemmas:

\begin{lemma}
    For $v$ as above, we have
    \begin{equation}
        \label{eq:sup-norm-under-Lipchitz}
        \sup_{x\in \Omega} \| \nabla v(x) \|_{\textup{op}} \leq \Lip(v) \, .
    \end{equation}
\end{lemma}
\begin{proof}
    For the proof of this statement, fix any unit vector $v$. Now if $x \in \Omega^\circ$ we have:
    \begin{align*}
        \| \nabla v(x) v \| &= \lim_{h\to 0} \frac{\| v(x + hv) - v(x) \|}{h} \\
        &\leq \Lip(v) \lim_{h \to 0} \frac{1}{h} \| x + hv - x \| \\
        &= \Lip(v)
    \end{align*}
    Maximizing over all unit vectors $v$ we get
    \begin{equation}
        \label{eq:Lip-bound-interior}
        \| \nabla v(x) \|_{\textup{op}} \leq \Lip(v) \, .
    \end{equation}
    Now since $v$ is a $C^1$ map $\Omega \to \Omega$ we have that $x \mapsto \| \nabla v(x) \|_{\textup{op}}$ is a continuous map $\Omega \to \R$.
    Then, equation \eqref{eq:Lip-bound-interior} can be extended to the boundary $\partial \Omega$ by continuity. We conclude by taking the supremum over $\Omega$.
\end{proof}

\noindent We also have a converse result:

\begin{lemma}
    For $v$ as above we have:
    \begin{equation}
        \label{eq:Lipchitz-under-sup-norm}
        \Lip(v) \leq \sup_{x\in \Omega} \| \nabla v(x) \|_{\textup{op}}
    \end{equation}
\end{lemma}
\begin{proof}
    To see this, take any two $x, y \in \Omega$. By the convexity of $\Omega$ and the mean value theorem, there is a $\theta \in [0,1]$ such that:
    \begin{equation*}
        v(x) - v(y) = \nabla v(x + \theta(y - x)) \cdot (x - y)
    \end{equation*} 
    Taking norms we have:
    \begin{align*}
        \| v(x) - v(y) \|   &\leq \| \nabla v(x + \theta(y - x))\|_{\textup{op}} \, \| x - y \| \\
        &\leq \sup_{z\in \Omega} \| \nabla v(z) \|_{\textup{op}} \, \| x - y \|
    \end{align*}
        So we have a Lipschitz condition with constant $\sup_{z \in \Omega} \| \nabla v(z) \|_{\textup{op}} < \infty$, where finiteness is guaranteed by the fact that $v$ is $C^1$ and $\Omega$ is compact.
        By the minimality of the Lipschitz constant, it must be that $\Lip(v) \leq \sup_{z\in \Omega} \| \nabla v(z) \|$.
\end{proof}

\begin{theorem*}
    \label{prop:reformulate-appendix-opt-problem}
    For notational convenience, define $f(s) \coloneq \max_i \sigma_i(s) - 1$, $g(s) \coloneq \min_i \sigma_i(s) - 1$ and let $\Omega_{[0,1]} \coloneq \Omega \times [0,1]$.
    Then, $\Lambda[\tau]$ can be written as:
    \begin{equation}
        \label{eq-appendix:reformulated-Lambda}
        \Lambda[\tau] = \sup_{ (s,t) \in \Omega_{[0,1]} } |\dot{\tau}(t)| \, \max{ \left\{ \left| \frac{f(s)}{1 + \tau(t) f(s)} \right|, \, \left| \frac{g(s)}{1 + \tau(t) g(s)} \right| \right\} }
    \end{equation}
\end{theorem*}
\begin{proof}
    First, for $t \in [0,1]$ define the sets
    \begin{equation*}
        \Omega_t = \left\{ X(s, t) \, : \, s \in \Omega \right\} \subset \Omega \, .
    \end{equation*}
    and recall the definition of the dynamics
    \begin{align*}
        & v_\tau(\cdot, t): \Omega_{\tau(t)} \to \Omega \\
        & v_\tau(\cdot, t): s \mapsto \dot{\tau}(t) \, (T - \text{id}) \circ X_\tau^{-1}(\cdot, t) \, .
    \end{align*}
    In \citet[Theorem 3.4]{approximation_paper}, it is established that for $T$ a $C^1$ map such that $\nabla T(s)$ has a positive eigen-spectrum for all $s \in \Omega$ (see Assumption \ref{assumption:regularity-of-transport}) the flow $X_\tau$ is a $C^1$ diffeomorphism onto its image $\Omega_{\tau(t)}$.
    Using the chain rule together with the inverse formula for the Jacobian we can compute:
    \begin{align*}
        (\nabla_s v_\tau)(s,t) &= \dot{\tau}(t) \, \left[ \nabla_s T \left( X_\tau^{-1}(s, t) \right) - I_d \right] \cdot \nabla_s X_\tau^{-1}(s, t) \\
                          &= \dot{\tau}(t) \, \left[ \nabla_s T \left( X_\tau^{-1}(s, t) \right) - I_d \right] \cdot \left( \nabla_s X_\tau \right)^{-1}( X_\tau^{-1}(s,t),t ) \\
    \end{align*}
    Thus, we conclude
    \begin{equation*}
        \label{eq:intermediate-expression-nabla-f}
        (\nabla_s v_\tau)(X_\tau(s,t),t) = \dot{\tau}(t) \, \left[ \, \nabla_s T(s) - I_d \, \right] \cdot \left( \nabla_s X_\tau \right)^{-1}(s,t) \, .
    \end{equation*}
    Now since $s \in \Omega \, \iff \, X_\tau(s,t) \in \Omega_{\tau(t)}$ it must be that
    \begin{equation*}
        \sup_{s \in \Omega_{\tau(t)}} \| v_\tau(s, t) \|_{\textup{op}} = \sup_{s \in \Omega} \| v_\tau(X_\tau(s,t),t) \|_{\textup{op}}
    \end{equation*}
    and we have, by the definition of $\Lambda[\tau]$:
    \begin{align*}
        \Lambda[\tau] &= \sup_{t \in [0,1]} \sup_{s \in \Omega_{\tau(t)}} \| \nabla_s v_\tau(s,t) \|_{\textup{op}} \\
        &= \sup_{t \in [0,1]} \sup_{s \in \Omega} \| \nabla_s v_\tau(X_\tau(s,t),t) \|_{\textup{op}} \\
        &= \sup_{t \in [0,1]} \sup_{s \in \Omega} \left\| \, \dot{\tau}(t) \, \left[ \, \nabla_s T(s) - I_d \, \right] \cdot \left( \nabla_s X_\tau \right)^{-1} \, \right\|_{\textup{op}} \, .
    \end{align*}
    To finish the proof, we will re-write the term inside the operator norm of the last line, above.
    
    Due to Assumption \ref{assumption:regularity-of-transport}
    we have that $\nabla T(s)$ is 
    diagonalizable by orthogonal transformations, for all $s\in \Omega$. Suppressing now the dependence on $s \in \Omega$, there exist an orthogonal matrix $U$ and a diagonal matrix $D$ such that:
    \begin{equation}
        \nabla T = U \, D \, U^{T} .
    \end{equation}
    Notice that we can then write:
    \begin{equation}
        \nabla T - I_d = U \, (D - I_d) \, U^{T}
    \end{equation}
    with $I_d$ the identity matrix in $\Rd$. Moreover, we can write:
    \begin{equation}
        \left[I_d + \tau(t) (\nabla T - I_d)\right]^{-1} = U \, \left[I_d + \tau(t) (D - I_d)\right]^{-1} \, U^{T} .
    \end{equation}
    Since both matrices are simultaneously diagonalizable, it is easy to compute the spectral norm of the product at a fixed $(s,t) \in \Omega \times [0,1]$, leading to:
    \begin{align*}
        \|\nabla_s v_\tau(X_\tau(s,t),t)\|_{\textup{op}} &= |\dot{\tau}(t)| \, \max_{1 \leq i \leq d} \left| \frac{ \left[ D(s) - I_d \right]_{ii} }{\left[ I_d + \tau(t) (D(s) - I_d) \right]_{ii} } \right| \\
         &= |\dot{\tau}(t)| \, \max_{1 \leq i \leq d} \left| \frac{ \sigma_i(s) - 1 }{ 1 + \tau(t) ( \sigma_i(s) - 1) } \right|
    \end{align*}
    writing $\sigma_i(s)$ for the $i$th eigenvalue of $\nabla T(s)$ at $s \in \Omega$. To proceed, define the one parameter family of extended real valued mappings:
    \begin{align*}
       & \phi_\theta : [-1, \infty] \to \overline{\R} \\
         & \phi_\theta: x \mapsto \left| \frac{x}{1 + \theta \, x} \right|
    \end{align*}
    with $\overline{\R}$ the extended reals. Fixing $s \in \Omega$ and defining $S \coloneq \left\{ \sigma_i(s) - 1 \, : \, 1 \leq i \leq d \right\}$ we have:
    \begin{align*}
        \|\nabla_s v_\tau(X_\tau(s,t),t)\|_{\textup{op}} &= | \dot{\tau}(t) | \, \max_{1 \leq i \leq d} \phi_{\tau(t)} \left( \sigma_i(s) - 1 \right) \\
        &= | \dot{\tau}(t) | \, \sup_{ x \in S } \phi_{\tau(t)} \left( x \right)
    \end{align*}
    Due to Lemma \ref{lemma:eigenvalues-continuity-and-positivity} we have that $\sigma_i(s) - 1 > -1$ for all $1 \leq i \leq d \, , \, s \in \Omega$ and also $0 \leq \tau(t) \leq 1$ for all $t \in [0,1]$
    due to the monotonicity of $\tau$. Thus, we can apply Lemma \ref{lemma:silly-but-useful} to conclude:
    \begin{align}
        \|(\nabla f)(X_\tau(s,t)) \|_{\textup{op}} &=  | \dot{\tau}(t) | \,  \max \left\{ \left| \frac{ \sup S }{ 1 + \tau(t) \sup S } \right| \, , \, \left| \frac{ \inf S }{ 1 + \tau(t) \inf S } \right| \right\} \nonumber \\
        \label{eq:bound-spectral-norm}
        &= |\dot{\tau}(t)| \, \max \left\{ \left|  \frac{f(s)}{1 + \tau(t) f(s)} \right|, \, \left| \frac{g(s)}{1 + \tau(t) g(s)} \right| \right\}
    \end{align}
    and in the last step we have used the definitions $f = \max_i \sigma_i - 1$ and $g = \min_i \sigma_i - 1$. This finishes the proof.
\end{proof}

\subsection{Proof of Theorem \ref{thm:Linfty-to-Lp-main-text}}
\label{subsec:proof-of-Linfty-to-Lp}

In this section we prove Theorem \ref{thm:Linfty-to-Lp-main-text} from the main text. First, we establish some preliminary results.
\begin{lemma}
    \label{lemma:Linfty-to-Lp-appendix}
    Let $\Omega$ be compact subset of $\Rd$ and $F,G:\Omega \to \R$ be continuous functions. Then, as $p \to \infty$ we have the following limit:
    \begin{equation}
        \left[ \int_{\Omega} F^{2p}(s) + G^{2p}(s) \, ds \right]^{\frac{1}{2p}} \to \, \max\left\{ \| F \|_{L^\infty(\Omega)} \, , \, \| G \|_{L^\infty(\Omega)} \right\}
    \end{equation}
\end{lemma}
\begin{proof}
    Start by noting that since $F, G$ are continuous and $\Omega$ is compact then $F, G \in L^p(\Omega)$ for $p \in [1, \infty]$ and so all expressions
    in the statement are finite.
    First, we prove an upper bound:
    \begin{align*}
        & \int_{\Omega} F^{2p}(s) + G^{2p}(s) \, ds \leq \\
        & \| F \|_{L^\infty}^{2p} \Leb{\Omega} + \| G \|_{L^\infty}^{2p} \Leb{\Omega} = \\
        & \Leb{\Omega} \cdot \left( \| F  \|_{L^\infty}^{2p} + \| G  \|_{L^\infty}^{2p} \right)
    \end{align*}
    where $\Leb{\Omega}$ is the Lebesgue measure of $\Omega$.
    Raising this to the power $\frac{1}{2p}$ and using that for positive reals $x,y > 0$ we have
    $\left( |x|^{2p} + |y|^{2p} \right)^{1/(2p)} \to \max\{x, y\}$ as $p \to \infty$, we obtain
    \begin{equation*}
        \left[ \Leb{\Omega} \cdot \left( \| F  \|_{L^\infty}^{2p} + \| G  \|_{L^\infty}^{2p} \right) \right]^{\frac{1}{2p}} \to \max\{ \| F  \|_{L^\infty} \, , \, \| G  \|_{L^\infty} \}
    \end{equation*}
    and thus
    \begin{equation}
        \label{eq:Lp-to-L-infinity-upper-bound}
        \lim_{p \to \infty} \left( \int_{\Omega} F^{2p}(s) + G^{2p}(s) \, ds \right)^{\frac{1}{2p}} \leq \max\{ \| F  \|_{L^\infty} \, , \, \| G  \|_{L^\infty} \}.
    \end{equation}

    To prove a lower bound, recall that by continuity of $F, G$ and compactness of $\Omega$ for any $\epsilon > 0$ there are $\delta_1, \delta_2 > 0$
    such that:
    \begin{align*}
        & \Leb{\{ x : F(x) \geq \|F\|_{L^\infty} - \epsilon\}} = \delta_1 \\
        & \Leb{\{ x : G(x) \geq \|G\|_{L^\infty} - \epsilon\}} = \delta_2
    \end{align*}
    Now we can estimate:
    \begin{gather*}
        \int_{\Omega} F^{2p}(s) + G^{2p}(s) \, ds \geq \\
        \int_{\{ F \, \geq \, \|F\|_{L^\infty} - \epsilon\}} F^{2p}(s) \, ds + \int_{\{ G \, \geq \, \|G\|_{L^\infty} - \epsilon\}} G^{2p}(s) \, ds \geq \\
        \delta_1 \left( \|F\|_{L^\infty} - \epsilon \right)^{2p} + \delta_2 \left( \|G\|_{L^\infty} - \epsilon \right)^{2p} = \\
        \min\{ \delta_1, \delta_2 \} \left[  \left( \|F\|_{L^\infty} - \epsilon \right)^{2p} + \left( \|G\|_{L^\infty} - \epsilon \right)^{2p} \right]
    \end{gather*}
    Now raising the last line to the power $\frac{1}{2p}$ and taking the limit as $p \to \infty$ we obtain:
    \begin{equation*}
        \left\{ \min\{ \delta_1, \delta_2 \} \left[ \left( \|F\|_{L^\infty} - \epsilon \right)^{2p} + \left( \|G\|_{L^\infty} - \epsilon \right)^{2p} \right] \right\}^{\frac{1}{2p}} \to \max\{ \| F  \|_{L^\infty} \, - \epsilon , \, \| G  \|_{L^\infty} - \epsilon \} .
    \end{equation*}
    Taking the limit $\epsilon \searrow 0 $ we obtain:
    \begin{equation}
        \label{eq:Lp-to-L-infinity-lower-bound}
        \lim_{p \to \infty} \left( \int_{\Omega} F^{2p}(s) + G^{2p}(s) \, ds \right)^{\frac{1}{2p}} \geq \max\{ \| F  \|_{L^\infty} \, , \, \| G  \|_{L^\infty} \} .
    \end{equation}
    Combining \eqref{eq:Lp-to-L-infinity-upper-bound} and \eqref{eq:Lp-to-L-infinity-lower-bound} we obtain the desired result.

\end{proof}
Another lemma that will be useful:
\begin{lemma}
    \label{lemma:monotonicty-Lp}
    Let $\Omega$ be compact subset of $\Rd$ and $F,G:\Omega \to \R$ be continuous functions. Set $s_p \coloneq \left[ \int_{\Omega} F^{2p}(s) + G^{2p}(s) \, ds \right]^{\frac{1}{2p}}$. Then, there are constants $c_{p,q} \geq 0$ for each $p \leq q$, with $p,q \in \{1, 2, \dots \} \cup \{\infty\}$, depending
    only on $\Leb{\Omega}$, such that:
    \begin{equation}
        s_p \leq c_{p,q} \, s_q .
    \end{equation}
    Moreover, these constants are uniformly bounded in $p$ and $q$ by $\sqrt{2 \, \text{Leb}(\Omega)}$.
\end{lemma}
\begin{proof}
    Notice that:
    \begin{equation*}
        s_p = \left[ \| F \|_{L^{2p}}^{2p} + \|G\|_{L^{2p}}^{2p} \right]^{\frac{1}{2p}}
    \end{equation*}
    Now by the H\"older inequality there are constants $C_{r,p}, D_{r,p} \geq 0$ such that:
    \begin{align*}
        s_p &= \left[ \| F \|_{L^{2p}}^{2p} + \|G\|_{L^{2p}}^{2p} \right]^{\frac{1}{2p}} \\
            &\leq \left[ C_{2p,2q}^{2p} \, \| F \|_{L^{2q}}^{2p} + C_{2p, 2q}^{2p} \, \|G\|_{L^{2q}}^{2p} \right]^{\frac{1}{2p}} \\ 
            &\leq C_{2p,2q} \, D_{2p, 2q} \, \left[ \| F \|_{L^{2q}}^{2q} + \|G\|_{L^{2q}}^{2q} \right]^{\frac{1}{2q}} \\
            &= c_{p,q} \, s_q
    \end{align*}
    where we have taken $c_{p,q} \coloneq C_{2p,2q} \, D_{2p, 2q}$.
    In fact, we can take \footnote{
        For completeness, we conduct the computation here: For a finite measure space $(X, \Sigma, \mu)$ and $p, q \in [1, \infty]$ consider a pair $q > p$ and some function $f \in L^q(X, \Sigma, \mu)$. 
        Then, by a Holder inequality and denoting the identity function by $\mathbf{1}$ we have:
        \begin{equation*}
            \norm{ f \cdot \mathbf{1} }_{L^p} \leq \left( \norm{ f^p }_{ L^{\frac{q}{p}} } \, \norm{ \mathbf{1}}_{L^{\frac{q}{q-p}}} \right)^{\frac{1}{p}} = \norm{ f }_{L^q} \, \mu(X)^{\frac{q-p}{qp}} = \norm{ f }_{L^q} \, \mu(X)^{\frac{1}{p} - \frac{1}{q}}
        \end{equation*}
        So above we can take $X$ to be either $\Omega$ equipped with the Borel sets and the Lebesgue measure or the discrete space $\{1, 2\}$ equipped with its power set and the counting measure.
    }
    \begin{align*}
        C_{2p,2q} &= \Leb{\Omega}^{\frac{1}{2p} - \frac{1}{2q}} \\
        D_{2p,2q} &= 2^{\frac{1}{2p} - \frac{1}{2q}} 
    \end{align*}
    and so
    \begin{equation}
        c_{p,q} = 2^{\frac{1}{2p} - \frac{1}{2q}} \, \Leb{\Omega}^{\frac{1}{2p} - \frac{1}{2q}} \, .
    \end{equation}
    Thus, the constants $c_{p,q}$ are uniformly bounded in $p$ and $q$ by $\sqrt{2 \, \text{Leb}(\Omega)}$.
  \end{proof}
    Recall, now, the spaces introduced in equations \eqref{eq:intro-T-2p}--\eqref{eq:intro-T-infty} of the main text:
    \begin{align*}
        & \mathcal{T}_{2} = \left\{ \tau \in H^1([0,1]) \, : \, \tau(0)=0, \, \tau(1)=1 \right\} \\
        & \mathcal{T}_{2}^b = \mathcal{T}_{2} \cap \left\{ \tau: 0 \leq \tau \leq 1 \right\}
    \end{align*}
    \begin{remark}
        In the interest of well-definedness recall that, by the Sobolev embedding theorem, elements of $H^1$ can be identified with continuous functions on $[0,1]$ and as such 
        the constraints $t: \tau(t) \geq 0$ as well as $\tau(0)=0, \, \tau(1)=1$ are well posed. Moreover, the set $\{ \tau \in C^0\left([0,1]\right) : 0 \leq \tau \leq 1 \text{ and }  \tau(0)=0, \, \tau(1)=1 \}$ is a closed subset of $C^0\left([0,1]\right)$
        and thus will be a closed subset of $H^1([0,1])$.
    \end{remark}
    Now we prove a lemma that will be used multiple times:
    \begin{lemma}
        \label{lemma:Psi-is-well-behaved}
        Recall the notation $f(s) = \max_i \sigma_i(s) - 1$ and $g(s) = \min_i \sigma_i(s) - 1$ for $\left(\sigma_i(s)\right)_i$ the eigenvalues of $\nabla T(s)$ at $s \in \Omega$.
        Moreover, we write $f^* = \sup_{s \in \Omega} f(s)$ and $g_* = \inf_{s \in \Omega} g(s)$.
        If Assumptions \ref{assumption:regularity-of-transport} and \ref{assumption:non-isometry} are satisfied we have:
    \begin{equation}\label{eq:g0f}
      -1<g_*<0< f^*<\infty,
    \end{equation}
    Moreover, there are non-negative reals:
    \begin{align*}
        & M_f = M_f(f, g, \Omega) \quad\textup{and}\quad M_g = M_g(g, f, \Omega) \\
        & m_f = m_f(f, g, \Omega) \quad\textup{and}\quad m_g = m_g(g, f, \Omega)
    \end{align*} 
    such that for all $\tau \in \mathcal{T}_2^b$ and all $p \in \N_{\infty} = \N \cup \{\infty\}$ we have:
    \begin{gather*}
        m_f \leq \left[ \int_{0}^{1} \int_\Omega \left(\frac{f(s)}{1+\tau(t) \, f(s)}\right)^{2p} \, ds \, dt \right]^{\frac{1}{2p}} \leq M_f \\ 
        m_g \leq \left[ \int_{0}^{1} \int_\Omega \left(\frac{g(s)}{1+\tau(t) \, g(s)}\right)^{2p} \, ds \, dt \right]^{\frac{1}{2p}} \leq M_g
    \end{gather*}
    and moreover, $m_f \cdot m_g > 0$ for all $p \in \N_{\infty}$. In fact, the bounds hold pointwise in $t$ in the sense that for any $\tau \in \mathcal{T}_2^b$ and any $p \in \N_{\infty}$ we have:
    \begin{align*}
        & m_f \leq \left[ \int_{\Omega} \left( \frac{f(s)}{1 + \tau(t) \, f(s) } \right)^{2p} \, ds \right]^{\frac{1}{2p}} \leq M_f \quad \text{for all} \quad t \in [0,1] \\
        & m_g \leq \left[ \int_{\Omega} \left( \frac{g(s)}{1 + \tau(t) \, g(s) } \right)^{2p} \, ds \right]^{\frac{1}{2p}} \leq M_g \quad \text{for all} \quad t \in [0,1]
    \end{align*}
  \end{lemma}
  \begin{proof}
    As usual, for $s \in \Omega$ denote $\sigma_{\max}(s) = \max_{i} \sigma_i(s)$ and $\sigma_{\min}(s) = \min_i \sigma_i (s)$.
    Firstly, recall that by Lemma \ref{lemma:eigenvalues-continuity-and-positivity} there is a $c > 0$ such that:
    \begin{equation*}
        \sigma_{\max}(s) \geq \sigma_{\min}(s) \geq c > 0 \quad \text{for all} \quad s \in \Omega
    \end{equation*}
    Moreover, again by Lemma \ref{lemma:eigenvalues-continuity-and-positivity} we have that $\sigma_{\max}, \sigma_{\min}: \Omega \to \R$ are continous and recalling that $\Omega$ is compact there is a $C < \infty$ such that:
    \begin{equation*}
        \sigma_{\min}(s) \leq \sigma_{\max}(s) \leq C  < \infty \quad \text{for all} \quad s \in \Omega
    \end{equation*}
    Therefore, we have shown the first part of the claim:
    \begin{equation*}
        -1 < c - 1 \leq g_* \leq  f^* \leq  C - 1 < \infty
    \end{equation*}
    
    Secondly, for any $\tau \in \mathcal{T}_2^b$ we have $1 + \tau(t) f(s) \geq  c > 0$ and $1 + \tau(t) g(s) \geq c > 0$ for all $(s,t) \in \Omega \times [0,1]$ and so we can estimate:
    \begin{align*}
        \sup_{s \in \Omega} \left| \frac{f(s)}{1 + \tau(t) \, f(s) } \right| &= \frac{ \sup_{s \in \Omega} |f(s)| }{ \inf_{(s,t) \in \Omega \times [0,1]} | 1 + \tau(t) \, f(s) | } \\
        &\leq \frac{ \max\left\{ |C - 1| \, , \, |c-1| \right\} }{c}
    \end{align*}
    Thus, for any $p \in \N_{\infty}$ we can take:
    \begin{equation*}
        M_f(f, g, \Omega) \coloneq \frac{ \max\left\{ |C - 1| \, , \, |c-1| \right\} }{c} \, \max\left\{ \Leb{\Omega} \, , \, 1 \right\} < \infty
    \end{equation*}
    Replacing $f$ by $g$ and arguing exactly as in the above line, we have proven the first part of the claim.

    Thirdly, notice that by Assumption \ref{assumption:non-isometry} there must exist an $s_0 \in \Omega$ and an index $1 \leq k \leq d$ such that:
    \begin{equation*}
        \lambda_k(s_0) \neq 1
    \end{equation*}
    As a result, either $\sigma_{\max}(s_0) \neq 0$ or $\sigma_{\min}(s_0) \neq 0$ and so $f(s_0) \neq 0$ or $g(s_0) \neq 0$. Assuming the latter and using the continuity of $g$ there is an $\epsilon > 0$ and an open set $U$, both depending on $g$, such that $|g(s)| \geq \epsilon $ for all $s \in U$.
    Now for finite $p$ we can use Jensen's inequality to estimate:
    \begin{align*}
        \int_{\Omega} \left( \frac{g(s)}{1 + \tau(t) \, g(s)} \right)^{2p} \, ds &\geq \Leb{\Omega}^{1-2p} \left( \int_{\Omega} \left| \frac{g(s)}{1 + \tau(t) \, g(s)} \right| \, ds \right)^{2p} \\
        &\geq \Leb{\Omega}^{1-2p} \left( \frac{1}{\sup_{(s,t) \in \Omega \times [0,1]} \left| 1 + \tau(t) \, g(s) \right| } \right)^{2p} \, \left( \int_{U} | g(s) | \, ds \right)^{2p} \\
        &\geq \frac{\epsilon^{2p} }{ C^{2p} } \, \frac{\Leb{U}}{\Leb{\Omega}^{2p-1}} 
    \end{align*}
    Thus, we have:
    \begin{align*}
        \left[ \int_{0}^{1} \int_{\Omega} \left( \frac{g(s)}{1 + \tau(t) \, g(s)} \right)^{2p} \, ds \, dt \right]^{\frac{1}{2p}} &\geq \frac{\epsilon }{ C } \, \frac{\Leb{U}^{\frac{1}{2p}}}{ \Leb{\Omega}^{\frac{2p-1}{2p} } }  \\
        & \geq \frac{ \epsilon }{ C } \, \frac{ \min\left\{ 1 \, , \, \Leb{U} \right\} }{ \max\left\{ 1 \, , \, \Leb{\Omega} \right\} }  \\
        & \eqcolon m_g(g, f, \Omega) > 0
    \end{align*}
    The argument is similar for $ p = \infty$ and the same lower bound holds.
    Finally, if $f(s_0) \neq 0$ we can use the same argument to obtain $m_f(f, g, \Omega) > 0$. Since we cannot have both $f(s_0) = 0$ and $g(s_0) = 0$ simultaneously, we have proven the claim.

  \end{proof}
  Using Lemma \ref{lemma:Psi-is-well-behaved} can now make the following definitions:
    \begin{definition}
        Recall that $\Omega_{[0,1]} = \Omega \times [0,1]$ and define the operators:
        \begin{align*}
            &\Psi^f: \mathcal{T}_{2}^b \to L^\infty\left(\Omega_{[0,1]}\right)\phantom{AA} \quad \textup{ and } \quad \Psi^g: \mathcal{T}_{2}^b \to L^\infty\left(\Omega_{[0,1]}\right) \\
            &\Psi^f[\tau](s,t) = \frac{f(s)}{1 + \tau(t) f(s)} \quad \textup{ and } \quad \Psi^g[\tau](s,t) = \frac{g(s)}{1 + \tau(t) g(s)}
        \end{align*}
        Now for $p \in \N_{\infty} = \N \cup \{ \infty \}$ introduce the functionals:
        \begin{align*}
            &\Lambda_p: \mathcal{T}_{2} \to \R \\
            &\Lambda_p[\tau]=\begin{cases}
                \left( \Big\|\dot{\tau} \, \Psi^f[\tau] \Big\|_{L^{2p}\left(\Omega_{[0,1]}\right)}^{2p} + \Big\| \dot{\tau} \, \Psi^g[\tau] \Big\|_{L^{2p}\left(\Omega_{[0,1]}\right)}^{2p} \right)^{\frac{1}{2p}} & \textup{if } \tau \in \mathcal{T}_{2}^b \\
                +\infty & \textup{if } \tau \in \mathcal{T}_{2} \setminus \mathcal{T}_{2}^b
            \end{cases}
        \end{align*}
        Moreover, if $\dot{\tau} \notin L^{2p}\left( \Omega_{[0,1]} \right)$ for some $p \in \N$ we set $\Lambda_p[\tau] = +\infty$.
    \end{definition}
    \begin{remark}
        For finite $p \in \N$ one may wish to use a more explicit formula for $\Lambda_p$:
        \begin{equation*}
            \Lambda_p[\tau] = \left( \int_{0}^{1} \int_{\Omega} \left( \frac{\dot{\tau}(t) \, f(s)}{1 + \tau(t) \, f(s) } \right)^{2p} \, ds \, dt + \int_{0}^{1} \int_{\Omega} \left( \frac{ \dot{\tau}(t) \, g(s) }{1 + \tau(t) \, g(s) } \right)^{2p} \, ds \, dt \right)^{\frac{1}{2p}}
        \end{equation*} 
    \end{remark}
    \begin{remark}
        Notice that the restriction of $\Lambda_{\infty}$ to $\mathcal{T}_\infty$ coincides with $\Lambda$, the objective of problem \eqref{eq:opt-problem}. This will play an important role in the proof of the main theorem.
    \end{remark}
    Now we prove a crucial property of the functionals $\Lambda_p$:    
    \begin{lemma}
        \label{lemma:lower-semi-cont-of-functionals}
        For each $p \in \N$ and each sequence $ \gamma_n \in \mathcal{T}_{2}$ converging to some $\gamma_\infty \in \mathcal{T}_{2}$ in the weak-$H^1$ topology we have:
        \begin{equation*}
            \Lambda_p[\gamma_\infty] \leq \liminf_{n \to \infty} \Lambda_p[\gamma_n]
        \end{equation*}
    \end{lemma}
    \begin{remark}
        Put succinctly, this lemma says that for each $p \in \N$ the functional $\Lambda_p$ is weakly lower semi-continuous in the affine space $\mathcal{T}_{2}$ equipped with the $H^1$ norm.
    \end{remark}
    \begin{proof}
        We note that our proof technique is based on \citet[Lemma 2.4 and Proposition 2.6]{garroni2001dielectric}. 
        The essence of this augment is to show that $\dot{\gamma}_n \rightharpoonup \dot{\gamma}_\infty$ in $L^{2p}\left([0,1]\right)$. This fact together with the
        well-behavedness of the operators $\Psi^f, \Psi^g$ and the weak lower semi-continuity of the $L^{2p}$ norm will allow us to conclude the proof.

      Fix a $p$ and assume without loss of generality that $\liminf_{n\to\infty} \Lambda_p[\gamma_n] = M < \infty$ else verifying the lower semi-continuity of $\Lambda_p$ is trivial.
        Moreover, we can, again without loss of generality, assume that $ \lim_{n\to\infty} \Lambda_p[\gamma_n] = \liminf_{n\to\infty} \Lambda_p[\gamma_n]$, by passing to a subsequence that attains the limit inferior.
        In fact, this allows us to assume that $\gamma_n \in \mathcal{T}_{2}^b$ for all $n \in \N$, again without loss of generality.
        By Lemma \ref{lemma:Psi-is-well-behaved} we obtain reals $m_f$ and $m_g$, noth both zero, such that for all $\tau \in \mathcal{T}_{2}^{b}$ and any $p \in \N_{\infty} = \N \cup \{\infty\}$ we have:
        \begin{align*}
           & m_f \leq \left[ \int_{\Omega} \left(\Psi^f[\tau](s, t)\right)^{2p} \, ds \right]^{\frac{1}{2p}} \quad \textup{for all } t \in [0,1] \\
           & m_g \leq \left[ \int_{\Omega} \Big(\Psi^g[\tau](s, t)\Big)^{2p} \, ds \quad \right]^{\frac{1}{2p}} \textup{for all } t \in [0,1]
        \end{align*}
        We use this to estimate:
        \begin{equation*}
          \left( m_f^{2p} + m_g^{2p} \right)^{\frac{1}{2p}} \, \left\| \dot{\gamma}_n \right\|_{L^{2p}([0,1])} \leq \left( \left\|\dot{\gamma}_n \, \Psi^f[\gamma_n] \right\|_{L^{2p}\left(\Omega_{[0,1]}\right)}^{2p} + \Big\| \dot{\gamma}_n \, \Psi^g[\gamma_n] \Big\|_{L^{2p}\left(\Omega_{[0,1]}\right)}^{2p} \right)^{\frac{1}{2p}} = \Lambda_p[\gamma_n]  \\
        \end{equation*}
        Noting that $\Lambda_p[\gamma_n] \to M < \infty$ we have that $\dot{\gamma}_n$ is bounded in $L^{2p}([0,1])$ for some $p \in \N$.
        Since $L^{2p}([0,1])$ is a reflexive Banach space, Theorem 7 in \citet{james1964weak} allows us to extract a subsequence that converges to some $\alpha \in L^{2p}([0,1])$ in the weak-$L^{2p}$ topology. 
        
        We now wish to show that $\alpha = \dot{\gamma}_\infty$.
        To see this, note that since $\dot{\gamma}_n$ is bounded in $L^{2p}([0,1])$ for $p\geq 1$ then, by the Poincar\'e inequality, it is also bounded in $H^1$. 
        Now recall that by the Rellich--Kondrachov theorem in $\R$ (Theorem 9.16 in \citet{brezis2011functional}) we have that $H^1$ compactly embeds into $C^0\left([0,1]\right)$.
        This implies then (Example 8.9 in \citet{dalmaso_gammaconvergence}) that $\gamma_n$ converges to $\gamma_\infty$ in $C^0\left([0,1]\right)$ also establishing that $\gamma_\infty \in \mathcal{T}_{2}^b$.
        We can now identify $\alpha$ as the weak derivative $\dot{\gamma}_\infty$ as follows: for any test function $\phi \in C^\infty_c([0,1])$ compute
        \begin{equation*}
            \int_{0}^{1} \alpha(t) \, \phi(t) \, dt = \lim_{i \to \infty} \int_{0}^{1} \dot{\gamma}_{n_i}(t) \, \phi(t) \, dt = - \lim_{i \to \infty} \int_{0}^{1} \gamma_{n_i}(t) \, \dot{\phi}(t) \, dt = - \int_{0}^{1} \gamma_\infty(t) \, \dot{\phi}(t) \, dt
        \end{equation*}
        where in the first equality we used the subsequential weak-$L^{2p}$ convergence $\dot{\gamma}_{n_i} \rightharpoonup \alpha$ and in the last equality we used the convergence of $\gamma_n \to \gamma_\infty$ in $C^0\left([0,1]\right)$. By definition, we have that $\alpha = \dot{\gamma}_\infty$ in $L^{2p}([0,1])$.
        Recall that our original goal was to show:
        \begin{equation*}
            \Lambda_p[\gamma_\infty] \leq \liminf_{n \to \infty} \Lambda_p[\gamma_n] \eqcolon M
        \end{equation*}
        Clearly, the subsequence $\tau_{n_i}$ satisfies $\Lambda_p[\gamma_{n_i}] \to M$ and if we verify that:
        \begin{equation*}
            \Lambda_p[\gamma_\infty] \leq \liminf_{i \to \infty} \Lambda_p[\gamma_{n_i}]
        \end{equation*}
        we are done. For this reason and by abuse of notation, we denote the subsequence $(\gamma_{n_i})_i$ by $(\gamma_n)_n$ for the rest of this proof.
        
        Now we will establish $\Psi^f[\gamma_n] \, \dot{\gamma}_n \rightharpoonup \Psi^f[\gamma_\infty] \, \dot{\gamma}_\infty$ in $L^{2p}\left(\Omega_{[0,1]}\right)$ using the notation $\Omega_{[0,1]} = \Omega \times [0,1]$.
        To see this, notice first that $\Psi^f[\gamma_n] \to \Psi^f[\gamma_\infty]$ in $L^{\infty}\left(\Omega_{[0,1]}\right)$ since:
        \begin{align*}
            & \sup_{(s,t) \in \Omega_{[0,1]}} \left| \frac{f(s)}{1 + f(s) \gamma_n(t)} - \frac{f(s)}{1 + f(s) \gamma_\infty(t)} \right| = \\ 
            &\sup_{(s,t) \in \Omega_{[0,1]}} \left| \frac{f(s)\left( \gamma_n(t) - \gamma_\infty(t)\right)}{ \left( 1 + f(s) \gamma_n(t) \right) \left(1 + f(s) \gamma_\infty(t) \right) } \right| \leq \\
            & \sup_{(s,t) \in \Omega_{[0,1]}} \left| \frac{f(s)}{\left( 1 + f(s) \gamma_n(t) \right) \left(1 + f(s) \gamma_\infty(t) \right) } \right| \, \sup_{t \in [0,1]} \left| \gamma_n(t) - \gamma_\infty(t) \right| \leq \\
            & \left| \frac{ \sup_s f(s) }{ \left( 1 + \inf_s f(s) \right)^2 } \right| \sup_{t \in [0,1]} \left| \gamma_n(t) - \gamma_\infty(t) \right| \to 0
        \end{align*}
        using that $\gamma_n \to \gamma_\infty$ in $L^\infty$ as well as that all $\tau \in \mathcal{T}_{2}^{b}$ are uniformly bounded above by $1$ and $f$ is continuous, bounded away from $-1$ in $\Omega$ by Lemma \ref{lemma:Psi-is-well-behaved}.
        Now for the H\"{o}lder conjugate $q$ of $2p$ and a test function $\psi \in L^q(\Omega_{[0,1]})$ we have:
        \begin{align*}
            & \int_{\Omega_{[0,1]}} \psi\, \left( \Psi^f[\gamma_n] \, \dot{\gamma}_n - \Psi^f[\gamma_\infty] \, \dot{\gamma}_\infty \right) = \\
            & \int_{\Omega_{[0,1]}} \psi\, \left( \Psi^f[\gamma_n] - \Psi^f[\gamma_\infty] \right) \, \dot{\gamma}_n + \int_{\Omega_{[0,1]}} \psi\, \Psi^f[\gamma_\infty] \,\left(  \dot{\gamma}_n - \dot{\gamma}_\infty \right)
        \end{align*}
        Since $\Psi^f[\gamma_n] \in L^\infty(\Omega_{[0,1]})$ (see Lemma \ref{lemma:Psi-is-well-behaved}) the second term vanishes by the weak convergence of $\dot{\gamma}_n$ to $\dot{\gamma}_\infty$ in\footnote{
            More precisely, since $\gamma_n \rightharpoonup \gamma_\infty$ in $L^{2p}([0,1])$ but $\psi \in L^q(\Omega_{[0,1]})$, with $\frac{1}{2p} + \frac{1}{q} = 1$, one has to use the Fubini theorem (Theorem 3.4.4 in \citet{bogachev2007measure}) to see that the map:
            \begin{align*}
                [0,1] &\to \R \\
                t &\mapsto \int_{\Omega} \psi(s,t) \, ds
            \end{align*}
            is in $L^q([0,1])$ and so the weak convergence of $\dot{\gamma}_n$ to $\dot{\gamma}_\infty$ in $L^{2p}([0,1])$ can be used to conclude.
            Indeed, by the Fubini theorem and the Jensen inequality we have:
            \begin{equation*}
                \int_{0}^{1} \left( \int_{\Omega} \psi(s,t) \, ds \right)^{q} \, dt \leq \int_{0}^{1} \int_{\Omega} \left( \psi(s,t) \right)^{q} \, ds \, dt = \int_{\Omega_{[0,1]}} \left( \psi(s,t) \right)^{q} \, ds \otimes dt < \infty
            \end{equation*}
        } $L^{2p}([0,1])$.
        For the first term, we argue as follows:
        \begin{equation*}
            \left| \int_{\Omega_{[0,1]}} \left( \psi \,  \dot{\gamma}_n \right) \, \left( \Psi^f[\gamma_n] - \Psi^f[\gamma_\infty] \right) \right| \leq  \sup_{(s,t) \in \Omega_{[0,1]}}\left| \Psi^f[\gamma_n] - \Psi^f[\gamma_\infty] \right| \, \left( \int_{\Omega_{[0,1]}} | \psi \, \dot{\gamma}_n |\, dt\right) \to 0
        \end{equation*}
        using that $\Psi^f[\gamma_n] \to \Psi^f[\gamma_\infty]$ in $L^\infty\left( \Omega_{[0,1]} \right)$ and bounding the second integral uniformly over $n$ by a H\"{o}lder inequality, noting that $\psi \in L^q\left( \Omega_{[0,1]} \right), \dot{\gamma_n} \in L^{2p}([0,1])$, with $\frac{1}{q} + \frac{1}{2p} = 1$, and we have shown that the $\dot{\gamma}_n$ are bounded in $L^{2p}([0,1])$.
        Thus, we have proven that:
        \begin{equation*}
            \Psi^f[\gamma_n] \, \dot{\gamma}_n \rightharpoonup \Psi^f[\gamma_\infty] \, \dot{\gamma}_\infty \text{ in } L^{2p}\left( \Omega_{[0,1]} \right)
        \end{equation*}
        Here, we use the standard fact\footnote{
            To see this, one idea presented in \citet{pham2025notes} is to assume $f_n \rightharpoonup f$ in $L^p$ and use the test function:
            \begin{equation*}
                \psi = \frac{\text{sgn}(f)}{\norm{f}_{L^{p}}^{p/q}} \, |f|^{p/q}
            \end{equation*}
            where $p$ and $q$ are H\"{o}lder conjugate. Note that we have $\norm{\psi}_{L^{q}} = 1$ and $\int f \, \psi = \norm{f}_{L^{p}}$.
            Now using the weak convergence together with the H\"{o}lder inequality, we obtain:
            \begin{equation*}
                \norm{f}_{L^{p}} = \int f \, \psi = \lim_{n\to\infty} \int f_n \, \psi \leq \liminf_{n\to\infty} \norm{f_n}_{L^p} \, \norm{\psi}_{L^q} = \liminf_{n\to\infty} \norm{f_n}_{L^p} 
            \end{equation*}
        } that the map $L^{2p} \to \R$ given by $h \mapsto \norm{h}_{L^{2p}\left(\Omega_{[0,1]}\right)}$ is lower-semi-continuous in the weak-$L^{2p}$ topology to conclude that:
        \begin{equation*}
            \gamma \mapsto \norm{\Psi^f[\gamma] \, \dot{\gamma}}_{L^{2p}\left(\Omega_{[0,1]}\right)}
        \end{equation*}
        is lower semi-continuous in weak-$L^{2p}\left(\Omega_{[0,1]}\right)$. The exact same argument establishes that the map:
        \begin{equation*}
            \gamma \mapsto \Big\|\Psi^g[\gamma] \, \dot{\gamma}\Big\|_{L^{2p}\left(\Omega_{[0,1]}\right)}
        \end{equation*}
        is also lower semi-continuous in weak-$L^{2p}\left(\Omega_{[0,1]}\right)$.

        To conclude, start by recalling that the limit inferior commutes with continuous, non-decreasing functions. 
        We can thus write:
        \begin{align*}
            \liminf_{n \to \infty} \Lambda_p[\gamma_n] &= \liminf_{n \to \infty} \left( \norm{\Psi^f[\gamma_n] \, \dot{\gamma}_n }_{L^{2p}\left(\Omega_{[0,1]}\right)}^{2p} + \Big\|\Psi^g[\gamma_n] \, \dot{\gamma}_n \Big\|_{L^{2p}\left(\Omega_{[0,1]}\right)}^{2p} \right)^{\frac{1}{2p}} \\
            &= \left( \liminf_{n \to \infty} \left( \norm{\Psi^f[\gamma_n] \, \dot{\gamma}_n}_{L^{2p}\left(\Omega_{[0,1]}\right)}^{2p} + \Big\|\Psi^g[\gamma_n] \, \dot{\gamma}_n \Big\|_{L^{2p}\left(\Omega_{[0,1]}\right)}^{2p} \right) \right)^{\frac{1}{2p}} \\
            &\geq \left( \liminf_{n \to \infty} \norm{\Psi^f[\gamma_n] \, \dot{\gamma}_n}_{L^{2p}\left(\Omega_{[0,1]}\right)}^{2p} + \liminf_{n \to \infty} \Big\|\Psi^g[\gamma_n] \, \dot{\gamma}_n \Big\|_{L^{2p}\left(\Omega_{[0,1]}\right)}^{2p} \right)^{\frac{1}{2p}} \\
            &\geq \left( \norm{\Psi^f[\gamma_\infty] \, \dot{\gamma}_\infty}_{L^{2p}\left(\Omega_{[0,1]}\right)}^{2p} + \Big\|\Psi^g[\gamma_\infty] \, \dot{\gamma}_\infty \Big\|_{L^{2p}\left(\Omega_{[0,1]}\right)}^{2p} \right)^{\frac{1}{2p}} \\
            &= \Lambda_p[\gamma_\infty]
        \end{align*}
        and in the third line we are using the boundedness of $\Psi^f[\gamma_n] \, \dot{\gamma}_n$ and $\Psi^g[\gamma_n] \, \dot{\gamma}_n$ in $L^{2p}\left(\Omega_{[0,1]}\right)$ that is evident since both terms lower bound $\Lambda_p[\gamma_n]$ which was assumed to be finite, for each $n \in \N$, and converging to $M< \infty$.
        This completes the proof.
    \end{proof}

  \begin{proposition}
    \label{prop:equi-coercive-Lp}
    The family $\left( \Lambda_p \right)_{p \in \N}$ is equi-coercive on $\mathcal{T}_{2}$ equipped with the weak-$H^1$ topology.
\end{proposition}
\begin{proof}
    To see this, we first leverage Proposition 7.7 in \citet{dalmaso_gammaconvergence} allowing us to prove equicoercivity by finding a lower semi-continuous coercive functional $\Phi: \mathcal{T}_{2} \to \R$
    such that $\Phi \leq \Lambda_p$ for all $p$.
    Recall that the functional $\Phi$ is coercive if for any $\alpha \in \R$ the sublevel set $\{ \tau \in \mathcal{T}_{2} : \Phi[\tau] \leq \alpha \}$ has compact closure; see Definitions 1.10 and 1.12 and Remark 1.11 in \citet{dalmaso_gammaconvergence}.
    Since $\mathcal{T}_{2}$ is homeomorphic to a linear subspace of the reflexive Hilbert space $H^1_0([0,1])$ it follows by Theorem 7 in \citet{james1964weak} that the closure of $H^1$ bounded sets is weak-$H^1$ compact.
    Thus, to show coercivity of such a $\Phi$ in weak-$H^1$ is suffices to show that for any $\alpha \in \R$ the sublevel set $\{ \tau \in \mathcal{T}_{2} : \Phi[\tau] \leq \alpha \}$ is bounded in $H^1$.

    Start by using that by Lemma \ref{lemma:monotonicty-Lp} the number $\sqrt{2 \, \Leb{\Omega}}$ is a uniform bound for the constants $c_{p,q}$.
    Define:
    \begin{equation*}
        \Phi = \frac{\Lambda_1}{ \sqrt{2 \, \Leb{\Omega}} } 
    \end{equation*}
    Then, by Lemma \ref{lemma:monotonicty-Lp} we have that
    \begin{equation}
        \label{eq:Phi-is-lower-bound}
        \Phi[\tau] \leq \Lambda_p[\tau] \quad \text{for all} \quad p\in \N
    \end{equation}
    for all $\tau\in \mathcal{T}_{2}^{b}$ whereas for $\tau \in \mathcal{T}_{2}$ this is inquality reduces to the triviality $+ \infty \leq + \infty$. 
    Now $\Phi$ is lower semi-continuous by Lemma \ref{lemma:lower-semi-cont-of-functionals} and so it suffices to show it is coercive.
    To see this, use Lemma \ref{lemma:Psi-is-well-behaved} to obtain reals $m_f$ and $m_g$, noth both zero, such that for all $\tau \in \mathcal{T}_{2}^{b}$ and any $p \in \N_{\infty} = \N \cup \{\infty\}$ we have:
    \begin{align*}
       & m_f \leq \left[ \int_{\Omega} \left(\Psi^f[\tau](s, t)\right)^{2p} \, ds \right]^{\frac{1}{2p}} \quad \textup{for all } t \in [0,1] \\
       & m_g \leq \left[ \int_{\Omega} \Big(\Psi^g[\tau](s, t)\Big)^{2p} \, ds \quad \right]^{\frac{1}{2p}} \textup{for all } t \in [0,1] \, .
    \end{align*}
    We use this to estimate:
    \begin{equation*}
      \left( m_f^2 + m_g^2 \right)^{\frac{1}{2}} \, \left\| \dot{\tau}_n \right\|_{L^{2}([0,1])} \leq \left( \left\|\dot{\tau}_n \, \Psi^f[\tau_n] \right\|_{L^{2}\left(\Omega_{[0,1]}\right)}^{2} + \Big\| \dot{\tau}_n \, \Psi^g[\tau_n] \Big\|_{L^{2}\left(\Omega_{[0,1]}\right)}^{2} \right)^{\frac{1}{2}} = \Lambda_1[\tau_n] \, .  \\ 
    \end{equation*}
    Combining this with the Poincar\'e inequality we have:
    \begin{equation}
        \label{eq:H1-norm-coercivity}
        \norm{\tau}_{H^1([0,1])} \leq C_{[0,1]} \norm{\dot{\tau}}_{L_{2}([0,1])} \leq  \frac{C_{[0,1]} \sqrt{2 \, \Leb{\Omega}}}{ \left( m_f^2 + m_g^2 \right)^{\frac{1}{2}}} \, \Phi[\tau]
    \end{equation}
    and the constant $C_{[0,1]}$ depends only on the domain $[0,1]$.
    In particular, this implies that $\Phi[\tau_i] \to \infty$ whenever $\tau_i \to \infty$ in $H^1$ and as a result the sublevel sets $\{ \Phi \leq \alpha \}$ are bounded in $H^1([0,1])$, for any $\alpha \in \R$.
    Combining this with the discussion at the beginning of the proof we can conclude.
\end{proof}
This proof leads to a useful lemma:
\begin{lemma}
    \label{lemma:boundedness-in-H1}
    Assume that for some sequence $\gamma_p \in \mathcal{T}_{2}$ there is a constant $0 \leq C < \infty$ such that:
    \begin{equation*}
        \Lambda_p[\gamma_p] \leq C \quad \text{for all} \quad p \in \N
    \end{equation*}  
    then, $\gamma_p$ is bounded in $H^1([0,1])$.
\end{lemma}
\begin{proof}
    First, notice that $\gamma_p \in \mathcal{T}_{2}^b$ for all $p \in \N$ since $\Lambda_p[\gamma_p] < \infty$.
    Then, apply equations \eqref{eq:Phi-is-lower-bound} and \eqref{eq:H1-norm-coercivity} with $C_{[0,1]}$ the Poincar\'e constant for the interval $[0,1]$:
    \begin{equation*}
        \norm{\gamma_p}_{H^1([0,1])} \leq C_{[0,1]} \, \norm{\dot{\gamma}_p}_{L^2([0,1])} \leq \frac{C_{[0,1]} \, \sqrt{2 \, \Leb{\Omega}}}{ \left( m_f^2 + m_g^2 \right)^{\frac{1}{2}} } \Lambda_p[\gamma_p] \leq \frac{C_{[0,1]} \, \sqrt{2 \, \Leb{\Omega}}}{ \left( m_f^2 + m_g^2 \right)^{\frac{1}{2}} } C < \infty
    \end{equation*}
    for all $p \in \N$.
\end{proof}

\begin{proposition}
    \label{prop:gamma-limit-Lp}
    There are constants $b_p \in \R$ such that $\frac{1}{b_p} \Lambda_p \to \Lambda_\infty$ in a $\Gamma$-sense, where the $\Gamma$-convergence is taken in the space $\mathcal{T}_2$ equipped with the weak $H^1$ topology.
\end{proposition}
\begin{proof}
    Start by recalling that by conditions (e)-(f) of \cite[Proposition 8.1]{dalmaso_gammaconvergence} it suffices to show that for any sequence $\gamma_p \to \gamma_\infty$ in $\mathcal{T}_{2}$ we have:
    \begin{equation}
        \label{eq:seq-gamma-limit-1}
        \liminf_{p \to \infty} \Lambda_p[\gamma_p] \geq \Lambda_\infty[\gamma_\infty]
    \end{equation}
    as well as that for any $\gamma_\infty \in \mathcal{T}_{2}$ there exists a sequence $\gamma_p \to \gamma_\infty$ such that:
    \begin{equation}
        \label{eq:seq-gamma-limit-2}
        \limsup_{p \to \infty} \Lambda_p[\gamma_p] \leq \Lambda_\infty[\gamma_\infty]
    \end{equation}
    We start by proving \eqref{eq:seq-gamma-limit-1}. Fix a sequence $\gamma_p \in \mathcal{T}_{2}$ converging to some $\gamma_\infty \in \mathcal{T}_{2}$, i.e.\ $\gamma_p \rightharpoonup \gamma_\infty$ in $H^1$. 
    Without loss of generality we can assume that $\liminf_p \Lambda_p[\gamma_p] < \infty$ and moreover that the sequence $\gamma_p$ attains the infimum in the sense that:
    \begin{equation}
        \label{eq:wlog-assumption-of-finite-sequence}
        \lim_{p \to \infty} \Lambda_p[\gamma_p] = \liminf_p \Lambda_p[\gamma_p] < \infty
    \end{equation}
    Moreover, we can $\{ \gamma_p \}_p \subset \mathcal{T}_{2}^b$ by removing the finitely many $\gamma \in \mathcal{T}_{2}^b \setminus \mathcal{T}_{2}$ that evaluate to $+\infty$ under $\Lambda_p$. 
    
    Our first task is to show that $\gamma_\infty \in \mathcal{T}_{2}^b$. 
    To do this, we appeal to Lemma \ref{lemma:boundedness-in-H1} yielding that $\{\gamma_p\}_p$ is a bounded set in $H^1([0,1])$.
    By \citet[Example 8.9]{dalmaso_gammaconvergence} and the Rellich–Kondrachov theorem \citep[Theorem 9.16]{brezis2011functional} in $\R$ it follows that $\gamma_n$ converges to $\gamma_\infty$ strongly in $C^0\left([0,1]\right)$.
    It then follows that $\gamma_\infty \in \mathcal{T}_{2}^b$ since $\gamma_p \in \mathcal{T}_{2}^b$ for all $p$.

    Now we can apply Lemma \ref{lemma:monotonicty-Lp} to obtain constants $ a_p \coloneq 2^{\frac{1}{p(p+1)}} $ such that for any $\gamma \in \mathcal{T}_{2}^b$ we have:
    \begin{equation}
        \label{eq:aux-Lp-bound-in-rescaling}
        \frac{1}{a_p} \Lambda_p[\gamma] \leq\Lambda_{p+1}[\gamma]
    \end{equation}
    moreover, notice that $b_p = \prod_{k=p}^{\infty} a_k$ is bounded by $2^{\sum_{i=1}^\infty \frac{1}{k^2} }$ uniformly in $p$. This, together with \eqref{eq:aux-Lp-bound-in-rescaling} yields:
    \begin{equation}
        \label{eq:aux-Lp-bound-in-rescaling-final}
        \frac{1}{b_p} \Lambda_p[\gamma] \leq \frac{1}{b_q} \Lambda_{q}[\gamma] \leq \Lambda_\infty[\gamma]
    \end{equation}
    for any $\gamma \in \mathcal{T}_{2}^b$ and all $p \leq q$.
    Using equation \eqref{eq:aux-Lp-bound-in-rescaling-final} together with Lemma \ref{lemma:lower-semi-cont-of-functionals} we have, for each $p$:
    \begin{equation}
        \frac{1}{b_p} \Lambda_p[\gamma_\infty] \leq \liminf_{n\to\infty} \frac{1}{b_p} \Lambda_p[\gamma_n] \leq \liminf_{n\to\infty} \frac{1}{b_n} \Lambda_n[\gamma_n]
    \end{equation}
    Finally, using the point-wise convergence established\footnote{This is where the fact that $\gamma_\infty \in \mathcal{T}_{2}^b$ is needed.} in Lemma \ref{lemma:Linfty-to-Lp-appendix} together with $b_p \to 1$ we obtain:
   
    \begin{equation*}
        \Lambda_\infty[\gamma_\infty] \leq \liminf_{n\to\infty} \frac{1}{b_n} \Lambda_n[\gamma_n]
    \end{equation*}
    The second condition \eqref{eq:seq-gamma-limit-2} is trivially satisfied by taking $\gamma_p = \gamma_\infty$ for all $p$ and using the point-wise convergence of Lemma \ref{lemma:Linfty-to-Lp-appendix} together with $\frac{1}{b_p} \to 1$.
\end{proof}
Now we can proceed with the proof of Theorem \ref{thm:Linfty-to-Lp-main-text}.
We recall the statement first:
\begin{theorem*}
    \label{thm:Linfty-to-Lp-appendix}
    Define $\mathcal{T}_2 = \left\{ \tau \in W^{1,2}([0,1]) : \tau(0) = 0, \tau(1) = 1 \right\}$ where $W^{1,2}([0,1])$ is the Sobolev space of functions with square integrable first derivative on the unit interval.
    For each positive integer $p$, consider the optimization problem:
    \begin{equation}
        \label{eq:opt-problem-Lp-approx}
        \inf_{\tau \in \mathcal{T}_2} \, \int_{0}^{1} \, \lambda_p(\tau(t), \dot{\tau}(t)) \, dt
    \end{equation}
    with:
    \begin{align*}
        \lambda_p(\tau(t), \dot{\tau}(t)) \coloneq \dot{\tau}(t)^{2p} \, \int_{\Omega} \frac{f(s)^{2p}}{(1 + \tau(t) f(s))^{2p}} \, ds +\dot{\tau}(t)^{2p} \, \int_{\Omega} \frac{g(s)^{2p}}{(1 + \tau(t) g(s))^{2p}} \, ds
    \end{align*}
    If $\tau_p$ is optimal for \eqref{eq:opt-problem-Lp-approx} then any subsequential $L^2$ limit $\tau_{p_j} \overset{L^2}{\to} \tau_\infty$ as $j\to\infty$ is optimal for \eqref{eq:opt-problem}.
\end{theorem*}
\begin{proof}
    Start by noticing that:
    \begin{align*}
        \arg \min_{\tau \in \mathcal{T}_{\infty}} \int_{0}^{1} \lambda_p(\tau(t), \dot{\tau}(t)) \, dt &= \arg \min_{\tau \in \mathcal{T}_{\infty}} \left( \int_{0}^{1} \lambda_p(\tau(t), \dot{\tau}(t)) \, dt \right)^{\frac{1}{2p}} \\ 
        &= \arg \min_{\tau \in \mathcal{T}_{\infty}} \Lambda_{p}[\tau] 
    \end{align*}
    and the equality holds in a sense of sets. This follows from the definition of $\Lambda_p$ as well as that $\lambda_p \geq 0$ and $x \mapsto x^{1 / 2p}$ is increasing in $\R^+$.
    Notice, moreover, that $\Lambda_{\infty}$ restricted to $\mathcal{T}_\infty \subset \mathcal{T}_2$ coincides with the objective $\Lambda$ of problem \eqref{eq:opt-problem-Lp-approx}.

    Now we apply the \textit{fundamental theorem of $\Gamma$-convergence}, Theorem 2.10 in \citet{braides2006handbook} appearing also as Corollary 7.20 in \citet{dalmaso_gammaconvergence}:
    by Proposition \ref{prop:gamma-limit-Lp} there are constants $b_p \to 1$ such that $(1 / b_p)\, \Lambda_p \to \Lambda_\infty$ in a $\Gamma$-sense and by Proposition \ref{prop:equi-coercive-Lp} it follows that the $(1/b_p)\Lambda_p$ is equi-coercive in the weak $H^1$ topology,
    thus it follows that if minimizers $ \tau_p \in \arg\min (1 / b_p) \, \Lambda_p$ form a pre-compact set in $\mathcal{T}_{2}^{b}$ then any weak $H^1$ subsequential limit of $\tau_p$ will converge to a minimizer of $\inf \Lambda_{\infty}$.
    Note that since $b_p > 0$ the $\tau_p$ will also be minimizers of $\Lambda_p$. 

    To establish pre-compactness in weak $H^1$ we show $H^1$ boundedness (Theorem 7 in \citet{james1964weak}).
    Due to the $\Gamma$-convergence and equicoercivity of the $(1/b_p) \, \Lambda_p$ we can apply Theorem 7.8 in \citet{dalmaso_gammaconvergence} to obtain:
    \begin{equation*}
        \lim_{p \to \infty} \inf_{\tau \in \mathcal{T}_{2}^{b}} \frac{1}{b_p} \Lambda_p[\tau] = \inf_{\tau \in \mathcal{T}_{2}^{b}} \Lambda_\infty[\tau] < \infty
    \end{equation*}
    so there is a constant $C \geq 0$ such that the sequence of minimizers $\tau_p$ satisfies $\Lambda_p[\tau_p] \leq C < \infty$ for all $p$. 
    Applying Lemma \ref{lemma:boundedness-in-H1} we have that the set $\{ \tau_p \}_p$ is bounded in $H^1([0,1])$ and so by \citet[Theorem 7]{james1964weak} it is pre-compact in the weak $H^1$ topology.
    
    Thus, we have established that subsequential weak-$H^1$ limits of the $\tau_p$ will be minimizers of $\Lambda_\infty$. As a final step, note that by Corollary 8.8 in \citet{dalmaso_gammaconvergence} (see also Example 8.9) and the Rellich--Kondrachov theorem \citep[Theorem 9.16]{brezis2011functional}, inside any $H^1$ norm bounded set weak $H^1$ convergence is equivalent to strong convergence in $L^2$.
    We just showed that the set $\{\tau_p\}_p$ is norm bounded, so we can conclude.
\end{proof}

\subsection{Proof of Theorem \ref{thm:characterizarion-of-Lp-solutions-main-text}}
\label{subsec:solving-the-Lp-ODE-variational-caculus}
In this section, we solve the Lagrangian problem \eqref{eq:opt-problem-Lp-approx-in-theorem} by employing the direct method of the Calculus of Variations.
First, some definitions:
\begin{definition}
    The function
    \begin{align*}
        & \lambda_p: \R \times \R \to \R \\
        & \lambda_p:  (x,v) \mapsto\int_{\Omega} \left( \frac{v f(s)}{1 + x f(s)} \right)^{2p} \, ds + \int_{\Omega} \left( \frac{v g(s)}{1 + x g(s)} \right)^{2p} \, ds
    \end{align*}
    will be called the \textit{Lagrangian} of problem \eqref{eq:opt-problem-Lp-approx-in-theorem}.
\end{definition}
Now recall the spaces introduced in equations \eqref{eq:intro-T-2p}--\eqref{eq:intro-T-infty} of the main text. For $p \in \N$ we have:
\begin{align*}
    \mathcal{T}_{2p} &\coloneq \left\{ \tau \in W^{1,2p}([0,1]) \, : \, \tau(0) = 0 \, , \, \tau(1) = 1 \right\} \\
    \mathcal{T}_{2p}^{b} &\coloneq  \left\{ \tau \in \mathcal{T}_{2p} \, : \, 0 \leq \tau \leq 1 \right\} \\
    \mathcal{T}_{\infty} &\coloneq \left\{ \tau \in C^1([0,1]) \, : \, \tau(0) = 0 \, , \, \tau(1) = 1 \, , \, \dot{\tau} \geq 0 \right\}
\end{align*}
where $W^{1,2p}([0,1])$ is the Sobolev space of functions with square integrable first derivative on the unit interval.
\begin{remark}
    Notice that
    \begin{equation*}
       \mathcal{T}_\infty \subset \mathcal{T}^b_{2p} \subset \mathcal{T}_{2p} 
    \end{equation*}
    for any $p \in \N$.
\end{remark}
First, we establish that to solve problem \eqref{eq:opt-problem-Lp-approx-in-theorem} it is sufficient to optimize over the space $\mathcal{T}_{2p}^{b}$.
For brevity, we denote problem \eqref{eq:opt-problem-Lp-approx-in-theorem} by:
\begin{equation}
    \label{eq:main-lagrangian-problem-abbreviated-in-appendix}
    \inf_{\tau \in \mathcal{T}_2} \int \lambda_p (\tau, \dot{\tau} ) \coloneq \inf_{\tau \in \mathcal{T}_2} \int_{0}^{1} \lambda_p(\tau(t), \dot{\tau}(t)) \, dt
\end{equation}
We now have:
\begin{lemma}
    \label{lemma:problem-can-be-restricted}
    For each $p \in \N$ we have:
    \begin{equation*}
        \inf_{\tau \in \mathcal{T}_{2p}^b} \int \lambda_p (\tau, \dot{\tau} ) = \inf_{\tau \in \mathcal{T}_{2}} \int \lambda_p (\tau, \dot{\tau} )
    \end{equation*}
\end{lemma}
\begin{proof}
    Since for each $p \in \N$ we have $\mathcal{T}_{2p}^{b} \subset \mathcal{T}_2$ it suffices to prove that:
    \begin{equation*}
        \inf_{\tau \in \mathcal{T}_{2p}^b} \int \lambda_p (\tau, \dot{\tau} ) \leq \inf_{\tau \in \mathcal{T}_{2}} \int \lambda_p (\tau, \dot{\tau} )
    \end{equation*}
    Now if $p > 1$ recall that for any $\tau \in \mathcal{T}_2 \setminus \mathcal{T}_{2p}$ we have
    \begin{equation*}
        \int \lambda_p \left( \tau, \dot{\tau} \right) = \infty \, .
    \end{equation*}
    thus, one has
    \begin{equation*}
        \inf_{\tau \in \mathcal{T}_{2p}} \int \lambda_p (\tau, \dot{\tau} ) \leq \inf_{\tau \in \mathcal{T}_{2}} \int \lambda_p (\tau, \dot{\tau} ) \, .
    \end{equation*}
    It suffices to show that for any $\tau \in \mathcal{T}_{2p}$ there exists a $\tau_b \in \mathcal{T}_{2p}^{b}$ such that:
    \begin{equation*}
        \int \lambda_p (\tau_b, \dot{\tau}_b ) \leq \int \lambda_p (\tau, \dot{\tau} )
    \end{equation*}
    To that end, take any path $\tau \in \mathcal{T}_{\infty}$ and consider the modified path $\tau_b \in \mathcal{T}_{2p}^{b}$ defined by:
    \begin{equation}
        \tau_b(t) \coloneq \min\left\{ \max \left\{0, \tau(t) \right\} , 1 \right\}
    \end{equation} 
    Due to Theorem 2.1.11 in \citet{ziemer2012weakly} we have that $\tau_b \in W^{1,q}([0,1])$. Indeed, the maps $x \mapsto \max\{0,x\}$ and $x \mapsto \{x, 1\}$ are Lipschitz and $\tau_b$ is in $L^p([0,1])$, for any positive integer $p$, since the Sobolev embedding guarantees $\tau \in C^0\left([0,1]\right)$.
    We now show that $\tau_b$ attains a lower value for our objective. 
    Notice that for any $t$ such that $\tau(t) \neq \tau_b(t)$ we have $\tau(t) < 0$ or $\tau(t) > 1$. By continuity, there exists $\epsilon > 0$ such that for any $s \in (t-\epsilon, t +\epsilon)$ we have $\tau(s) < 0$ or $\tau(t) > 1$.
    This implies that for all $s \in (t-\epsilon, t +\epsilon)$ we have $\tau_b(s) = 0$ or $\tau_b(s) = 1$ and hence $\dot{\tau}_b(s) = 0$ for a.e. $s \in (t-\epsilon, t +\epsilon)$.
    This last part follows from the fact that functions in $W^{1,2p}([0,1])$, for $p\geq1$, satisfy the fundamental theorem of calculus and hence for any $0 < \delta \leq \epsilon$ we have:
    \begin{equation*}
        \int_{t-\delta}^{t+\delta} \dot{\tau}_b(s) \, ds = \tau_b(t+\delta) - \tau_b(t-\delta) = 0
    \end{equation*}
    Taking $\delta \searrow 0$ and using the Lebesgue differentiation theorem we obtain $\dot{\tau}_b(s) = 0$ for a.e. $s \in (t-\epsilon, t +\epsilon)$.
    Finally, notice that $\lambda_p(\tau(t), \dot{\tau}(t)) = 0$ whenever $\dot{\tau}(t) = 0$ so for a set $\mathcal{N}$ of measure zero we have $\left\{t: \tau_b(t) \neq \tau(t) \right\} \subset \left\{ t : \lambda_p(\tau(t), \dot{\tau}(t)) = 0 \right\} \cup \mathcal{N}$ and so we can estimate:
    \begin{align*}
        \int_{0}^{1} \lambda_p(\tau_b(t), \dot{\tau}_b(t)) \, dt &= \int_{ \left\{t: \tau_b(t) \neq \tau(t) \right\} } \lambda_p(\tau_b(t), \dot{\tau}_b(t)) \, dt + \int_{ \left\{t: \tau_b(t) = \tau(t) \right\} } \lambda_p(\tau_b(t), \dot{\tau}_b(t)) \, dt \\
                                                                 &= \int_{ \left\{t: \tau_b(t) = \tau(t) \right\} } \lambda_p(\tau(t), \dot{\tau}(t)) \, dt \\
                                                                 &\leq \int_{0}^{1} \lambda_p(\tau(t), \dot{\tau}(t)) \, dt
    \end{align*}
    and the last step follows since $\lambda_p \geq 0$.
    This completes the proof.
\end{proof}
The above proof establishes that to solve problem \eqref{eq:opt-problem-Lp-approx} we can instead solve:
\begin{equation}
    \label{eq:opt-problem-in-Tau-2p-b}
    \inf_{\tau \in \mathcal{T}_{2p}^{b}} \int \lambda_p(\tau, \dot{\tau}) \coloneq \inf_{\tau \in \mathcal{T}_{2p}^{b}} \int_{0}^{1} \lambda_p(\tau(t), \dot{\tau}(t)) \, dt
\end{equation}
Our goal is prove that problem \eqref{eq:opt-problem-in-Tau-2p-b} has a minimizer in $C^2([0,1])$ that satisfies the strong Euler--Lagrange equations.
To achieve that goal, we will need to inroduce an auxiliary Lagrangian:
\begin{definition}
    \label{def:auxiliary-Lagrangian}
    For $p \in \N$ and some $\epsilon > 0$ we define the auxiliary Lagrangian:
    \begin{align*}
        & \lambda_p^\epsilon: [0,1] \times \R \to \R \\
        & \lambda_p^\epsilon:  (x,v) \mapsto\int_{\Omega} \left( \frac{v f(s)}{1 + x f(s)} \right)^{2p} \, ds + \int_{\Omega} \left( \frac{v g(s)}{1 + x g(s)} \right)^{2p} \, ds + \epsilon \, v^2
    \end{align*}
\end{definition}
\begin{remark}
    The reason we cannot work directly with the original Lagrangian is the following: our goal, here, is to identitfy a minimizer of problem \eqref{eq:opt-problem-Lp-approx}.
    One standard approach would be to write the stong Euler--Lagrange equations for the Lagrangian $\lambda_p$, solve it, and claim that this solution is optimal for \eqref{eq:opt-problem-Lp-approx}.
    This, however, crucially relies on the convexity of the map $(x,v) \mapsto \lambda_p(x,v)$ which does not hold here.\footnote{
    In fact, this non-convexity is crucial to the structure of the problem: physically, it is exactly this interaction between the position and velocity variables that allow us
    to reduce the \say{cost} at position $x$ by traveling at low speeds $v$.}

    A second approach would be to take a minimizer $\tau_p$, which can be shown to exist, and characterize it by showing it satisfies the strong Euler--Lagrange equations.
    Although insightful, this approach cannot be rigorously justified since it assumes the minimizer it at least of class $C^2([0,1])$.
    To show this, one would employ standard regularity theory, requiring:
    \begin{equation*}
        \frac{\partial^2 \lambda_p}{\partial v^2} (x,v) > 0 \textup{ for all } (x,v) \in [0,1] \times \R
    \end{equation*}
    Unfortunately, this condition is violated here but only mildly. Indeed, we have:
    \begin{align*}
        \frac{\partial^2 \lambda_p}{\partial v^2} \geq 0 \text{ and } \frac{\partial^2 \lambda_p}{\partial v^2} = 0 \textup{ if and only if } v = 0
    \end{align*} 
    Thus, it will be possible to show that the minimizer $\tau_p$ of \eqref{eq:opt-problem-Lp-approx} satisfies the strong Euler--Lagrange (sEL) equations by constructing a perturbed $\epsilon$-Lagrangian with smooth minimizers $\tau^\epsilon_p$, satisfying the sEL equations,
    and recovering $\tau_p$ as an $L^2$ subsequential limit of $\tau^\epsilon_p$ as $\epsilon \to 0$. This will, again, involve the theory of $\Gamma$-convergence.
\end{remark}
We now study the family of problems:
\begin{equation}
    \label{eq:expanded-opt-problem-epsilon-p}
    \inf_{\tau \in \mathcal{T}_{2}} \int \lambda^\epsilon_p(\tau, \dot{\tau}) \coloneq \inf_{\tau \in \mathcal{T}_{2}} \int_{0}^{1} \lambda_p^\epsilon(\tau(t), \dot{\tau}(t)) \, dt
\end{equation}
First, we establish existence of a minimizer:
\begin{lemma}
    \label{lemma:existence-of-minimizer-in-B}
    The problem \eqref{eq:expanded-opt-problem-epsilon-p} has a minimizer in $\mathcal{T}_{2p}^{b}$.
\end{lemma}
\begin{proof}
    We now apply Theorem 4.1 in \citet{dacorogna2007calculusofvariations} with a slight modification to account for the additional constraint $\{ \tau  \, : \, 0 \leq \tau \leq 1 \}$. Indeed, the theory developed in 
    \citet{dacorogna2007calculusofvariations} applies directly to the problem:
    \begin{equation*}
        \inf_{\tau \in \mathcal{T}_{2p}} \int_{0}^{1} \lambda_p^\epsilon \left( \tau(t), \dot{\tau}(t) \right) \, dt \, .
    \end{equation*}
    However, the proof of Theorem 4.1 in \citet{dacorogna2007calculusofvariations} carries through for $\mathcal{T}_{2p}^{b}$ identically, since the set $\{ \tau  \, : \, 0 \leq \tau \leq 1 \}$ is closed inside $W^{1,2p}([0,1])$, due to the Sobolev embedding theorem.
    To apply Theorem 4.1 in \citet{dacorogna2007calculusofvariations} we need to, first, verify the convexity of the map:
    \begin{align*}
        v \mapsto \lambda_p^\epsilon(x,v)
    \end{align*}
    for all $x \in [0,1]$. This is immediate since the map in question is just:
    \begin{equation*}
        v \mapsto v^{2p} \,\, \left[ \int_{\Omega} \left(\frac{1}{1 + x f(s)} \right)^{2p} \, ds + \int_{\Omega} \left(\frac{1}{1 + x g(s)} \right)^{2p} \, ds \right] + \epsilon \, v^2
    \end{equation*}    
    and the term in parentheses is finite for all $x$. The second and final condition to be verified is the coercivity of $(x,v) \mapsto \lambda_p^\epsilon(x,v)$, i.e., we need to find $\alpha_1 > 0$ and $\alpha_2, \alpha_3 \in \R$ such that:
    \begin{equation*}
        \lambda_p^\epsilon(x,v) \geq \alpha_1 \, |v|^{2p} + \alpha_2 \, |x|^q + \alpha_3
    \end{equation*}
    for positive integer $q$ such that $2p > q$.
    To do so, notice first that for any $\epsilon > 0$ we have:
    \begin{equation*}
        \lambda_p^\epsilon \geq \lambda_p \geq 0
    \end{equation*}
    Secondly, we use Lemma \ref{lemma:Psi-is-well-behaved} to obtain constants $m_f, m_g \geq 0$ such that $m_f \cdot m_g \neq 0$ satisfying:
    \begin{align*}
        \int_{\Omega} \left( \frac{1}{1 + x f(s)} \right)^{2p} \, ds &\geq m_f^{2p} \\
        \int_{\Omega} \left( \frac{1}{1 + x g(s)} \right)^{2p} \, ds &\geq m_g^{2p}
    \end{align*}
    We can now estimate:
    \begin{align*}
        \lambda_p^\epsilon(x,v) &\geq v^{2p} \, \min\left\{ m_f^{2p} \, , \, m_g^{2p} \right\} \, .
    \end{align*}
    Thus, we can take $\alpha_1 = \min\left\{ m_f^{2p} \, , \, m_g^{2p} \right\}$, $\alpha_2 = 0$ and $\alpha_3 = 0$ to conclude the proof.
\end{proof}
The next step is to establish that minimizers of \eqref{eq:expanded-opt-problem-epsilon-p} satisfy the strong Euler--Lagrange equations.

\begin{lemma}
    \label{lemma:minimizer-satisfies-strong-EL}
    Any minimizer of problem \eqref{eq:expanded-opt-problem-epsilon-p} satisfies the strong Euler--Lagrange equations.
\end{lemma}
\begin{proof}
    Here, we apply Theorem 4.12 in \citet{dacorogna2007calculusofvariations}, adapted to the case of a Lagrangian $\lambda_p^\epsilon$ with a bounded first variable (see the discussion in Lemma \ref{lemma:existence-of-minimizer-in-B}).
    To claim that a minimizer $\tau_\epsilon$ of the problem \eqref{eq:expanded-opt-problem-epsilon-p} satisfies the Euler--Lagrange equation it suffices to have that the minimizer is at least $C^2([0,1])$ as well as that the Lagrangian $\lambda_p^\epsilon$ is $C^2\left( [0,1] \times \R \right)$ and satisfies the following regularity conditions:
    there are functions $\alpha_1 \in L^1([0,1])$ and $\alpha_2 \in L^{2p/(2p-1)}([0,1])$ and a constant $\beta \in \R$ such that for every $(x, v) \in [0,1] \times \R$ we have:
    \begin{align*}
        \left| \lambda_p^\epsilon(x,v) \right| \, , \, \left| \partial_x \lambda_p^\epsilon(x,v) \right| & \leq \alpha_1(x) + \beta | v |^{2p} \\
        \left| \partial_v \lambda_p^\epsilon(x,v) \right| & \leq \alpha_2(x) + \beta | v |^{2p-1}
    \end{align*}
    With regards to the regularity of the minimizer we have:
    \begin{equation*}
        \frac{\partial^2 \lambda_p^\epsilon}{\partial v^2} (x,v) \geq \epsilon > 0 \textup{ for all } (x,v) \in [0,1] \times \R
    \end{equation*}
    which by Theorem 4.36 in \citet{dacorogna2007calculusofvariations} guarantees that $\tau_\epsilon \in C^\infty([0,1])$.
    
    Now to verufy the regularity of the we appeal to Lemma \ref{lemma:Psi-is-well-behaved} to obtain constants $M_f, M_g \geq 0$ such that:
    \begin{align*}
        \int_{\Omega} \left( \frac{f(s)}{1 + x f(s)} \right)^{2p} \, ds \leq M_f^{2p} \\
        \int_{\Omega} \left( \frac{g(s)}{1 + x g(s)} \right)^{2p} \, ds \leq M_g^{2p}
    \end{align*}
    Setting $C = \max\left\{ M_f^{2p} \, , \, M_g^{2p} \right\}$ and recalling that for each $q \leq p$ there is a constant $A_{q, p} > 0$ such that $|v|^q \leq A_{q, p} \, |v|^p $ for all $v \in \R$, we have:
    \begin{align*}
        \left| \lambda_p^\epsilon(x,v)               \right| &= v^{2p} \left( \int_{\Omega} \left( \frac{f(s)}{1 + x f(s)} \right)^{2p} \, ds + \int_{\Omega} \left( \frac{g(s)}{1 + x g(s)} \right)^{2p} \, ds \right) + \epsilon \, v^2  \\
                                                    & \leq  \left( 2^{2p} \, C^{2p} + \epsilon \, A_{2, 2p} \right) \, v^{2p} \\
        \left| \partial_x \, \lambda_p^\epsilon(x,v) \right| &= 2p \, v^{2p} \left( \int_{\Omega} \left( \frac{f(s)}{1 + x f(s)} \right)^{2p+1} \, ds + \int_{\Omega} \left( \frac{g(s)}{1 + x g(s)} \right)^{2p+1} \, ds \right) + \epsilon \, v^2 \\ 
                                                    & \leq \left( 2 \, C \right)^{2p+1} \, v^{2p} \\
        \left| \partial_v \, \lambda_p^\epsilon(x,v) \right|  &= 2p \, |v|^{2p-1} \left( \int_{\Omega} \left( \frac{f(s)}{1 + x f(s)} \right)^{2p} \, ds + \int_{\Omega} \left( \frac{g(s)}{1 + x g(s)} \right)^{2p} \, ds \right) + 2\epsilon \, v \\ 
                                                    & \leq \left( 2^{2p} \, C^{2p} + 2 \, \epsilon \, A_{1, 2p-1} \right) \, |v|^{2p-1}
    \end{align*}
    so taking:
    \begin{equation*}
        \beta = \max \left\{ 2^{2p} \, C^{2p} + \epsilon \, A_{2, 2p} \, , \, (2 \, C)^{2p + 1} \, , \, 2^{2p} \, C^{2p} + 2 \, \epsilon \, A_{1, 2p-1}  \right\}
    \end{equation*}
    we have that the conditions are satisfied.
    Finally, since the functions $f$ and $g$ are continuous on the compact set $\Omega$ and for each $s \in \Omega$ and $x\in [0,1]$ the maps $ x \mapsto \frac{f(s)}{1 + x f(s)}$ and $x \mapsto \frac{g(s)}{1 + x g(s)}$ are smooth, we can apply the Leibniz Rule to conclude that the Lagrangian $\lambda_p^\epsilon$ is smooth on $[0,1] \times \R$.
    This completes the proof.
\end{proof}
We can now write the strong Euler--Lagrange equations for the Lagrangian $\lambda_p^\epsilon$:
\begin{lemma}
    \label{lemma:strong-EL-for-Lp-epsilon}
    For each $p \in \N$ define:
    \begin{align*}
        & K_p: [0,1] \to \R \\
        & K_p(x) \mapsto \lambda_p(x, 1) = \int_{\Omega} \left( \frac{f(s)}{1 + x f(s)} \right)^{2p} \, ds + \int_{\Omega} \left( \frac{g(s)}{1 + x g(s)} \right)^{2p} \, ds 
    \end{align*}
    For each $\epsilon > 0$ the strong Euler--Lagrange equations for the Lagrangian $\lambda^\epsilon_p$
    \begin{equation}
        \label{eq:strong-EL-for-Lp-epsilon-general}
        \frac{d}{dt} \left[ \frac{\partial \lambda_p^\epsilon}{\partial v} \left(\tau(t), \dot{\tau}(t)\right) \right] - \frac{\partial \lambda_p^\epsilon}{\partial x}\left(\tau(t), \dot{\tau}_p(t)\right) = 0 \, .
    \end{equation}
    can be re-written as
    \begin{equation}
        \label{eq:strong-EL-for-Lp-epsilon}
        \ddot{\tau}(t) = \frac{ \dot{\tau}^{2p}(t) \, K_{p+1} (\tau(t)) }{ \dot{\tau}^{2p-2}(t) \, K_p(\tau(t)) + \frac{\epsilon}{p(2p-1)} } \, .
    \end{equation}
    Moreover, the strong Euler--Lagrange equations for $\lambda_p$ are obtained by setting $\epsilon = 0$ in \eqref{eq:strong-EL-for-Lp-epsilon}.
\end{lemma}
\begin{proof}
    For notational convenience, we supress the $p$ and $\epsilon$ notation on $\lambda_p^\epsilon$, for the rest of this proof.
    Now, suppoose $\tau$ satisfies the strong Euler--Lagrange equations \eqref{eq:strong-EL-for-Lp-epsilon-general} for $\lambda$.
    We have
    \begin{equation*}
        \frac{d}{dt} \left( \frac{\partial \lambda_p}{\partial v}\left(\tau_p(t), \dot{\tau}_p(t)\right) \right) - \frac{\partial \lambda_p}{\partial x} \left(\tau_p(t), \dot{\tau}_p(t)\right) = 0
    \end{equation*}
    with boundary conditions $\tau_p(0) = 0$ and $\tau_p(1) = 1$.
    This is equivalent to:
    \begin{align}
        \label{eq:EL-eq-generic}
        \ddot{\tau}_p(t) \, \frac{\partial^2 \lambda_p}{\partial v^2}\left(\tau_p(t), \dot{\tau}_p(t)\right) + \dot{\tau}_p(t) \, \frac{\partial^2 \lambda_p}{\partial v \, \partial x}\left(\tau_p(t), \dot{\tau}_p(t)\right) = \frac{\partial \lambda_p}{\partial x} \left(\tau_p(t), \dot{\tau}_p(t)\right) 
    \end{align}
    Now to make computations less cumbersome, notice the decomposition:
    \begin{align*}
        & \lambda(x,v) = \lambda_f(x,v) + \lambda_g(x,v) + \epsilon \, v^2  \\
        & \lambda_f(x,v) \coloneq \int_{\Omega} \left( \frac{v f(s)}{1 + x f(s)} \right)^{2p} \, ds \\
        & \lambda_g(x,v) \coloneq \int_{\Omega} \left( \frac{v g(s)}{1 + x g(s)} \right)^{2p} \, ds
    \end{align*}
    We conduct computations on $\lambda_f$ since those for $\lambda_g$ are identical.
    Now, once again, the continuity of $f:\Omega \to \R$ on the compact set $\Omega$ and the smoothness of the maps $x \mapsto \frac{f(s)}{1 + x f(s)}$ and $x \mapsto \frac{g(s)}{1 + x g(s)}$ for $x \in [0,1]$ allows us to differentiate under the integral sign to obtain
    \begin{align*}
        &\frac{\partial \lambda_f}{\partial v}(x,v) = \int_{\Omega} \frac{ 2p \, v^{2p-1} f(s)^{2p}  }{ (1 + x \, f(s) )^{2p} } \, ds \\
        &\frac{\partial \lambda_f}{\partial x}(x,v) = \int_{\Omega} \frac{ -2p \, v^{2p} f(s)^{2p+1} }{ (1 + x \, f(s) )^{2p+1} } \, ds \, .
    \end{align*}
    Differentiating once more we have
    \begin{align*}
        \frac{\partial^2 \lambda_f}{\partial v^2}(x,v) &= \int_{\Omega} \frac{ 2p(2p-1) \, v^{2p-2} f(s)^{2p}  }{ (1 + x \, f(s) )^{2p} } \, ds \\
        \frac{\partial^2 \lambda_f}{\partial x \, \partial v}(x,v) &= \int_{\Omega} \frac{ -(2p)^2 \, v^{2p-1} f(s)^{2p+1}  }{ (1 + x \, f(s) )^{2p+1} } \, ds \, ,
    \end{align*}
    thus, combining with the corresponding expressions for $\lambda_g$ we can re-write \eqref{eq:EL-eq-generic} as:
    \begin{align*}
        \ddot{\tau}(t) \, & \int_{\Omega} \frac{ {2p}({2p}-1) \, \dot{\tau}(t)^{{2p}-2} f(s)^{2p}  }{ (1 + \tau(t) f(s) )^{2p} } \, ds + \ddot{\tau}(t) \, \int_{\Omega} \frac{ {2p}({2p}-1) \, \dot{\tau}(t)^{{2p}-2} g(s)^{2p}  }{ (1 + \tau(t) g(s) )^{2p} } \, ds \, + 2 \, \epsilon \, \ddot{\tau}(t) \, + \\
        \dot{\tau}(t) \,  & \int_{\Omega} \frac{ -(2p)^2 \, \dot{\tau}(t)^{2p-1} f(s)^{2p+1}  }{ (1 + \tau(t) \, f(s) )^{2p+1} } \, ds + \dot{\tau}(t) \, \int_{\Omega} \frac{ -(2p)^2 \, \dot{\tau}(t)^{2p-1} g(s)^{2p+1}  }{ (1 + \tau(t) \, g(s) )^{2p+1} } \, ds \, = \\
                            & \int_{\Omega} \frac{ -{2p} \, \dot{\tau}(t)^{{2p}} f(s)^{2p+1} }{ (1 + \tau(t) f(s) )^{2p+1} } \, ds + \int_{\Omega} \frac{ -{2p} \, \dot{\tau}(t)^{{2p}} g(s)^{2p+1} }{ (1 + \tau(t) g(s) )^{2p+1} } \, ds \, .
    \end{align*}
    Dividing\footnote{Note that if $\dot{\tau}(t) = 0$ for some $t$ then the equation is trivially satisfied.} both sides by $2p \, \dot{\tau}(t)^{2p}$ and re-arranging we obtain:
    \begin{align*}
        & \frac{\ddot{\tau}_p(t)}{ \dot{\tau}(t)^2} \left( \int_{\Omega} \frac{ ({2p}-1) f(s)^{2p}  }{ (1 + \tau(t) f(s) )^{2p} } \, ds + \int_{\Omega} \frac{ ({2p}-1) g(s)^{2p}  }{ (1 + \tau(t) g(s) )^{2p} } \, ds \right) \, +  \, \frac{\epsilon \, \ddot{\tau}(t)}{p \, \dot{\tau}(t)^{2p}} = \\
        & \int_{\Omega} \frac{ 2p \, f(s)^{2p+1}  }{ (1 + \tau(t) \, f(s) )^{2p+1} } \, ds + \int_{\Omega} \frac{ 2p \, g(s)^{2p+1}  }{ (1 + \tau(t) \, g(s) )^{2p+1} } \, ds \, + \\
    - \,& \int_{\Omega} \frac{f(s)^{2p+1}  }{ (1 + \tau(t) \, f(s) )^{2p+1} } \, ds - \int_{\Omega} \frac{ g(s)^{2p+1}  }{ (1 + \tau(t) \, g(s) )^{2p+1} } \, ds 
    \end{align*}
    Combining the last four terms and dividing both sides by $2p-1$ times the factor multiplying $\frac{\ddot{\tau}_p(t)}{\dot{\tau}_p(t)^2}$ and recalling the definition of $K_p$ we obtain:
    \begin{align*}
        \frac{\ddot{\tau}_p(t)}{ \dot{\tau}_p(t)^2} + \frac{\epsilon}{p(2p-1)} \frac{\ddot{\tau}_\epsilon(t)}{ \dot{\tau}^{2p}(t) \, K_p(\tau(t)) } = \frac{ K_{p+1}(\tau(t)) }{ K_p(\tau(t)) }
    \end{align*}
    Notice that $K_p$ is never zero since by Lemma \ref{lemma:Psi-is-well-behaved} we can lower bound it by $m_f^{2p} + m_g^{2p} > 0$.
    We can now re-arange the equation to obtain \eqref{eq:strong-EL-for-Lp-epsilon}:
    \begin{equation}
        \label{eq:ODE-for-tau-epsilon}
        \ddot{\tau}(t) = \frac{ \dot{\tau}^{2p}(t) \, K_{p+1} (\tau(t)) }{ \dot{\tau}^{2p-2}(t) \, K_p(\tau(t)) + \frac{\epsilon}{p(2p-1)} } \\ \, .
    \end{equation}
    Finally, notice that for $\epsilon = 0$ the Lagrangian $\lambda_p$ can be decomposed as:
    \begin{equation*}
        \lambda_p(x,v) = \lambda_f(x,v) + \lambda_g(x,v)
    \end{equation*}
    and following the same steps yields \eqref{eq:strong-EL-for-Lp-epsilon} with $\epsilon = 0$, that is,
    \begin{equation}
        \label{eq:second-order-ODE-for-tau}
        \ddot{\tau}(t) = \frac{ \dot{\tau}^{2}(t) \, K_{p+1} (\tau(t)) }{ K_p(\tau(t)) } \, .
    \end{equation}
    This completes the proof.
\end{proof}
Now problem \eqref{eq:expanded-opt-problem-epsilon-p} is related to the original problem \eqref{eq:opt-problem-in-Tau-2p-b} in a natural way.
To see this, start by defining:
\begin{definition}
    For $p \in \N$ and $\epsilon > 0$ we define the family of functionals
    \begin{align*}
        & \mathcal{L}^\epsilon_p: \mathcal{T}_{2} \to \R \\
        & \mathcal{L}^\epsilon_p=\begin{cases}
            \int_{0}^{1} \lambda_p^\epsilon \left( \tau(t), \dot{\tau}(t) \right) \, dt & \text{if } \tau \in \mathcal{T}_{2p}^{b} \\
            +\infty & \text{if } \tau \in \mathcal{T}_{2} \setminus \mathcal{T}_{2p}^{b} 
        \end{cases} 
    \end{align*}
    and for $\epsilon = 0$ we let
    \begin{align*}
        & \mathcal{L}^0_p: \mathcal{T}_{2} \to \R \\
        & \mathcal{L}^0_p=\begin{cases}
            \int_{0}^{1} \lambda_p \left( \tau(t), \dot{\tau}(t) \right) \, dt & \text{if } \tau \in \mathcal{T}_{2p}^{b} \\
            +\infty & \text{if } \tau \in \mathcal{T}_{2} \setminus \mathcal{T}_{2p}^{b}  \, .
        \end{cases}
    \end{align*}
\end{definition}
\begin{remark}
    Notice that by Lemma \ref{lemma:problem-can-be-restricted} we have that
    \begin{equation*}
        \inf_{\tau \in \mathcal{T}_2} \mathcal{L}^0_p[\tau] = \inf_{\tau \in \mathcal{T}_{2}} \int \lambda_p(\tau , \dot{\tau} ) \, dt \, ,
    \end{equation*}
    that is, we can solve problem \eqref{eq:opt-problem-Lp-approx} by solving problem \eqref{eq:opt-problem-in-Tau-2p-b} which is captured by the functional $\mathcal{L}^0_p$.
\end{remark}
\begin{lemma}
    \label{lemma:L^epsilon-p-Gamma-converge-and-are-equicoercive}
    For each $p \in \N$, the family of functionals $ \left( \mathcal{L}^\epsilon_p \right)_{\epsilon > 0}$ is equicoercive and, moreover:
    \begin{equation*}
        \mathcal{L}^\epsilon_p \to \mathcal{L}^0_p \quad \text{as } \epsilon \searrow 0  \quad \text{in the $\Gamma$-sense,}
    \end{equation*}
    where the $\Gamma$-convergence is understood over the space $\mathcal{T}_2$ equipped with the weak $H^1$ topology.
\end{lemma}
\begin{proof}
    Notice that we have the point-wise, monotone convergence
    \begin{equation*}
        \mathcal{L}_p^\epsilon[\tau] \searrow \mathcal{L}_p^0 [\tau] \textup{ as } \epsilon \searrow 0
    \end{equation*}
    for all $\tau \in \mathcal{T}_{2}$, thus for any $\tau \in \mathcal{T}_{2}$ we have
    \begin{equation*}
        \limsup_{\epsilon \to 0} \mathcal{L}_p^\epsilon[\tau] = \mathcal{L}_p^0[\tau] \, .
    \end{equation*}
    Now fix a sequence $\gamma_n \rightharpoonup \gamma_0$ in $H^1([0,1])$.
    Again due to the monotonicity of $\mathcal{L}_p^\epsilon$ we have that for each $n \in \N$:
    \begin{equation*}
        \mathcal{L}_p^{ {1}/{n}}[\gamma_n] \geq \mathcal{L}_p^{0}[\gamma_n]
    \end{equation*}
    Taking the limit inferior as $n \to \infty$ we obtain:
    \begin{equation*}
        \liminf_{n \to \infty} \mathcal{L}_p^{ {1}/{n}}[\gamma_n] \geq \liminf_{n \to \infty} \mathcal{L}_p^{0}[\gamma_n] \geq \mathcal{L}_p^{0}[\gamma_0]
    \end{equation*}
    And the last step follows by Lemma \ref{lemma:lower-semi-cont-of-functionals}: there, it is proven that $\mathcal{L}_p^0$ is weakly lower semi-continuous in the $H^1$ topology, albeit in the different notation $\Lambda_p \equiv \mathcal{L}_p^0$.
    By the sequential characterizarion of $\Gamma$ limits in metric spaces---see \citet[Theorem 2.1]{braides2006handbook} part (c)---we have established that the $\mathcal{L}_p^\epsilon$ converge to $\mathcal{L}_p^0$ in the $\Gamma$-sense.
    
    To see equicoercivity, we use the monotonicity of $\mathcal{L}_p^\epsilon $ together with Lemma \ref{lemma:monotonicty-Lp} to obtain:
    \begin{equation*}
        \mathcal{L}_p^\epsilon \geq \mathcal{L}_p^0 \gtrsim \mathcal{L}_1^0 \, ,
    \end{equation*}
    with the hidden constant only depending on $\Omega$, and $\mathcal{L}_1^0 \equiv \Lambda_1$ is lower-semi continous and coercive, as seen in lemmas \ref{lemma:lower-semi-cont-of-functionals} and \ref{prop:equi-coercive-Lp}. We can now conlude by applying \citet[Proposition 7.7]{dalmaso_gammaconvergence}
\end{proof}
Before showing that there is a minimizer of \eqref{eq:opt-problem-in-Tau-2p-b} that satisfies the strong Euler--Lagrange equations, we need one more technical lemma:
\begin{lemma}
    \label{lemma:velocity-can-be-put-into-phase-space-box}
    Fix a $p \in \N$ and let $\tau_\epsilon$ be a minimizer of problem \eqref{eq:expanded-opt-problem-epsilon-p}, guaranteed to exist by Lemma \ref{lemma:existence-of-minimizer-in-B}.
    There is a constant $M > 0$ such that for all $\epsilon > 0$ we have:
    \begin{equation*}
        \| \dot{\tau}_\epsilon \|_{L^\infty([0,1])} \leq M \,.
    \end{equation*}
\end{lemma}
\begin{proof}
    To prove this claim, we start by bounding the second derivative of $\tau_\epsilon$ by constants independent of $\epsilon$.
    Using the notation of Lemma \ref{lemma:strong-EL-for-Lp-epsilon} we have
    \begin{equation}
        \label{eq:bound-1-on-acceleration}
        \left| \ddot{\tau}_\epsilon(t) \right| \leq \left| \frac{\dot{\tau}_\epsilon^{2p}(t) K_{p+1}(\tau_\epsilon(t)) }{ \dot{\tau}_\epsilon^{2p-2}(t) K_p(\tau_\epsilon(t)) } \right| \leq \dot{\tau}^2_\epsilon(t) \, \frac{ \sup_{x \in [0,1] } | K_{p+1}(x) | }{\inf_{x \in [0,1] } | K_{p}(x) |}
    \end{equation}
    using that for all $\epsilon$ the map $t \mapsto \tau_\epsilon(t)$ is a bijection on $[0,1]$.
    Now we apply Lemma \ref{lemma:Psi-is-well-behaved} to obtain constants $M_f, M_g, m_f, m_g \geq 0$ depending on $f,g$ and $\Omega$, with $m_f \cdot m_g \neq 0$, such that for each $p \in \N$ we have
    \begin{equation*}
        m_f^{2p} + m_g^{2p} \leq K_p(x) \leq  M_f^{2p} +  M_g^{2p} \quad \text{for all } x \in [0,1]
    \end{equation*}
    Thus, equation \eqref{eq:bound-1-on-acceleration} becomes:
    \begin{equation}
        \label{eq:bound-2-on-acceleration}
        \left| \ddot{\tau}_\epsilon(t) \right| \leq D \, \dot{\tau}^2_\epsilon(t)
    \end{equation}
    for a constant $D = D(f,g,\Omega, p, p+1) > 0$ depending only on $f,g,\Omega, p$ and $p+1$.

    Now, for a fixed $p \in \N$, recall that the family $( \tau_\epsilon )_\epsilon$ is a family of minimizers for the problems:
    \begin{equation*}
        \inf_{\tau \in \mathcal{T}_{2p}^{b}} \mathcal{L}_p^\epsilon[\tau]
    \end{equation*}
    By\footnote{In effect, this is a weaker version of the fundamental theorem of $\Gamma$-convergence.} the $\Gamma$-convergence of $\mathcal{L}_p^\epsilon$ to $\mathcal{L}_p^0$ as well as the equi-coercivity of $(\mathcal{L}_p^\epsilon)_{\epsilon > 0}$ we have \citep[Theorem 7.8]{dalmaso_gammaconvergence} that
    \begin{equation*}
        \lim _{\epsilon \to 0} \mathcal{L}_p^\epsilon[\tau_\epsilon] = \min_{\tau \in \mathcal{T}_2} \mathcal{L}_p^0[\tau] < \infty \, .
    \end{equation*}
    Thus, there is a constant $B > 0$ such that $\mathcal{L}_p^\epsilon[\tau_\epsilon] \leq B$ for all $\epsilon > 0$.
    Moreover, looking into the proof of Lemma \ref{lemma:L^epsilon-p-Gamma-converge-and-are-equicoercive} we see that:
    \begin{equation*}
       \mathcal{L}_1^0 \leq  C_{\Omega} \, \mathcal{L}_p^\epsilon
    \end{equation*}
    and the constant $C_{\Omega} > 0$ only depends on $\Omega$. Moreover, looking at equation \eqref{eq:H1-norm-coercivity} of Proposition \ref{prop:equi-coercive-Lp} and noting that $\mathcal{L}_1^0 \equiv \Lambda_1$  we find another constant $C_{f, g, [0,1]} > 0$, depending only on $f$, $g$ and the domain $[0,1]$ such that
    \begin{equation*}
       \| \tau \|_{H^1([0,1])}^2 \leq  C_{f, g, [0,1]} \, \mathcal{L}_1^0[\tau]
    \end{equation*}
    obtaining
    \begin{equation}
        \label{eq:bound-H1-in-terms-of-mathcalL}
        \| \tau \|_{H^1([0,1])}^2 \leq C_{\Omega, f, g, [0,1]} \, \mathcal{L}_p^\epsilon[\tau] 
    \end{equation}
    for all $\tau \in \mathcal{T}_{2}$.
    Combining with the above we get
    \begin{equation}
        \label{eq:bound-on-velocity}
        \| \dot{\tau}_\epsilon \|_{L^2([0,1])}^2 \leq  \| \tau_\epsilon \|_{H^1([0,1])}^2 \leq B \, C_{\Omega, f, g, [0,1]} \eqcolon C < \infty \, ,
    \end{equation}
    that is, the family $(\dot{\tau}_\epsilon)_\epsilon$ is bounded in $L^2([0,1])$ uniformly in $\epsilon$ and $p$.
    Now by the fundamental theorem of calculus we have:
    \begin{equation}
        \label{eq:fundamental-theorem-of-calculus}
        \dot{\tau}_\epsilon(t) = \dot{\tau}_\epsilon(0) + \int_{0}^{t} \ddot{\tau}_\epsilon(s) \, ds
    \end{equation}
    Thus, combining \eqref{eq:fundamental-theorem-of-calculus} with \eqref{eq:bound-2-on-acceleration} and \eqref{eq:bound-on-velocity} we have
    \begin{equation*}
        |\dot{\tau}_\epsilon(t) | \leq \int_{0}^{t} |\ddot{\tau}_\epsilon(s) | \, ds + | \dot{\tau}_\epsilon(0) | \leq D \int_{0}^{1} | \dot{\tau}_\epsilon(s) |^2 \, ds + | \dot{\tau}_\epsilon(0) | \leq  C \, D + | \dot{\tau}_\epsilon(0) |
    \end{equation*}
    for all $\epsilon > 0$.
    If we can show that $| \dot{\tau}_\epsilon(0) |$ is uniformly bounded in $\epsilon$ we are done.
    To do this, we integrate \eqref{eq:fundamental-theorem-of-calculus} once more and using the boundary conditions $\tau_\epsilon(0) = 0$ and $\tau_\epsilon(1) = 1$ we obtain:
    \begin{align*}
        1 = \dot{\tau}_\epsilon(0) + \int_{0}^{1} \int_{0}^{t} \ddot{\tau}_\epsilon(s) \, ds \, dt 
    \end{align*}
    Applying \eqref{eq:bound-2-on-acceleration} and \eqref{eq:bound-on-velocity} one again we obtain
    \begin{equation*}
        | \dot{\tau}_\epsilon(0) | \leq 1 + \int_{0}^{1} \int_{0}^{t} | \ddot{\tau}_\epsilon(s) | \, ds \, dt \leq 1 + \int_{0}^{1} \int_{0}^{1} D |\dot{\tau}_\epsilon(s) |^2 \, ds \, dt \leq 1 + C \, D
    \end{equation*}
    for all $\epsilon > 0$.
    This completes the proof.
\end{proof}
We can now show:
\begin{proposition}
    \label{prop:strong-EL-for-original-problem}
    Problem \eqref{eq:main-lagrangian-problem-abbreviated-in-appendix} has a minimizer $\tau_p$ that satisfies the strong Euler--Lagrange equations:
    \begin{equation*}
        \frac{d}{dt} \left[ \frac{\partial \lambda_p}{\partial v} \left(\tau_p(t), \dot{\tau}_p(t)\right) \right] - \frac{\partial \lambda_p}{\partial x}\left(\tau_p(t), \dot{\tau}_p(t)\right) = 0
    \end{equation*}
\end{proposition}
\begin{proof}
    By Lemma \ref{lemma:strong-EL-for-Lp-epsilon} the strong Euler--Lagrange equations for the problem \eqref{eq:opt-problem-in-Tau-2p-b} are given by the ODE
    \begin{equation}
        \label{eq:ODE-for-tau-limiting-penultimate-lemma}
        \ddot{\tau}(t) = \dot{\tau}^{2}(t) \frac{ K_{p+1}(\tau(t)) }{ K_p(\tau(t)) }
    \end{equation}
    with $K_p$ is defined by the relation
    \begin{equation*}
        K_p(x) = \int_{\Omega} \left( \frac{f(s)}{1 + x f(s)} \right)^{2p} \, ds + \int_{\Omega} \left( \frac{g(s)}{1 + x g(s)} \right)^{2p} \, ds \, .
    \end{equation*}
    Now by Lemma \ref{lemma:existence-of-minimizer-in-B} each optimization problen \eqref{eq:expanded-opt-problem-epsilon-p} has a minimizer $\tau_\epsilon$.
    By Lemma \ref{lemma:minimizer-satisfies-strong-EL} the $\tau_\epsilon$ satisfy the strong Euler--Lagrange equations for the Lagrangian $\lambda_p^\epsilon$ that, by Lemma \ref{lemma:strong-EL-for-Lp-epsilon},
    can be written as:
    \begin{equation}
        \label{eq:ODE-for-tau-epsilon-penultimate-lemma}
        \ddot{\tau}_\epsilon(t) = \frac{ \dot{\tau}_\epsilon^{2p}(t) \, K_{p+1} (\tau_\epsilon(t)) }{ \dot{\tau}_\epsilon^{2p-2}(t) \, K_p(\tau_\epsilon(t)) + \frac{\epsilon}{p(2p-1)} } \\ \, .
    \end{equation}
    Let $\tau_0$ be a solution to the ODE \eqref{eq:ODE-for-tau-limiting-penultimate-lemma}. We wish to show that it will also be a minimizer of the problem \eqref{eq:main-lagrangian-problem-abbreviated-in-appendix}.
    To do this, we will establish that for a subseqence $\tau_{\epsilon_i} \to \tau_0$ in $L^2([0,1])$ and, by the $\Gamma$-convergence of the family $\mathcal{L}_p^\epsilon$ to $\mathcal{L}_p^0$, $\tau_0$ will be a minimizer of the problem \eqref{eq:main-lagrangian-problem-abbreviated-in-appendix}.
    More explicitly, recall that the $\tau_\epsilon$ are, by definition, solutions of the following problems
    \begin{equation*}
        \tau_\epsilon \in \arg\min_{\tau \in \mathcal{T}_{2}^{b}} \mathcal{L}^\epsilon_p[\tau]
    \end{equation*}
    and, moreover, problem \eqref{eq:main-lagrangian-problem-abbreviated-in-appendix} is equivalent to
    \begin{equation*}
        \inf_{\tau \in \mathcal{T}_{2}^{b}} \mathcal{L}^0_p[\tau] \, .
    \end{equation*}
    In Lemma \ref{lemma:L^epsilon-p-Gamma-converge-and-are-equicoercive} we showed that the family $(\mathcal{L}^\epsilon_p)_{\epsilon > 0}$ is equicoercive and $\Gamma$-converges to $\mathcal{L}^0_p$.
    Thus, by the \textit{fundamental theorem of $\Gamma$-convergence} \citep[Theorem 2.10]{braides2006handbook} we have that cluster points of pre-compact sequences of minimizers of $\mathcal{L}_p^\epsilon$ are minimizers of $\mathcal{L}_p^0$.
    By \citet[Theorem 7.8]{dalmaso_gammaconvergence} we also have
    \begin{equation*}
        \lim_{\epsilon \to 0} \mathcal{L}_p^\epsilon[\tau_\epsilon] = \inf_{\tau \in \mathcal{T}_2} \mathcal{L}_p^0[\tau] \eqcolon C < \infty
    \end{equation*}
    and using equation \eqref{eq:bound-H1-in-terms-of-mathcalL} from Lemma \ref{lemma:velocity-can-be-put-into-phase-space-box} we have
    \begin{equation*}
        \| \tau_\epsilon \|_{H^1([0,1])}^2 \leq C_{\Omega, f, g, [0,1]} \, \mathcal{L}_p^\epsilon[\tau_\epsilon] \leq C_{\Omega, f, g, [0,1]} \, C
    \end{equation*}
    for some constant $C_{f, g, \Omega, [0,1]} > 0$ depending only on $f,g, \Omega$ and the domain $[0,1]$. Thus, we have shown that the family $(\tau_\epsilon)_\epsilon$ is uniformly bounded in $H^1([0,1])$ and so it is pre-compact in the weak-$H^1$ topology.
    Thus, a weak-$H^1$ cluster point of the $\tau_\epsilon$ will be a minimizer of \eqref{eq:main-lagrangian-problem-abbreviated-in-appendix}.
    However, the Rellich--Kondrachov theorem \citep[Theorem 9.16]{brezis2011functional}, the $H^1$ boundedness of the $\tau_\epsilon$, and \citet[Example 8.9]{dalmaso_gammaconvergence} show that weak-$H^1$ cluster points of the $\tau_\epsilon$ are strong $L^2$ cluster points and vice-versa.
    Therefore, to conclude, it suffices to identify $\tau_0$ as a strong $L^2$ cluster point of the $\tau_\epsilon$.
    
    As done above, the key is to apply a Gr\"onwall estimate together with the Arzel\`a-Ascoli Theorem.
    First, for $\epsilon \geq 0$ we write the dynamics as:
    \begin{align*}
        \frac{d}{dt}
        \begin{pmatrix}
            \dot{\tau}_\epsilon(t) \\
            \tau_\epsilon(t)
        \end{pmatrix}
        =
        \begin{pmatrix}
            f_\epsilon(\tau_\epsilon(t), \dot{\tau}_\epsilon(t)) \\
            \tau_\epsilon(t)    
        \end{pmatrix}
    \end{align*}
    where:
    \begin{align*}
        f_\epsilon(x,v) &\coloneq \frac{v^{2p} K_{p+1}(x)}{v^{2p-2} K_p(x) + \frac{\epsilon}{p(2p-1)}} \\
    \end{align*}
    Now, by Lemma \ref{lemma:velocity-can-be-put-into-phase-space-box} there is an $M > 0$ such that for all $\epsilon >0$ we have $\| \dot{\tau}_\epsilon \|_{L^\infty([0,1])} < M$ and so we can restrict the above 
    dynamics to the phase space box $[0,1] \times [0,M]$. Owing to the smoothness of $f_\epsilon$ we now have uniform estimates for any derivative on $[0,1] \times [0,M]$ showing, in particular, that the family
    $(f_\epsilon)_{\epsilon > 0}$ is uniformly bounded and uniformly equi-continuous. \footnote{
        For example, equi-continuity follows by differentiating first:
        \begin{align*}
            \partial_x f_\epsilon(x,v) &= \frac{v^{2p} K_{p+1}'(x)}{v^{2p-2} K_p(x) + \frac{\epsilon}{p(2p-1)}} - \frac{v^{2p} K_{p+1}(x) \, v^{2p-2} K_p'(x)}{(v^{2p-2} K_p(x) + \frac{\epsilon}{p(2p-1)})^2} \\
            \partial_v f_\epsilon(x,v) &= \frac{2p v^{2p-1} K_{p+1}(x)}{v^{2p-2} K_p(x) + \frac{\epsilon}{p(2p-1)}} + \frac{v^{2p} K_{p+1}(x) \, (2p-2) v^{2p-3} K_p(x)}{(v^{2p-2} K_p(x) + \frac{\epsilon}{p(2p-1)})^2}
        \end{align*}
        and now noticing that there is a constant $B > 0$ such that for all $(x,v) \in [0,1] \times [0,M]$ and all $\epsilon > 0$ we have:
        \begin{align*}
            \left| \partial_x f_\epsilon(x,v) \right|  \, , \, \left| \partial_v f_\epsilon(x,v) \right| &\leq B
        \end{align*}
        Thus, for all $(x_1,v_1), (x_2,v_2) \in [0,1] \times [0,M]$ and all $\epsilon > 0$, by the mean value theorem and a Cauchy--Schwarz estimate we have:
        \begin{align*}
            \left\| f_\epsilon(x_1,v_1) - f_\epsilon(x_2,v_2) \right\| &\leq \sqrt{2} \, B \left\| (x_1,v_1) - (x_2,v_2) \right\| \\
        \end{align*}
     on $[0,1] \times [0,M]$ as well as that $f_0$ is Lipschitz continuous.}
    Thus, applying the Arzel\`a-Ascoli theorem we obtain that there is a sequence $\epsilon_i \to 0$ such that $f_{\epsilon_i} \to f_0$ uniformly on $[0,1] \times [0,M]$
    and then applying a Gr\"onwall estimate (see Theorem 2.1 in \citet{howard1998gronwall}) we obtain:
    \begin{equation*}
        \| \tau_{\epsilon_i} - \tau_0 \|_{L^\infty([0,1])} \to 0
    \end{equation*}
    In particular, this implies that $\tau_{\epsilon_i} \to \tau_0$ in $L^2([0,1])$ and we can conclude.
\end{proof}
\begin{theorem*}
    \label{thm:solution-to-Lp-ODE}
    Problem \eqref{eq:opt-problem-Lp-approx} has a solution $\tau_p$ which satisfies
    \begin{align}
        \label{eq:ODE-solution-of-Lp-problem}
        \frac{d}{dt} \tau_p(t) = \frac{1}{Z_p} \Bigg( 
        \int_{\Omega} 
        & \frac{f(s)^{2p}}{(1 + \tau_p(t) f(s))^{2p}} \, + \frac{g(s)^{2p}}{(1 + \tau_p(t) g(s))^{2p}} \, ds 
        \Bigg)^{-\frac{1}{2p}} \textup{ for all } t \in (0,1)
    \end{align}    
    together with the boundary conditions $\tau_p(0) = 0$ and $\tau_p(1) = 1$, and where
    \begin{equation*}
        Z_p = \int_{0}^{1} \left( \int_{\Omega} \left( \frac{f(s)}{1 + \tau_p(t) f(s)} \right)^{2p} \, ds + \int_{\Omega} \left( \frac{g(s)}{1 + \tau_p(t) g(s)} \right)^{2p} \, ds \right)^{-\frac{1}{2p}} \, dt.
    \end{equation*}
\end{theorem*}
\begin{proof}
    By Proposition \ref{prop:strong-EL-for-original-problem} we have that the problem \eqref{eq:main-lagrangian-problem-abbreviated-in-appendix} has a minimizer $\tau_p$ that satisfies the strong Euler--Lagrange equations which, by Lemma \ref{lemma:strong-EL-for-Lp-epsilon} can be written as:
    \begin{equation*}
        \ddot{\tau}(t) = \dot{\tau}^{2}(t) \frac{ K_{p+1}(\tau(t)) }{ K_p(\tau(t)) }
    \end{equation*}
    with 
    \begin{equation*}
        K_p(x) = \int_{\Omega} \left( \frac{f(s)}{1 + x f(s)} \right)^{2p} \, ds + \int_{\Omega} \left( \frac{g(s)}{1 + x g(s)} \right)^{2p} \, ds \, .
    \end{equation*}
    Notice that this can be re-written as
    \begin{equation*}
        \frac{\ddot{\tau}(t)}{ \dot{\tau}(t)} = \dot{\tau}(t) \left(-\frac{1}{2p}\right) \frac{ \frac{d}{dx} K_{p}(\tau(t)) }{ K_p(\tau(t)) } \, ,
    \end{equation*}
    or, equivalently:
    \begin{align*}
        \frac{d}{dt} \ln \dot{\tau}_p(t)  &= \frac{d}{dt} \left[ -\frac{1}{2p} \ln K_p(\tau_p(t)) \right] \iff 
        \dot{\tau}_p(t) = \frac{1}{Z_p} K_p(\tau_p(t))^{-\frac{1}{2p}} \\
    \end{align*}
    where $Z_p \in \R \setminus \{0\}$ is a constant. Writing $K_p$ out explicitly we obtain:
    \begin{align*}
        \dot{\tau}_p(t) &= \frac{1}{Z_p} \left( \int_{\Omega} \left( \frac{f(s)}{1 + \tau_p(t) f(s)} \right)^{2p} \, ds + \int_{\Omega} \left( \frac{g(s)}{1 + \tau_p(t) g(s)} \right)^{2p} \, ds \right)^{-\frac{1}{2p}}
    \end{align*}
    Finally, we integrate both sides from $t=0$ to $t=1$ and use the boundary conditions $\tau_p(0) = 0$ and $\tau_p(1) = 1$ to obtain:
    \begin{align*}
        1 &= \frac{1}{Z_p} \int_{0}^1 \left( \int_{\Omega} \left( \frac{f(s)}{1 + \tau_p(t) f(s)} \right)^{2p} \, ds + \int_{\Omega} \left( \frac{g(s)}{1 + \tau_p(t) g(s)} \right)^{2p} \, ds \right)^{-\frac{1}{2p}} \, dt \, .
    \end{align*}
    Solving for $Z_p$ completes the proof.
\end{proof}

\subsection{Proof of Theorem \ref{thm:solution-of-optimal-problem-is-solution-of-L-inf-ODE-main-text}}
\label{subsec:appendix-subseq-convergence-in-L2}
In this subsection, we establish the subsequential $L^2$ convergence of the $\tau_p$ to $\tau_\infty$.
Throughout the rest of this section we denote for all $p\in\mathbb{N}$ and all $x\in [0,1]$
    \begin{align*}
        F_p(x) & := \frac{1}{Z_p} \left[ \int_{\Omega} \left( \frac{f(s)}{1 + x f(s)} \right)^{2p} \, ds + \int_{\Omega} \left( \frac{g(s)}{1 + x f(s)} \right)^{2p} \, ds \right]^{-\frac{1}{2p}}, \\
        F_\infty(x) & := \max \left\{ \, \sup_{s \in \Omega} \, \left| \frac{f(s)}{1 + x f(s)} \right| \, , \, \sup_{s \in \Omega} \, \left| \frac{g(s)}{1 + x g(s)} \right| \right\}^{-1}
    \end{align*}
    and additionally
    \begin{equation*}
      f^* = \sup_{s \in \Omega} f(s) = \sup_{s \in \Omega} \sigma_{\max}(s) - 1,\qquad
      g_* = \inf_{s \in \Omega} g(s) = \inf_{s \in \Omega} \sigma_{\min}(s) - 1.
    \end{equation*}
    Lastly, we let $\tau_p$ and $\tau_\infty$ denote the solutions to the following ODEs:
    \begin{align}
        \label{eq:ODE-for-tau-p}
        \dot{\tau}_p(t) &= F_p(\tau_p(t)) \textup{ for all } t \in (0,1) \\
        \label{eq:ODE-for-tau-infty}
        \dot{\tau}_\infty(t) &= F_\infty(\tau_\infty(t)) \textup{ for all } t \in (0,1) \, .
    \end{align}
    Notice, now, that for any $x\in [0,1]$ the map $y\mapsto \frac{y}{1+xy}$ is monotonically increasing for $y\in [-1,\infty)$, by \eqref{eq:g0f}
    \begin{equation*}
      F_\infty(x) = \max \left\{\frac{f^*}{1+xf^*},\frac{-g_*}{1+xg_*} \right\}^{-1}
      =\min \left\{ \frac{1}{f^*}+x,-\frac{1}{g_*}-x \right\}.
  \end{equation*}
  As a minimum of two linear function with slope one, this function has a Lipschitz constant of $1$.
  We now prove two lemmas that will be useful in this section.
  \begin{lemma}
    \label{lemma:equi-boundedness}
    The sequence $F_p: [0,1] \to \R$ is equi-bounded in $p$.
\end{lemma}
\begin{proof}
    We use Lemma \ref{lemma:Psi-is-well-behaved} to obtain constants $M_f, M_g, m_f, m_g \geq 0$ such that $m_f \cdot m_g \neq 0$ satisfying:
    \begin{align*}
        &m_f^{2p} \leq \int_{\Omega} \left( \frac{1}{1 + x f(s)} \right)^{2p} \, ds \leq M_f^{2p} \\
        &m_f^{2p} \leq \int_{\Omega} \left( \frac{1}{1 + x g(s)} \right)^{2p} \, ds \leq M_g^{2p}
    \end{align*}
    This yields the estimate:
    \begin{equation*}
        F_p \leq \frac{1}{Z_p} \left( \frac{1}{ m_f^{2p} + m_g^{2p} } \right)^{\frac{1}{2p}}
    \end{equation*}
    Now to bound $\frac{1}{Z_p}$ we use the implicit expression given to us by Theorem \ref{thm:solution-to-Lp-ODE}:
    \begin{equation*}
        Z_p = \int_{0}^{1} \left( \int_{\Omega} \left( \frac{f(s)}{1 + \tau_p(t) f(s)} \right)^{2p} \, ds + \int_{\Omega} \left( \frac{g(s)}{1 + \tau_p(t) g(s)} \right)^{2p} \, ds \right)^{-\frac{1}{2p}} \, dt
    \end{equation*}
    where $\tau_p$ solves ODE \eqref{eq:ODE-solution-of-Lp-problem} with boundary conditions $\tau_p(0) = 0$ and $\tau_p(1) = 1$.
    Combining with our above estimate we obtain:
    \begin{equation}
        \label{eq:bound-on-1/Zp}
        \frac{1}{Z_p} \leq \left( M_f^{2p} + M_g^{2p} \right)^{\frac{1}{2p}}
    \end{equation}
    Thus, we conclude that:
    \begin{equation*}
        0 \leq \liminf_{p \to \infty} F_p(x) \leq \limsup_{p \to \infty} F_p(x) \leq \frac{ \max\{  M_f \, , \, M_g \} }{ \max\{ m_f \, , \, m_g \} } \quad \text{for all } x \in [0,1]
    \end{equation*}    
    This proves the lemma.
\end{proof}

\begin{lemma}
    \label{lemma:uniform-equi-continuity}
    The sequence $F_p: [0,1] \to \R$ is uniformly equi-continuous.
\end{lemma}
\begin{proof}
    Notice first that for each $p$, $F_p$ is $C^1([0,1])$. To see this note that since $f, g: \Omega \to \R$ are continuous, $\Omega$ is compact and the maps $x \mapsto \frac{f(s)}{1 + x f(s)}$ and $x \mapsto \frac{g(s)}{1 + x g(s)}$ are smooth in $[0,1]$,
    we can exchange integration over $s$ with differentiation over $x$, to obtain:
    \begin{align}
        \label{eq:derivative-of-Fp-1}
        \frac{d}{dx} F_p(x) &= \left(-\frac{1}{2p \, Z_p}\right) \frac{ \int_{\Omega} \frac{d}{dx} \left( \frac{f(s)}{1 + x f(s)} \right)^{2p} \, ds + \int_{\Omega} \frac{d}{dx} \left( \frac{g(s)}{1 + x g(s)} \right)^{2p} \, ds }{ \left( \int_{\Omega} \left( \frac{f(s)}{1 + x f(s)} \right)^{2p} \, ds + \int_{\Omega} \left( \frac{g(s)}{1 + x g(s)} \right)^{2p} \, ds \right)^{ \frac{1}{2p} + 1} } \\
        \label{eq:derivative-of-Fp-2}
        &= \left(\frac{1}{Z_p}\right) \frac{ \int_{\Omega}  \left( \frac{f(s)}{1 + x f(s)} \right)^{2p+1} \, ds + \int_{\Omega}  \left( \frac{g(s)}{1 + x g(s)} \right)^{2p+1} \, ds }{ \left( \int_{\Omega} \left( \frac{f(s)}{1 + x f(s)} \right)^{2p} \, ds + \int_{\Omega} \left( \frac{g(s)}{1 + x g(s)} \right)^{2p} \, ds \right)^{\frac{1}{2p} + 1} }
    \end{align}
    Notice that since the integrands are all continuous functions of $x$ ( they are bounded away from the singularity, again due to $\min\{f, g \} \geq c > -1$ on $\Omega$ by Lemma \ref{lemma:Psi-is-well-behaved}) the map $x \mapsto F_p(x)$ is indeed $C^1$.
    Recall that to show uniform equi-continuity is suffices to find a number $K>0$ such that for all $p$, if any $x, y \in [0,1]$ satisfy $|x-y|<\delta$ we have $|F_p(x) - F_p(y)| < K |x-y|$.
    By the mean value theorem applied to the convex domain $\Omega$, this $K$ can be chosen as:
    \begin{equation}
        \label{eq:K-is-defined}
        K = \sup_{p \in \{1, 2, \dots\}} \, \sup_{x \in [0,1]} \, \left| \frac{d F_p(x)}{dx} \right|
    \end{equation}
    provided this quantity is finite. We now show this is indeed the case.
    The key observation is to bound:
    \begin{align*}
        &\int_{\Omega}  \left( \frac{f(s)}{1 + x f(s)} \right)^{2p+1} \, ds + \int_{\Omega}  \left( \frac{g(s)}{1 + x g(s)} \right)^{2p+1} \, ds \\ 
        &\leq \max \left\{ \sup_{s \in \Omega} \left| \frac{f(s)}{1 + x f(s)} \right| \, , \, \sup_{s \in \Omega} \left| \frac{g(s)}{1 + x g(s)} \right| \right\} \left( \int_{\Omega} \left( \frac{f(s)}{1 + x f(s)} \right)^{2p} \, ds + \int_{\Omega} \left( \frac{g(s)}{1 + x g(s)} \right)^{2p} \, ds \right) \\
    \end{align*}
    Thus, equation \eqref{eq:derivative-of-Fp-2} yields:
    \begin{align*}
        \left| \frac{d F_p(x)}{dx} \right| &\leq \frac{1}{Z_p} \, \frac{ \max \left\{ \sup_{s \in \Omega} \left| \frac{f(s)}{1 + x f(s)} \right| \, , \, \sup_{s \in \Omega} \left| \frac{g(s)}{1 + x g(s)} \right| \right\} }{ \left( \int_{\Omega} \left( \frac{f(s)}{1 + x f(s)} \right)^{2p} \, ds + \int_{\Omega} \left( \frac{g(s)}{1 + x g(s)} \right)^{2p} \, ds \right)^{\frac{1}{2p}} } \\ 
    \end{align*}
    Now using \eqref{eq:bound-on-1/Zp} from Lemma \ref{lemma:equi-boundedness} we obtain:
    \begin{align*}
        \left| \frac{d F_p(x)}{dx} \right| &\leq \left( M_f^{2p} + M_g^{2p} \right)^{\frac{1}{2p}} \, \frac{ \max \left\{ \sup_{s \in \Omega} \left| \frac{f(s)}{1 + x f(s)} \right| \, , \, \sup_{s \in \Omega} \left| \frac{g(s)}{1 + x g(s)} \right| \right\} }{ \left( \int_{\Omega} \left( \frac{f(s)}{1 + x f(s)} \right)^{2p} \, ds + \int_{\Omega} \left( \frac{g(s)}{1 + x g(s)} \right)^{2p} \, ds \right)^{\frac{1}{2p}} } \\ 
    \end{align*}
    and taking the limit $p \to \infty$ we see that:
    \begin{align*}
        \limsup_{p \to \infty} \left| \frac{d F_p(x)}{dx} \right| & \leq \max\{ M_f \, , \, M_g \} \, \frac{ \max \left\{ \sup_{s \in \Omega} \left| \frac{f(s)}{1 + x f(s)} \right| \, , \, \sup_{s \in \Omega} \left| \frac{g(s)}{1 + x g(s)} \right| \right\} }{\max \left\{ \sup_{s \in \Omega} \left| \frac{f(s)}{1 + x f(s)} \right| \, , \, \sup_{s \in \Omega} \left| \frac{g(s)}{1 + x g(s)} \right| \right\} } \\
        & = \max\{ M_f \, , \, M_g \} 
    \end{align*}
    using Lemma \ref{lemma:Linfty-to-Lp-appendix} to handle the limit in the denominator.
    Thus, we conclude that $K$ as defined in \eqref{eq:K-is-defined} is finite and the lemma has been proven.
\end{proof}

\begin{lemma}
    \label{lemma:uniform-convergence-of-Fp}
    There is a subsequence $p_j \to \infty$ such that for constants $\epsilon_{p_j} \searrow 0$ we have:
    \begin{equation*}
        | F_{p_j}(\tau_{p_j}(t)) - F_\infty(\tau_{p_j}(t)) | \leq \epsilon_{p_j}
    \end{equation*}
\end{lemma}
\begin{proof}
    Note that by Lemma \ref{lemma:Linfty-to-Lp-appendix}, for any $x \in [0,1]$ we have $F_p(x) \to F_\infty(x)$.
    By Lemma \ref{lemma:equi-boundedness} the family $F_p: [0,1] \to \R$ is equi-bounded and by Lemma \ref{lemma:uniform-equi-continuity} it is also uniformly equi-continuous.
    Then, by an application of the Arzela--Ascoli theorem, we obtain the existence of a subsequence $p_j \to \infty$ such that $F_{p_j} \to F_\infty$ uniformly in $[0,1]$.
    Since $\tau_p(t) \in [0,1]$ for all $t$ we then obtain the desired result.
\end{proof}

\begin{theorem*}
    Let $\tau_\infty$ be the solution to the following ODE:
    \begin{align*}
        \frac{d}{dt} \tau_\infty(t) = \frac{1}{Z} \max \Bigg\{ 
        \left\| \frac{f(s)}{1 + \tau_\infty(t) f(s)} \right\|_{L^\infty},
        \left\| \frac{g(s)}{1 + \tau_\infty(t) g(s)} \right\|_{L^\infty} 
        \Bigg\}^{-1}
    \end{align*}    
    with boundary conditions $\tau_\infty(0) = 0$ and $\tau_\infty(1) = 1$ and
    \begin{equation*}
        Z = \int_{0}^{1} \max \Bigg\{
        \left\| \frac{f(s)}{1 + \tau_\infty(t) f(s)} \right\|_{L^\infty},
        \left\| \frac{g(s)}{1 + \tau_\infty(t) g(s)} \right\|_{L^\infty}
        \Bigg\}^{-1} \, dt \, .
    \end{equation*}
    If $\tau_p$ are solutions to the parametric ODEs \eqref{eq:opt-problem-Lp-approx-in-theorem} then there is a subsequence $\tau_{p_j} \to \tau_\infty$ in $L^2$ and therefore, by Theorem \ref{thm:Linfty-to-Lp-main-text}, $\tau_\infty$ is optimal for \eqref{eq:opt-problem}.
\end{theorem*}
\begin{proof}
    The proof is based on a Gr\"onwall estimate as well as the the Arzela--Ascoli theorem.
    Consider the ODEs:
    \begin{align*}
        \dot{\tau}_p &= F_p(\tau_p) \\
        \dot{\tau}_\infty &= F_\infty(\tau_\infty)
    \end{align*}
    with initial conditions $\tau_p(0) = \tau_\infty(0) = 0$ and $\tau_p(1) = \tau_\infty(1) = 1$.
    Now, by the discussion in the beginning of the section the function $F_\infty$ is $1$-Lipschitz thus for all $x, y \in [0,1]$ we have
    \begin{equation*}
        \label{eq:F_inf-is-Lip}
        | F_\infty(x) - F_\infty(y) | \leq | x - y | \, .
    \end{equation*}
    Moreover, by Lemma \ref{lemma:uniform-convergence-of-Fp} there is a subsequence $p_j \to \infty$ such that for constants $\epsilon_{p_j} \searrow 0$ we have:
    \begin{equation*}
        | F_{p_j}(\tau_{p_j}(t)) - F_\infty(\tau_{p_j}(t)) | \leq \epsilon_{p_j}
    \end{equation*}
    It then follows by an application of Gr\"onwall's inequality \citep[Theorem 2.1]{howard1998gronwall} that for all $t \in [0,1]$ we have
    \begin{equation*}
        | \tau_{p_j}(t) - \tau_\infty(t) | \leq \, e^{t} \, \int_{0}^{t} \epsilon_{p_j} \, e^{- s} \, ds
    \end{equation*}
    where we have used the initial conditions $\tau_p (0) = \tau_\infty(0) = 0$ for all $p \in \N$. Thus, we get:
    \begin{equation*}
        \| \tau_{p_j} - \tau_\infty \|_{L^\infty([0,1])} \leq \epsilon_{p_j}
    \end{equation*}
    and so
    \begin{equation*}
        \norm{\tau_{p_j} - \tau_\infty}_{L^2} \to 0 \quad \text{as } j \to \infty
    \end{equation*}
    which completes the proof.
\end{proof}

\subsection{Proof of Theorem \ref{thm:solution-of-L-inf-ODE-main-text}}
\label{appendix:solving-the-L-infty-ODE}
In this subsection, we prove Theorem \ref{thm:solution-of-L-inf-ODE-main-text}.
Let $(\sigma_i(s))_i$ denote the spectrum of $\nabla T(s)$ at $s \in \Omega \subset \Rd$ that is real by Assumption \ref{assumption:regularity-of-transport}.
For each $s \in \Omega$ we write $\sigma_{\max}(s) = \max_i \sigma_i(s)$ and $\sigma_{\min}(s) = \min_i \sigma_i(s)$.
Moreover, we introduce we the notation:
\begin{align*}
    f^* \coloneq \sup_{s \in \Omega} f(s) = \sup_{s \in \Omega} \sigma_{\max}(s) - 1 \\
    f_* \coloneq \inf_{s \in \Omega} f(s) = \inf_{s \in \Omega} \sigma_{\max}(s) - 1 \\
    g^* \coloneq \sup_{s \in \Omega} g(s) = \sup_{s \in \Omega} \sigma_{\min}(s) - 1 \\
    g_* \coloneq \inf_{s \in \Omega} g(s) = \inf_{s \in \Omega} \sigma_{\min}(s) - 1
\end{align*}
We start by establishing a critical lemma:
    \begin{lemma}
        \label{lemma:triviality-of-suprema}
        Suppose $\tau \in C^1([0,1])$ with $\tau(0) = 0$ and $\tau(1) = 1$. Then, for any $t \in [0,1]$ we have:
        \begin{equation*}
            \max \left\{ \, \sup_{s \in \Omega} \, \left| \frac{f(s)}{1 + \tau(t) f(s)} \right| \, , \, \sup_{s \in \Omega} \, \left| \frac{g(s)}{1 + \tau(t) g(s)} \right| \right\} = \max \left\{ \frac{f^*}{1 + \tau(t) f^*} \, , \, \frac{ - g_*}{1 + \tau(t) g_*} \, \right\}
        \end{equation*} 
\end{lemma}
\begin{proof}
    Recall that for all $t$ we have $ 0 \leq \tau(t) \leq 1$.
    This can be seen by noting that $\tau(0) = 0$, $\tau(1) = 1$ and $\dot{\tau}(t) \geq 0$ for all $t$, as seen manifestly by the ODE.
    We start by computing the left term inside the $\max$.
    Define the parametric real valued map:
    \begin{align*}
        & \phi_\theta : [-1, \infty] \to \overline{\R} \\
        & \phi_\theta: x \mapsto \left| \frac{x}{1 + \theta x} \right|
    \end{align*}
    And notice that:
    \begin{align*}
        \sup_{s \in \Omega} \left| \frac{f(s)}{1 + \tau(t) f(s)} \right| &= \sup_{s \in \Omega} \phi_{\tau(t)}(f(s)) \, . \\
    \end{align*}
    Letting $S = \{ f(s) : s \in \Omega\}$ we can apply Lemma \ref{lemma:silly-but-useful} to obtain:
    \begin{equation*}
        \sup_{s \in \Omega} \phi_{\tau(t)}(f(s)) = \max \left\{ \frac{ \sup S }{ 1 + \tau(t) \sup S } , \frac{ - \inf S }{ 1 + \tau(t) \inf S } \right\}
    \end{equation*}
    and since $\sup S = f^*$ and $\inf S = f_*$ we conclude that
    \begin{equation*}
        \sup_{s \in \Omega} \left| \frac{f(s)}{1 + \tau(t) f(s)} \right| = \max \left\{ \frac{f^*}{1 + \tau(t) f^*} \, , \, \frac{ - f_*}{1 + \tau(t) f_*} \right\} \, .
    \end{equation*}
    Arguing similarly for $g$ we obtain:
    \begin{equation*}
        \sup_{s \in \Omega} \left| \frac{g(s)}{1 + \tau(t) g(s)} \right| = \max \left\{ \frac{g^*}{1 + \tau(t) g^*} \, , \, \frac{ - g_*}{1 + \tau(t) g_*} \right\}
    \end{equation*}
    and so we have:
    \begin{align*}
        &\max \left\{ \, \sup_{s \in \Omega} \, \left| \frac{f(s)}{1 + \tau(t) f(s)} \right| \, , \, \sup_{s \in \Omega} \, \left| \frac{g(s)}{1 + \tau(t) g(s)} \right| \right\} = \\
        &\max \left\{ \max \left\{ \frac{f^*}{1 + \tau(t) f^*} \, , \, \frac{ - f_*}{1 + \tau(t) f_*} \, \right\} \, , \, \max \left\{ \frac{g^*}{1 + \tau(t) g^*} \, , \, \frac{ - g_*}{1 + \tau(t) g_*} \, \right\} \right\} \eqqcolon M \\
    \end{align*}
    Notice the definition of $M$ above, for brevity.

    \subsection*{Case I: $g^* \geq 0$ and $f_* \leq 0$.}
    Now, since for all $s\in \Omega$ we have $ \sigma_{\max}(s) \geq \sigma_{\min}(s) \implies f(s) \geq g(s)$ and so:
    \begin{align}
        \label{eq:core-ineq-for-f-star-g-star-1}
        & f^* \geq g^* \geq 0 \\
        \label{eq:core-ineq-for-f-star-g-star-2}
        & 0 \geq f_* \geq g_*
    \end{align}
    Consider, now, the following sub-case:
    \begin{equation}
        \label{eq:sub-case-1-core-1}
        \frac{f^*}{1 + \tau(t) f^*}  \leq \frac{- f_*}{1 + \tau(t) f_*}
    \end{equation}
    which is equivalent to $\frac{1}{2} \left( - \frac{1}{f_*} - \frac{1}{f^*} \right) \leq \tau(t)$. Now due to inequalities \eqref{eq:core-ineq-for-f-star-g-star-1} - \eqref{eq:core-ineq-for-f-star-g-star-2} we have $- \frac{1}{f^*} \geq - \frac{1}{g^*}$ and $- \frac{1}{f_*} \geq - \frac{1}{g_*}$ which imply that $ \frac{1}{2} \left(- \frac{1}{g^*} - \frac{1}{g_*} \right) \leq \frac{1}{2} \left(- \frac{1}{f^*} - \frac{1}{f_*} \right) \leq \tau(t)$.
    Therefore, again by direct calculation it must be that $\frac{g^*}{1 + \tau(t) g^*}  \leq \frac{- g_*}{1 + \tau(t) g_*}$ 
    Thus, we have computed:
    \begin{equation*}
        M = \max\left\{ - \frac{f_*}{1 + \tau(t) f_*} \, , \, \frac{ - g_*}{1 + \tau(t) g_*} \right\}
    \end{equation*}
    However, by \eqref{eq:core-ineq-for-f-star-g-star-1} and without loss of generality assuming $f_* \neq 0$ we have $-\frac{1}{f_*} \geq -\frac{1}{g_*}$ so for any $t \in [0,1]$ we have
    $0 \geq - \frac{1}{f_*} - \tau(t) \geq - \frac{1}{g_*} - \tau(t)$, as $f_* \geq -1$, leading to $ - \frac{1}{ \frac{1}{f_*} + \tau(t) } \leq  - \frac{1}{ \frac{1}{g_*} + \tau(t) }$ that is:
    \begin{equation}
        \label{eq:sub-case-1-core-2}
        \frac{-f_*}{1 + \tau(t) f_*} \leq \frac{-g_*}{1 + \tau(t) g_*}
    \end{equation}
    In other words:
    \begin{equation*}
        M = \frac{-g_*}{1 + \tau(t) g_*} = \max\left\{ \frac{f^*}{1 + \tau(t) f^*} \, , \, \frac{ - g_*}{1 + \tau(t) g_*} \right\}
    \end{equation*}
    and the second equality holds by combining \eqref{eq:sub-case-1-core-1} and \eqref{eq:sub-case-1-core-2}.
    Entering in the next sub-case, suppose that:
    \begin{equation}
        \label{eq:sub-case-2-core-1}
        \frac{f^*}{1 + \tau(t) f^*}  \geq \frac{- f_*}{1 + \tau(t) f_*}
    \end{equation}
    which is equivalent to $ \tau(t) \leq \frac{1}{2} \left( - \frac{1}{f_*} - \frac{1}{f^*} \right)$. Now if, in addition, $\frac{1}{2} \left(- \frac{1}{g^*} - \frac{1}{g_*} \right) \leq \tau(t) \leq \frac{1}{2} \left(- \frac{1}{f^*} - \frac{1}{f_*} \right)$ then $\frac{g^*}{1 + \tau(t) g^*}  \leq \frac{- g_*}{1 + \tau(t) g_*}$
    and so:
    \begin{equation*}
        M = \max \left\{ \frac{f^*}{1 + \tau(t) f^*} \, , \, \frac{ - g_*}{1 + \tau(t) g_*} \right\}
    \end{equation*}
    On the other hand, if $\tau(t) \leq \frac{1}{2} \left(- \frac{1}{g^*} - \frac{1}{g_*} \right) \leq \frac{1}{2} \left(- \frac{1}{f^*} - \frac{1}{f_*} \right)$ then $\frac{-g_*}{1 + \tau(t) g_*} \leq \frac{g^*}{1 + \tau(t) g^*}$ and so:
    \begin{equation*}
        M = \max \left\{ \frac{f^*}{1 + \tau(t) f^*} \, , \, \frac{ g^* }{1 + \tau(t) g^*} \right\}
    \end{equation*}
    In this we can simplify further: due to \eqref{eq:core-ineq-for-f-star-g-star-1} we have that $\frac{1}{f^*} \leq \frac{1}{g^*}$ and so for any $t \in [0,1]$ we have $ 0 \leq \frac{1}{f^*} + \tau(t) \leq \frac{1}{g^*} + \tau(t)$, leading to:
    \begin{equation}
        \label{eq:sub-case-2-core-2}
        \frac{f^*}{1 + \tau(t) f^*} \geq \frac{g^*}{1 + \tau(t) g^*}
    \end{equation}
    Thus:
    \begin{equation*}
        M = \frac{f^*}{1 + \tau(t) f^*} = \max \left\{ \frac{f^*}{1 + \tau(t) f^*} \, , \, \frac{ -g_* }{1 + \tau(t) g_*} \right\}
    \end{equation*}
    and the second equality holds by using $\frac{-g_*}{1 + \tau(t) g_*} \leq \frac{g^*}{1 + \tau(t) g^*}$.
    
    We have proved the desired equality in an exhaustive collection of sub-cases, given $g^* \geq 0$ and $f_* \leq 0$. We can now examine the other cases.

    \subsection*{Case II: $g^* < 0$ and $f_* > 0$.}
    Notice that of $\sup_{s\in\Omega} g \eqqcolon g^* < 0$ then for all $s \in \Omega$ we have $g(s) < 0$. Similarly, $f_* > 0$ implies that $f(s) > 0$ for all $s \in \Omega$.
    Thus, recalling the function $\phi_\theta$ we have that for all $t \in [0,1]$: 
    \begin{align*}
        & \sup_{s \in \Omega} \phi_{\tau(t)}\left( f(s) \right) = \left| \frac{\sup_{s \in \Omega} f(s)}{1 + \tau(t) \sup_{s \in \Omega} f(s)} \right| = \frac{f^*}{1 + \tau(t) f^*} \\
        & \sup_{s \in \Omega} \phi_{\tau(t)}\left( g(s) \right) = \left| \frac{\inf_{s \in \Omega} g(s)}{1 + \tau(t) \inf_{s \in \Omega} g(s)} \right| = \frac{-g_*}{1 + \tau(t) g_*}
    \end{align*}
    since in the first case we will always be in the first (non-negative argument) branch of $\phi_{\tau(t)}$ and in the second case we will always be in the second (non-positive argument) branch.

    \subsection*{Case III: $g^* \geq 0$ and $f_* > 0$.}
    As above, we leverage $f_* > 0$ to compute:
    \begin{equation*}
        \sup_{s \in \Omega} \phi_{\tau(t)}\left( f(s) \right) = \frac{f^*}{1 + \tau(t) f^*}
    \end{equation*}
    On the other hand, have established that:
    \begin{equation*}
        \sup_{s \in \Omega} \phi_{\tau(t)}\left( g(s) \right) = \max \left\{ \frac{g^*}{1 + \tau(t) g^*} \, , \, \frac{-g_*}{1 + \tau(t) g_*} \right\}
    \end{equation*}
    However, in Case I we computed that $\frac{g^*}{1 + \tau(t) g^*} \leq \frac{f^*}{1 + \tau(t) f^*}$ and thus we have:
    \begin{equation*}
        M = \max\left\{ \sup_{s \in \Omega} \phi_{\tau(t)}\left( f(s) \right) \, , \, \sup_{s \in \Omega} \phi_{\tau(t)}\left( g(s) \right) \right\} = \max \left\{ \frac{f^*}{1 + \tau(t) f^*} \, , \, \frac{-g_*}{1 + \tau(t) g_*} \right\}
    \end{equation*}

    \subsection*{Case IV: $g^* < 0$ and $f_* \leq 0$.}
    As in Case III, one notes that:
    \begin{equation*}
        \sup_{s \in \Omega} \phi_{\tau(t)}\left( f(s) \right) = \max \left\{ \frac{f^*}{1 + \tau(t) f^*} \, , \, \frac{-f_*}{1 + \tau(t) f_*} \right\}
    \end{equation*}
    and:
    \begin{equation*}
        \sup_{s \in \Omega} \phi_{\tau(t)}\left( g(s) \right) = \frac{-g_*}{1 + \tau(t) g_*}
    \end{equation*}
    which leads, in the same way, to:
    \begin{equation*}
        M = \max \left\{ \frac{f^*}{1 + \tau(t) f^*} \, , \, \frac{-g_*}{1 + \tau(t) g_*} \right\}
    \end{equation*}
    This concludes the proof of the lemma.
\end{proof}
Now using Lemma \ref{lemma:triviality-of-suprema} we can recast our ODE as:
\begin{equation}
    \label{eq:ODE-path-simplified}
    \dot{\tau}(t) =  \frac{1}{Z} \, \max \left\{ \, \frac{f^*}{1 + \tau(t) f^*} \, , \, - \frac{g_*}{1 + \tau(t) g_*} \, \right\}^{-1} \, .
\end{equation}
Before proceeding with the proof of the main theorem of this section, we need one last lemma:
\begin{lemma}
    \label{lemma:transition}
    If there is some $t_0 \in [0,1]$ such that $\tau(t_0) = \tau_0 \coloneq \frac{1}{2} \left( - \frac{1}{f*} - \frac{1}{g_*} \right)$ then the ODE \eqref{eq:ODE-path-simplified} is equivalent to the system:
    \begin{equation}
        \label{eq:ODE-path-system-reformulation}
        \begin{cases}
            \dot{\tau}(t) = \frac{1}{Z} \left( \frac{1}{f^*} + \tau(t)   \right) & \text{if } \tau(t) \leq \tau_0 \\
            \dot{\tau}(t) = \frac{1}{Z} \left( - \frac{1}{g_*} - \tau(t) \right) & \text{if } \tau(t) \geq \tau_0 
        \end{cases}
    \end{equation} 
    If such $t_0$ does not exist then we have two cases: if $f^* \geq -g_*$ the solution to \eqref{eq:ODE-path-simplified} is
    \begin{equation}
        \tau(t) = \frac{ (f^* + 1)^t - 1 }{f^*} \quad \text{for all } t \in [0,1]\, ,
    \end{equation}
    else if $f^* < -g_*$ then the solution to \eqref{eq:ODE-path-simplified} is
    \begin{equation}
        \tau(t) = \frac{ (g_* + 1)^{t} - 1 }{g_*} \quad \text{for all } t \in [0,1] \, .
    \end{equation}
\end{lemma}
\begin{remark}
    \label{remark:integral-definition-of-Z}
    Note that the constant $Z$ is fully determined by the equation
    \begin{equation*}
        Z = \int_{0}^{1} \min\left\{ -\frac{1}{g_*} - \tau(t) \, , \, \frac{1}{f^*} + \tau(t) \right\} dt \, ,
    \end{equation*}
    which follows by integrating both sides of the ODE and using the boundary conditions $\tau(0) = 0$ and $\tau(1) = 1$.
\end{remark}
\begin{proof}
    For any $c \in \R$, note that the the function $x \mapsto c / (x + c)$ is monotonic on any (possibly unbounded) interval that does not not contain the solution to $x + c = 0$.
    By Lemma \ref{lemma:Psi-is-well-behaved} we have that $f^* \geq g_* > -1$ and so applying the above obervation, togetther with $\dot{\tau} \geq 0$ and $\tau(0) = 0$ and $\tau(1) = 1$, 
    we have that the $\max$ in the right hand side of the ODE will transition between its two arguments at most once, at some $t_0 \in [0,1]$ satisfying:
    \begin{equation}
        \label{eq:transition-time}
        \frac{f^*}{1 + \tau(t_0) \, f^*} = \frac{- g_*}{1 + \tau(t_0) \, g_*} \iff \tau(t_0) = \frac{1}{2} \left( - \frac{1}{f^*} - \frac{1}{g_*} \right)
    \end{equation} 
    Now since $0 \leq \tau(t) \leq 1$ for all $t \in [0,1]$ and we have the ordering $-1 < g_* \leq f^*$, for the above equation to have a solution it must be that $g_* \leq 0 \leq f^*$.
    This implies that for $t \leq t_0$ the $\max$ chooses the left term and for $t \geq t_0$ the $\max$ chooses the right term.
    
    If there does not exist $t_0 \in [0,1]$ satisfying \eqref{eq:transition-time} then the form of the ODE \eqref{eq:ODE-path-simplified} is determined by the relation of its two arguements at $t = 0$.
    If $f^* \geq -g_*$ then we have
    \begin{equation}
        \dot{\tau}(t) = \frac{1}{Z} \left( \frac{1}{f^*} + \tau(t)   \right) \quad \text{for all } t \in [0,1] \, ,
    \end{equation}
    which, together with the boundary conditions $\tau(0) = 0$ and $\tau(1) = 1$ has the solution
    \begin{equation}
        \tau(t) = \frac{ (f^* + 1)^t - 1 }{f^*} \, .
    \end{equation}
    One can check this explicitly, noting that in this case the constant $Z$ is determined by the equation
    \begin{equation*}
        Z = \int_{0}^{1} \left( \frac{1}{f^*} + \tau(t) \right) dt \, .
    \end{equation*}
    Similarly, if $f^* < -g_*$ then we have
    \begin{equation}
        \dot{\tau}(t) =  \frac{ (g_* + 1)^t - 1 }{g_*}  \quad \text{for all } t \in [0,1] \, .
    \end{equation}
\end{proof}
\begin{theorem*}
    Let:
    \begin{align*}
        & t_0 = \frac{\ln \left[ \frac{1}{2} \left( 1 - \frac{f^*}{g_*} \right) \right]}{ -\ln \left( g_* + 1 \right) + \ln \left[ \frac{1}{4} \left(\frac{ f^* }{ -g_* } + \frac{-g_*}{f^*} + 2  \right) \right] } \\
        & \tau(t_0) = - \frac{1}{2}\left(\frac{1}{f^*} + \frac{1}{g_*}\right)
\end{align*}
    Now if $0 \leq t_0 \leq 1$ then the solution to the ODE \eqref{eq:ODE-solution-of-optimal-problem-in-theorem} is given by:
    \begin{equation*}
        \label{eq:solution-of-L-inf-ODE}
        \tau_\infty(t) =
        \begin{cases}
            \frac{1}{f^*} \left\{ \frac{1}{4 \left( g_* + 1 \right)} \left(\frac{ f^* }{ -g_* } + \frac{-g_*}{f^*}  + 2 \right) \right\}^t - \frac{1}{f^*} \, , & \text{if } t \leq t_0 \\
            \frac{1}{2}\left( \frac{1}{g_*} - \frac{1}{f^*} \right) \left\{ \frac{ g_* + 1 }{ \frac{1}{2}\left(1 - \frac{g_*}{f^*} \right) } \right\}^t \left\{ \frac{1}{2} \left( 1 - \frac{f^*}{g_*} \right) \right\}^{1-t }  - \frac{1}{g_*} \, , & \text{if } t \geq t_0 \, .
        \end{cases}
    \end{equation*}
    Otherwise, if $t_0 \notin [0,1]$ we have two cases: if $f^* \geq -g_*$ then the solution to \eqref{eq:ODE-solution-of-optimal-problem-in-theorem} is given by
    \begin{equation*}
    \label{eq:solution-of-L-inf-ODE-simple-form}
    \tau_\infty(t) = \frac{ (f^* + 1)^t - 1 }{f^*} \, ,
    \end{equation*}
    and if $f^* < -g_*$ then the solution to \eqref{eq:ODE-solution-of-optimal-problem-in-theorem} is given by
    \begin{equation*}
    \label{eq:solution-of-L-inf-ODE-simple-form-2}
    \tau_\infty(t) = \frac{ (g_* + 1)^{t} - 1 }{g_*} \, .
    \end{equation*}
\end{theorem*}
\begin{proof}
    To avoid overloading notation, let $\tau$ be a solution to ODE \eqref{eq:ODE-solution-of-optimal-problem-in-theorem} (as opposed to $\tau_\infty$).
    Now, assuming there is a $t_0 \in [0,1]$ such that $\tau(t_0) = \tau_0$ we can solve the ODE \eqref{eq:ODE-solution-of-optimal-problem-in-theorem} by solving the system of ODEs \eqref{eq:ODE-path-system-reformulation} together with the consistency condition 
    \begin{equation*}
        \lim_{\delta \searrow 0}\tau(t_0 + \delta) = \lim_{\delta\searrow 0} \tau(t_0 - \delta) = \tau_0 \, .
    \end{equation*}
    We note that that $t_0$ will be uniquely determined to be the expression in the statement of the theorem. 
    For $t \in [0, t_0]$ we can use an integrating factor to recast the ODE as:
    \begin{equation*}
        \frac{d}{dt} \left( \tau(t) e^{-t/Z} \right) = \frac{ e^{-t/Z} }{Z f^*} 
    \end{equation*}
    Integrating both sides and using $\tau(0) = 0$ we have:
    \begin{equation}
        \tau(t) = \frac{e^{\frac{t}{Z}} - 1}{f^*}
    \end{equation}
    Finally, using $\tau(t_0) = \tau_0$ we can compute $Z$:
    \begin{equation}
        \label{eq:first-expression-for-Z}
        \frac{1}{Z} = \frac{1}{t_0} \ln \left( \tau_0 f^* + 1 \right)
    \end{equation}
    This yields the path:
    \begin{equation}
        \label{eq:tau-for-t-less-than-t0}
        \tau(t) = \frac{e^{\frac{t}{t_0} \ln \left( \tau_0 f^* + 1 \right) } - 1}{f^*}
    \end{equation}
    Repeating the integrating factor approach for the case $t \in [t_0, 1]$ and using $\tau(t_0) = \tau_0$ we obtain:
    \begin{equation}
        \tau(t) = \frac{e^{\frac{t_0 - t}{Z}} - 1}{g_*} + e^{\frac{t_0 - t}{Z}} \tau_0
    \end{equation}
    Using that $\tau(1) = 1$ we can compute $Z$:
    \begin{equation}
        \label{eq:second-expression-for-Z}
        \frac{1}{Z} = \frac{1}{t_0 - 1} \ln \left( \frac{g_* + 1}{ \tau_0 g_* + 1 } \right)
    \end{equation}
    This yields the path:
    \begin{equation}
        \label{eq:tau-for-t-greater-than-t0}
        \tau(t) = \frac{e^{\frac{t_0 - t}{t_0 - 1} \ln \left( \frac{g_* + 1}{ \tau_0 g_* + 1} \right) } - 1}{g_*} + e^{\frac{t_0 - t}{t_0 - 1} \ln \left( \frac{g_* + 1}{ \tau_0 g_* + 1} \right)} \tau_0
    \end{equation}
    Now equating \eqref{eq:first-expression-for-Z} and \eqref{eq:second-expression-for-Z} we can compute $t_0$:
    \begin{align}
        \label{eq:expression-for-t0-pre-sub}
        \frac{t_0 - 1}{t_0} &= \frac{\ln \left( \frac{g_* + 1}{ \tau_0 g_* + 1} \right) }{ \ln \left( \tau_0 f^* + 1 \right) } \\
        \frac{1}{t_0} &= 1 - \frac{\ln \left( \frac{g_* + 1}{ \tau_0 g_* + 1} \right) }{ \ln \left( \tau_0 f^* + 1 \right) } \\
        t_0 &= \frac{1}{1 - \frac{\ln \left( g_* + 1 \right) - \left( \tau_0 g_* + 1 \right) }{ \ln \left( \tau_0 f^* + 1 \right) }} \\
        \label{eq:expression-for-t0-pre-sub-final}
        t_0 &= \frac{\ln \left( \tau_0 f^* + 1 \right)}{ -\ln \left( g_* + 1 \right) + \ln \left( \tau_0 f^* + 1 \right) + \ln \left( \tau_0 g_* + 1 \right) }
    \end{align}
    Recalling that $\tau_0 = - \frac{1}{2}\left(1/f^* + 1/g_*\right)$ we can compute:
    \begin{align}
        \label{eq:aux-tau-expressions}
        \tau_0 f^* + 1 &= - \frac{1}{2} \left( 1 + \frac{f^*}{g_*} \right) + 1 = \frac{1}{2}\left(1 - \frac{f^*}{g_*} \right) \\
        \tau_0 g_* + 1 &= - \frac{1}{2} \left( 1 + \frac{g_*}{f^*} \right) + 1 = \frac{1}{2}\left(1 - \frac{g_*}{f^*} \right)
    \end{align}
    Combining the logarithm terms in the denominator of \eqref{eq:expression-for-t0-pre-sub-final} we conclude:
    \begin{align}
        t_0 &= \frac{\ln \left[ \frac{1}{2} \left( 1 - \frac{f^*}{g_*} \right) \right]}{ - \ln \left( g_* + 1 \right) + \ln \left[ \frac{1}{2} \left( 1 - \frac{ f^* }{ g_* }  \right) \frac{1}{2} \left( 1 - \frac{g_*}{f^*}  \right) \right] } \\
    \end{align}
    Which is the same as:
    \begin{align}
        \label{eq:expression-for-t0}
        t_0 &= \frac{\ln \left[ \frac{1}{2} \left( 1 - \frac{f^*}{g_*} \right) \right]}{ -\ln \left( g_* + 1 \right) + \ln \left[ \frac{1}{4} \left(\frac{ f^* }{ -g_* } + \frac{-g_*}{f^*}  \right) + \frac{1}{2}  \right] }
    \end{align}
        We can now substitute \eqref{eq:expression-for-t0} into the expression for \eqref{eq:tau-for-t-less-than-t0} to obtain:
    \begin{align}
        \tau(t) &= \frac{ e^{ \left\{ -\ln \left( g_* + 1 \right) + \ln \left[ \frac{1}{4} \left(\frac{ f^* }{ -g_* } + \frac{-g_*}{f^*} + 2 \right) \right]  \right\} t } - 1 }{f^*} \\
            &= \frac{\left\{ \frac{1}{4 \left( g_* + 1 \right)} \left(\frac{ f^* }{ -g_* } + \frac{-g_*}{f^*} + 2  \right) \right\}^t - 1}{f^*}
    \end{align}
    Now for \eqref{eq:tau-for-t-greater-than-t0} start by computing, for $t_0 \leq t \leq 1$, the quantity:
    \begin{align*}
        t_0 - t = \frac{ t \, \ln \left( \frac{ g_* + 1 }{ \tau_0 g_* + 1 } \right) + (1-t) \, \ln \left( \tau_0 f^* + 1 \right) }{ \ln \left[ \frac{ \left( \tau_0 f^* + 1 \right) \left( \tau_0 g_* + 1 \right)  }{g_* + 1} \right] }
    \end{align*}
    Proceed by substituting into \eqref{eq:tau-for-t-greater-than-t0} and refactoring to get:
    \begin{equation*}
        \tau(t) = \left( \frac{1}{g_*} + \tau_0 \right) e^{ t \, \ln \left( \frac{ g_* + 1 }{ \tau_0 g_* + 1 } \right) + (1-t) \, \ln \left( \tau_0 f^* + 1 \right) } - \frac{1}{g_*}
    \end{equation*}
    Finally, using the expressions \eqref{eq:aux-tau-expressions} together with the known form of $\tau_0$ and combining the logarithms in the exponential we get:
    \begin{equation*}
        \tau(t) = \frac{1}{2}\left( \frac{1}{g_*} - \frac{1}{f^*} \right) \left\{ \frac{ g_* + 1 }{ \frac{1}{2}\left(1 - \frac{g_*}{f^*} \right) } \right\}^t \left\{ \frac{1}{2} \left( 1 - \frac{f^*}{g_*} \right) \right\}^{1-t }  - \frac{1}{g_*}
    \end{equation*}
    Thus, the solution to the ODE \eqref{eq:ODE-path-system-reformulation} is:
    \begin{equation}
        \label{eq:solution}
        \tau(t) =
        \begin{cases}
            \frac{1}{f^*} \left\{ \frac{1}{4 \left( g_* + 1 \right)} \left(\frac{ f^* }{ -g_* } + \frac{-g_*}{f^*}  + 2 \right) \right\}^t - \frac{1}{f^*} \, , & \text{if } t \leq t_0 \\
            \frac{1}{2}\left( \frac{1}{g_*} - \frac{1}{f^*} \right) \left\{ \frac{ g_* + 1 }{ \frac{1}{2}\left(1 - \frac{g_*}{f^*} \right) } \right\}^t \left\{ \frac{1}{2} \left( 1 - \frac{f^*}{g_*} \right) \right\}^{1-t }  - \frac{1}{g_*} \, , & \text{if } t \geq t_0
        \end{cases}
    \end{equation}
    Lastly, if there is no $t_0 \in [0,1]$ such that $\tau(t_0) = \tau_0$ then, by Lemma \ref{lemma:transition} we can conclude.
\end{proof}

\subsection{Proof of Theorem \ref{thm:non-asymptotic-result-main-text} and Corollary \ref{corr:Lipschitz-constant-improvement} }
\label{subsec:exponential-Lipschitz-improvement}
In this subsection, we prove Theorem~\ref{thm:non-asymptotic-result-main-text} and Corollary~\ref{corr:Lipschitz-constant-improvement}.
Let $(\sigma_i(s))_i$ denote the spectrum of $\nabla T(s)$ at $s \in \Omega \subset \Rd$ that is real by Assumption \ref{assumption:regularity-of-transport}.
For each $s \in \Omega$ we write $\sigma_{\max}(s) = \max_i \sigma_i(s)$ and $\sigma_{\min}(s) = \min_i \sigma_i(s)$ as well as $f(s) = \sigma_{\max}(s) - 1$ and $g(s) = \sigma_{\min}(s) - 1$.
Moreover, we introduce the notation:
\begin{align*}
    &\sigma_{\max}^* \coloneq \sup_{s \in \Omega} \sigma_{\max}(s) \\
    &\sigma_{\min}^* \coloneq \inf_{s \in \Omega} \sigma_{\min}(s) 
\end{align*}
as well as
\begin{align*}
    &f^* \coloneq \sigma_{\max}^* - 1 \\
    &g_* \coloneq \sigma_{\min}^* - 1 \, .
\end{align*}
Lastly, for real valued functions $\alpha, \beta: \R^2 \to \R$ write $\alpha \asymp \beta$ if there are constants $c_1, c_2 > 0$ and some $y \in \R^2$ such that $c_1 \beta(x) \leq \alpha(x) \leq c_2 \beta(x)$ for all $x_i \geq y_i$ and $i \in \{1,2\}$.
Now we can proceed with the proof of Theorem \ref{thm:non-asymptotic-result-main-text}.
\begin{theorem*}
    For $\Lambda$ as in equation \eqref{eq:lip-jacobian-schedule} and $\bar{\tau}$ the trivial schedule $\bar{\tau}: t \mapsto t$ we have
    \begin{equation*}
        \Lambda[\bar{\tau}] = \max \left\{ \sigma_{\max}^* - 1 \, , \, \frac{1 - \sigma^*_{\min}}{\sigma^*_{\min}} \right\} \, .
    \end{equation*}
    Now if $\tau_\infty$ is the optimal schedule, i.e., the minimizer of problem \eqref{eq:opt-problem} obtained as a solution of ODE \eqref{eq:ODE-solution-of-optimal-problem-in-theorem},  we have the following cases:
    If there exists $t_0 \in [0,1]$ such that $\tau_\infty(t_0) = - \frac{1}{2}\left( \frac{1}{\sigma_{\max}^* - 1} + \frac{1}{{\sigma_{\min}^* - 1}} \right)$ then
    \begin{align*}
      \Lambda[\tau_\infty] = & \ln\left[\frac{\sigma^*_{\max}-1}{\sigma^*_{\min}}\right] + 
                        \ln\left[\frac{1}{4}\left(\frac{1}{1 - \sigma_{\min}^*} +\frac{1 - \sigma_{\min}^*}{(\sigma^*_{\max}-1)^2} + \frac{2}{\sigma_{\max}^* -1} \right)\right] \, .
    \end{align*}
    Otherwise, we have
    \begin{equation*}
        \Lambda[\tau_\infty] = 
        \begin{cases}
            \ln \left( \sigma^*_{\max} \right) & \text{if } \sigma_{\max}^* + \sigma_{\min}^* \geq 2 \\
            -\ln \left( \sigma^*_{\min}\right) & \text{if } \sigma_{\max}^* + \sigma_{\min}^* < 2 \, .
        \end{cases}
    \end{equation*}
    \end{theorem*}
\begin{proof}
    Recall the notation $\Omega_{[0,1]} = \Omega \times [0,1]$.
    By Theorem \eqref{thm:reformulate-opt-problem} for any $\tau \in \mathcal{T}_\infty$ we can re-write $\Lambda[\tau]$ as:
    \begin{align*}
        \Lambda[\tau] = \sup_{ (x,t) \in \Omega_{[0,1]} } |\dot{\tau}(t)| \, \max \Bigg\{ 
            \left| \frac{f(x)}{1 + \tau(t) f(x)} \right| \, , \,
            \left| \frac{g(x)}{1 + \tau(t) g(x)} \right| 
        \Bigg\}.
    \end{align*}
    Using Lemma \ref{lemma:exchange-max-sup} we can push the $\sup$ over $s \in \Omega$ inside the $\max$ and then apply Lemma \ref{lemma:triviality-of-suprema} to conclude:
    \begin{align}
        \label{eq:expression-for-Lambda-easy}
        \Lambda[\tau] = \sup_{t \in [0,1]} \dot{\tau}(t) \, \max \left\{ \frac{f^*}{1 + \tau(t) f^*} \, , \, \frac{ - g_*}{1 + \tau(t) g_*} \, \right\}
    \end{align}
    where we have used that all $\tau \in \mathcal{T}_\infty$ are non-decreasing.

    Now take $\tau = \bar{\tau}$ where $\bar{\tau}(t) = t$ is the trivial schedule.
    Then, we have:
    \begin{align*}
        \Lambda[\bar{\tau}] &= \sup_{t \in [0,1]} \max \left\{ \frac{f^*}{1 + t \, f^*} \, , \, \frac{-g_*}{1 + t \, g_*} \right\}  \\ 
        &=  \max \left\{ \max \Big\{ f^* \, , \, \frac{ f^*}{1 + f^*}  \Big\} \, , \, \max\left\{  -g_* \, , \, \frac{ -g_*}{1 + g_*} \right\} \right\} \\ 
        &=  \max \left\{ f^* \, , \, \frac{ -g_*}{1 + g_*}  \right\} \\ 
        &=  \max \left\{ \sigma_{\max}^* - 1 , \frac{1-\sigma_{\min}^*}{\sigma_{\min}^*} \right\}
    \end{align*}
    where in the second equality we have used the obvious monotonicity of the function $t \mapsto a / (t + a)$ for $a \in \R$ and in the third equality we argue as follows:
    recall that
    \begin{equation*}
        f^* \geq g_* > -1
    \end{equation*}
    by Lemma \ref{lemma:Psi-is-well-behaved}. If $f^* \geq 0$ then we clearly have $f^* \geq f^* / (1 + f^*)$ and if $f^* \leq 0$ then $f^* / (1 + f^*) \leq f^*$ since $1/(1+f^*) > 1$.
    A similar argument shows that the second nested $\max$ reduces to $-g/(1+g)$.

    Now for $\tau = \tau_\infty$ the optimal path, recall that by Theorem \ref{thm:solution-of-L-inf-ODE-main-text} the form of $\tau_\infty$ depends on the existence of a solution $t_0$ to the equation:
    \begin{equation}
        \label{eq:transition-time-exists}
        \tau_\infty(t_0) = - \frac{1}{2}\left(\frac{1}{f^*} + \frac{1}{g_*}\right) \quad \text{for} \quad t_0 \in [0,1]
    \end{equation}
    This solution will be called the \textit{transition time}.
    \subsection*{Existence of a transition time $t_0$}
    Assume there is a transition time $t_0 \in [0,1]$ as stated above.
    We claim that \eqref{eq:transition-time-exists} implies that $f^* \geq 0 \geq g_*$.
    Recall that by Lemma \ref{lemma:Psi-is-well-behaved} we have that $f^* \geq g_* > -1$.
    Now if $f^* \geq g_* \geq 0$ then \eqref{eq:transition-time-exists} cannot have a solution since the left hand side is non-negative and the right hand side is negative.
    On the other hand, if $0 \geq f^* \geq g_* > -1$ then $-1 > 1/f^* , 1/g_*$ and so the right hand side of \eqref{eq:transition-time-exists} is greater then one whereas $\tau_\infty(t) \leq 1$ for all $t \in [0,1]$.

    With this information we can use the monotonicity of $\tau_\infty$ to write:
    \begin{equation}
        \label{eq:expression-for-tau-infty-cases-explicit}
        \dot{\tau}_\infty(t) = 
        \begin{cases}
            \frac{1}{Z} \, \frac{1 + f^* \tau_\infty(t)}{f^*}   & \text{if } t \leq t_0 \\
            \frac{1}{Z} \,  \frac{1 + g_* \tau_\infty(t)}{- g_*} & \text{if } t \geq t_0
        \end{cases}
    \end{equation}
    and $Z$ is defined by the integral expression in remark \ref{remark:integral-definition-of-Z}.
    We can now combine \eqref{eq:expression-for-Lambda-easy} with \eqref{eq:expression-for-tau-infty-cases-explicit} to obtain:
    \begin{align*}
        \sup_{s \in \Omega} \Lambda[\tau_\infty](s,t) &= 
        \begin{cases}
            \frac{ \dot{\tau}_\infty(t) \, f^*}{1 + \tau(t) \, f^*} & \text{if } t \leq t_0 \\
            \frac{ \dot{\tau}_\infty(t) \, (-g_*)}{1 + \tau(t) \, g_*} & \text{if } t \geq t_0
        \end{cases} \\
        & =
        \begin{cases}
            \frac{1}{Z} \, \frac{1 + f^* \tau(t)}{f^*} \, \frac{f^*}{1 + \tau(t) \, f^*} & \text{if } t \leq t_0 \\
            \frac{1}{Z} \, \frac{1 + g_* \tau(t)}{- g_*} \, \frac{-g_*}{1 + \tau(t) \, g_*} & \text{if } t \geq t_0
        \end{cases} \\
        & = \frac{1}{Z}
    \end{align*}  
    Now combining equation \eqref{eq:first-expression-for-Z} with equation \eqref{eq:expression-for-t0-pre-sub-final}, both part of the proof of Theorem \ref{thm:solution-of-L-inf-ODE-main-text}, we obtain an explicit expression for $Z$:
    \begin{align*}
        \frac{1}{Z} &= \frac{1}{t_0} \ln \left( \tau_0 f^* + 1 \right) \\
        &= - \ln (g_* + 1) + \ln \left[ \frac{1}{4} \left(\frac{ f^* }{ -g_* } + \frac{-g_*}{f^*} + 2 \right) \right]
    \end{align*}
    Re-writing this in terms of $\sigma^*_{\min}$ and $\sigma^*_{\max}$ we can conclude.

    \subsection*{Non-existence of a transition time $t_0$}
    Assume, first, that $f^* \geq -g_*$.
    By Lemma \ref{lemma:transition} the solution to the ODE is given by:
    \begin{equation}
        \tau(t) = \frac{ (f^* + 1)^t - 1 }{f^*}
    \end{equation}
    and we can compute:
    \begin{align}
        \dot{\tau}(t) &= \ln (f^* + 1) \, \frac{(f^* + 1)^t}{f^*} \\
                &= \ln (f^* + 1) \, \left( \tau(t) + \frac{1}{f^*} \right) \\
                &= \ln (f^* + 1) \, \frac{ 1 + \tau(t) f^* }{f^*}
    \end{align}
    Moreover:
    \begin{equation}
        \max \left\{ \, \frac{f^*}{1 + \tau(t) f^*} \, , \, \frac{-g_*}{1 + \tau(t) g_*} \, \right\} = \frac{f^*}{1 + \tau(t) f^*}
    \end{equation}
    as the two expressions in the $\max$ can never match, given that $t_0 \notin [0,1]$, and so their relation for all $t$ is determined by their relation at $t = 0$. 
    Thus:
    \begin{equation*}
        \sup_{s \in \Omega} \Lambda[\tau_\infty](s,t) = \ln (f^* + 1) = \ln \left( \sigma^*_{\max} \right)
    \end{equation*}
    Similarly, if $f^* < -g_*$, we obtain:
    \begin{equation*}
        \sup_{s \in \Omega} \Lambda[\tau_\infty](s,t) = -\ln (g_* + 1) = -\ln \left( \sigma^*_{\min} \right) \, .
    \end{equation*}
    Finally, writing:
    \begin{equation*}
        f^* \geq -g_* \iff \sigma_{\max}^* - 1 \geq 1 - \sigma_{\min}^* \iff \sigma_{\max}^* + \sigma_{\min}^* \geq 2
    \end{equation*}
    we can conclude.
\end{proof}

\begin{corollary*}
    Viewing $\Lambda[\tau_\infty]$ and $\Lambda[\bar{\tau}]$ as functions of $\sigma^*_{\max}$ and $1/\sigma^*_{\min}$ we have:
    \begin{align*}
        \Lambda[\tau_\infty] &\asymp \ln  \frac{\sigma^*_{\max}}{\sigma^*_{\min}} \\
        \Lambda[\bar{\tau}]  &\asymp \max \left\{ \sigma_{\max}^* \, , \, \left( \sigma_{\min}^* \right)^{-1} \right\} 
    \end{align*}
\end{corollary*}
\begin{proof}
    The proof is a computation. Let us start with $\Lambda[\bar{\tau}]$. For small enough $\sigma_{\min}^*$ and large enough $\sigma_{\max}^*$ we have:
    \begin{align}
        \label{eq:small-lambda-min-large-lambda-max-1}
        & 1 / 2 \leq \frac{\sigma_{\max}^* - 1}{\sigma^*_{\max}} \leq 2 \\
        \label{eq:small-lambda-min-large-lambda-max-2}
        & 1 / 2 \leq 1 - \sigma_{\min}^* \leq 2 
    \end{align}
    so:
    \begin{align*}
        \Lambda[\bar{\tau}] &= \max \left\{ \frac{\sigma_{\max}^* - 1}{\sigma_{\max}^*} \sigma_{\max}^* \, , \, \frac{1 - \sigma_{\min}^*}{\sigma_{\min}^*} \right\} \\
    \end{align*}
    which allows us to conclude:
    \begin{align*}
        \frac{1}{4} \, \max \left\{ \sigma_{\max}^* \, , \, \left( \sigma_{\min}^* \right)^{-1} \right\}  \leq \Lambda[\bar{\tau}] \leq 4 \, \max \left\{ \sigma_{\max}^* \, , \, \left( \sigma_{\min}^* \right)^{-1} \right\} 
    \end{align*}

    For $\Lambda[\tau_\infty]$, start by observing that under the double limit $\sigma^*_{\min} \to 0$ and $\sigma^*_{\max} \to \infty$, due to the continuity of $\tau$, the following equation has a solution:
    \begin{equation*}
        \tau(t) = - \frac{1}{2}\left( \frac{1}{f^*} + \frac{1}{g_*} \right) \quad \text{some} \quad t \in [0,1]
    \end{equation*}
    Thus, we can focus out attention to equation \eqref{eq:generic-formula-non-asymptotic-result} of Theorem \ref{thm:non-asymptotic-result-main-text}.
    Under equations \eqref{eq:small-lambda-min-large-lambda-max-1} - \eqref{eq:small-lambda-min-large-lambda-max-2} and by a similar algebraic manipulation as above we can find a real number $C > 0$ such that
    \begin{equation*}
        \ln \frac{\sigma_{\max}^*}{\sigma_{\min}^*} - \ln C \leq \Lambda[\tau_\infty] \leq \ln \frac{\sigma_{\max}^*}{\sigma_{\min}^*} + \ln C \, ,
    \end{equation*}
    and noting that $\sigma_{\max}^* / \sigma_{\min}^* \to \infty$ we can conclude
    \begin{equation*}
        \ln \frac{\sigma_{\max}^*}{\sigma_{\min}^*} \left( 1 - \ln C \right) \leq \Lambda[\tau_\infty] \leq \ln \frac{\sigma_{\max}^*}{\sigma_{\min}^*} \left( 1 + \ln C \right) \, .
    \end{equation*}
\end{proof}

\section{Example \ref{subsec:univariate-gaussian} and unbounded domains}
\label{sec:appendix-b}
In this appendix, we discuss example \ref{subsec:univariate-gaussian}, particularly with regards to the difficulty presented by working on $\R$, an unbounded domain. 
To apply the results of section \ref{sec:main-results} we need to work in a compact, convex $\Omega \subset \R$. As such, we instead choose to solve an approximate problem.
To that end consider:
\begin{definition}
    \label{def:epsilon-transport}
    For a subset $\Omega \subset \Rd$, we say that a measurable map $T:\Omega \to \R$ is an $\epsilon$-transport of the measure $\mu \in \mathcal{P}(\Omega)$ to the measure $\nu \in \mathcal{P}(\Omega)$ if for all Borel sets $A \subset \Omega$ we have:
    \begin{equation*}
        \left| T_{\sharp} \mu(A) - \nu(A) \right| < \epsilon
    \end{equation*}
\end{definition} 
Note that this is definitionally equivalent to saying that the \textit{total variation distance} of $T_{\sharp} \mu$ and $\nu$ is less than $\epsilon$. 

Let $\mu_1$ and $\mu_2$ be as in example \ref{subsec:univariate-gaussian}. First, we choose $\Omega = [-M, M]$ for some $M > 0$ large enough.
Instead of transporting $\mu_1$ to $\mu_2$ we will $\epsilon$-transport $\hat{\mu}_1 = C_1 \, \mu_1 \, \id_{\Omega}$ to $\hat{\mu}_2 = C_2 \, \mu_2 \, \id_{\Omega}$ and the $C_i$ are normalizing constants. 
Since $\hat{\mu}_i$ can be made arbitrarily close to $\mu_i$, for $i \in \{1,2\}$, this will be good enough for our purposes.
In proposition \ref{prop:truncated-transport-for-1D-Gaussians} we show that this can be achieved by the optimal transport map $T$ between $\mu_1$ and $\mu_2$, given by the formula $T(x) = \mu_2 + \frac{\theta_2}{\theta_1}(x - \mu_1)$.
This justifies optimizing the schedule of the linear linear flow $X(\cdot, t) = (1-t)\,\textup{id} + t \, T$ and $T$ is as above.
\begin{proposition}
    \label{prop:truncated-transport-for-1D-Gaussians}
    For $i \in \{1,2\}$ let $\mu_i = \mathcal{N}(\mu_i, \sigma_i^2)$ be Gaussian distributions $\R$ with $\sigma_i > 0$ and $\mu_i \in \R$ and let $T$ be the (unique) optimal transport map
    coupling $\mu$ to $\nu$.
    For any $\epsilon > 0$ there exists some $M > 0$ such that the truncations $\hat{\mu}_i = C_i \, \mu_i \, \id_{[-M, M]}$ are such that $T$ is an $\epsilon$-transport (see definition \ref{def:epsilon-transport}) from $\hat{\mu}_1$ to $\hat{\mu}_2$.
\end{proposition}
\begin{proof}
    Start by noting that $C_i$ ensures that $\hat{\mu}_i$ integrates to one and so:
    \begin{equation*}
        C_i = \frac{1}{\Phi\left( \frac{M - \mu_i}{\sigma_i} \right) - \Phi\left( \frac{-M - \mu_i}{\sigma_i} \right)}
    \end{equation*}
    and $\Phi$ is the cumulative distribution function of the standard normal distribution.
    Recall that the transport map $T$ is affine:
    \begin{equation*}
        T(x) = \mu_2 + \frac{\theta_2}{\theta_1}(x - \mu_1)
    \end{equation*}
    and choose $M$ large enough such that:
    \begin{align*}
        \left| C_1 -  C_2 \right| &< \epsilon / 3 \quad \text{ and } \\
        \int_{\R \setminus [-M, M]} \frac{C_1}{\sqrt{2\pi} \theta_2} \exp\left( - \frac{(x - \mu_2)^2}{2\theta_2^2} \right) \, dx &< \epsilon / 3 \quad \text{ and } \\
        \int_{\R \setminus [T(-M), T(M)]} \frac{C_1}{\sqrt{2\pi} \theta_2} \exp\left( - \frac{(x - \mu_2)^2}{2\theta_2^2} \right) \, dx &< \epsilon / 3 \, . \\
    \end{align*}
    We claim that the density of $T_{\sharp} \hat{\mu}_1$ satisfies:
    \begin{equation*}
        \frac{ d \left(T_{\sharp} \hat{\mu}_1 \right) }{d x} = C_1 \, \frac{1}{\sqrt{2\pi} \theta_2} \exp\left( - \frac{(x - \mu_2)^2}{2\theta_2^2} \right) \id_{[T(-M), T(M)]}
    \end{equation*}
    which follows by\footnote{
        Note that we also used the monotonicity of $T$ to write:
        \begin{equation*}
            \id_{\left[ -M, M \right] } \left(T^{-1}(x)\right) = \id_{ \left[ T(-M), T(M) \right] }(x)
        \end{equation*}
    } the change of variables formula:
    \begin{equation*}
        \frac{ d \left(T_{\sharp} \hat{\mu}_1 \right) }{d x} = \frac{1}{\det \nabla T\left( T^{-1}(x)\right)} \frac{d \hat{\mu}_1}{dx} \left(T^{-1}(x)\right)
    \end{equation*}
    as seen for example in section 1.6. of \cite{figalli2023invitation}.
    Thus, for any Borel $A \subset [-M, M]$ we have:
    \begin{align*}
        | T_{\sharp} \hat{\mu}_1(A) - \hat{\mu}_2(A) | &\leq \int_{A} \left| \frac{ d \left(T_{\sharp} \hat{\mu}_1 \right) }{d x} - \frac{ d \hat{\mu}_2 }{d x} \right| \, dx \\
        & = \int_{A} \left| C_1 \, \id_{[T(-M), T(M)]}  - C_2 \,  \id_{[-M, M]} \right| \, \frac{1}{\sqrt{2\pi} \theta_2} \exp\left( - \frac{(x - \mu_2)^2}{2\theta_2^2} \right) \, dx  \\
        & \leq \int_{\R} \left| C_1 - C_2 \right| \, \frac{1}{\sqrt{2\pi} \theta_2} \exp\left( - \frac{(x - \mu_2)^2}{2\theta_2^2} \right) \, dx \\
        & + \int_{\R \setminus [-T(M), T(M)]} \frac{C_1}{\sqrt{2\pi} \theta_2} \exp\left( - \frac{(x - \mu_2)^2}{2\theta_2^2} \right) \, dx \\
        & + \int_{\R \setminus [-M, M]} \frac{C_1}{\sqrt{2\pi} \theta_2} \exp\left( - \frac{(x - \mu_2)^2}{2\theta_2^2} \right) \, dx \\
        & < \epsilon/3 + \epsilon/3 + \epsilon/3 \\ 
        & = \epsilon
    \end{align*}

\end{proof}

\end{document}